\documentclass{article}

\usepackage{fullpage}
\usepackage{parskip}

\usepackage[utf8]{inputenc}
\usepackage[T1]{fontenc}
\usepackage{lmodern}
\usepackage[hidelinks]{hyperref}
\usepackage{url}
\usepackage{booktabs}
\usepackage{amsfonts}
\usepackage{nicefrac}
\usepackage{microtype}
\usepackage{paralist}

\usepackage{graphicx}
\graphicspath{{imgs/}}
\usepackage{caption,subcaption}
\usepackage[dvipsnames]{xcolor}
\usepackage{pgfplots}
\pgfplotsset{compat=1.14}

\usepackage{amsmath,amsthm,amssymb}
\usepackage{mathtools}
\usepackage{altvars}
\usepackage{bbm}

\usepackage[capitalise]{cleveref}
\crefname{equation}{}{}
\Crefname{equation}{Equation}{Equations}
\crefname{figure}{Figure}{Figures}
\Crefname{proposition}{Proposition}{Propositions}
\Crefname{conjecture}{Conjecture}{Conjectures}
\usepackage[numbers]{natbib}

\newtheorem{theorem}{Theorem}[section]
\newtheorem{lemma}[theorem]{Lemma}
\newtheorem{example}[theorem]{Example}
\newtheorem{corollary}[theorem]{Corollary}
\newtheorem{proposition}[theorem]{Proposition}
\newtheorem{conjecture}[theorem]{Conjecture}
\newtheorem{definition}[theorem]{Definition}

\usepackage{custom_tex}
\newcommand{\robrisk}{R_{\mathrm{rob}}}
\newcommand{\stdrisk}{R_{\mathrm{std}}}
\newcommand{\bayesrisk}{R_{\mathrm{Bay}}}
\newcommand{\bayesclass}{\hty^*_{\mathrm{Bay}}}
\newcommand{\linclass}{\hty_{\mathrm{lin}}}
\newcommand{\intclass}{\hty_{\mathrm{int}}}

\newcommand{\interior}{\textbf{int}\hspace{2pt}}
\newcommand{\norm}[1]{\left\| #1 \right\|}

\newcommand{\primetranspose}{
  \mathcode`'="8000
  \begingroup\lccode`~=`'
  \lowercase{\endgroup\def~}{^\top}
}

\newcommand{\oneortwo}[2]{%
  \makeatletter%
  \if@twocolumn%
  #2%
  \else%
  #1%
  \fi%
  \makeatother%
}

\newcommand{\blfootnote}[1]{% for fund line from https://tex.stackexchange.com/a/30726
  \begingroup
  \renewcommand\thefootnote{}\footnote{#1}%
  \addtocounter{footnote}{-1}%
  \endgroup
}

\newcommand{\revone}[1]{#1}

\title{Provable tradeoffs in adversarially robust classification\thanks{Equal contribution from all authors.}}

\author{%
  Edgar Dobriban%
  \thanks{%
    Department of Statistics and Data Science,
    University of Pennsylvania,
    Philadelphia, PA, 19104 USA
    (emails: \{dobriban, dahong67\}@wharton.upenn.edu).
  }%
  \and
  Hamed Hassani%
  \thanks{%
    Department of Electrical and Systems Engineering,
    University of Pennsylvania,
    Philadelphia, PA, 19104 USA
    (emails: \{hassani, arobey1\}@seas.upenn.edu).
  }%
  \and
  David Hong%
  \footnotemark[1]%
  \and
  Alexander Robey%
  \footnotemark[2]%
}

\begin{document}

\maketitle

\blfootnote{%
  E. Dobriban was supported in part by NSF BIGDATA grant IIS 1837992.
  D. Hong was supported in part by
  NSF BIGDATA grant IIS 1837992,
  the Dean's Fund for Postdoctoral Research of the Wharton School,
  and NSF Mathematical Sciences Postdoctoral Research Fellowship DMS 2103353.
  The research of A. Robey and H. Hassani was supported by
  NSF HDR TRIPODS award 1934876,
  NSF award CPS-1837253,
  NSF award CIF-1910056,
  NSF CAREER award CIF-1943064,
  and the Air Force Office of Scientific Research Young Investigator Program (AFOSR-YIP) under award \#FA9550-20-1-0111.
}

\begin{abstract}
% !TEX root = ../adv.tex

It is well known that machine learning methods can be vulnerable to adversarially-chosen perturbations of their inputs. Despite significant progress in the area, foundational open problems remain. In this paper, we address several key questions. We derive exact and approximate \emph{Bayes-optimal robust classifiers} for the important setting of two- and three-class Gaussian classification problems with arbitrary imbalance, for $\ell_2$ and $\ell_\infty$ adversaries. In contrast to classical Bayes-optimal classifiers, determining the optimal decisions here cannot be made pointwise and new theoretical approaches are needed. 
We develop and leverage new tools, including recent breakthroughs from probability theory on robust isoperimetry, which, to our knowledge, have not yet been used in the area. 
Our results reveal fundamental tradeoffs between standard and robust accuracy that grow when data is imbalanced.  We also show further results, including an analysis of classification calibration for convex losses in certain models, and finite sample rates for the robust risk. 

\end{abstract}


% !TEX root = ../adv.tex

\section{Introduction}

Machine learning methods, such as deep neural nets, have shown remarkable performance in numerous application domains ranging from computer vision to natural language processing \citep[see, e.g.,][]{lecun2015deep}. However, despite this documented success, it is now well-known that many of these methods are also highly vulnerable to adversarial attacks.  Indeed, it has been repeatedly shown that adversarially-chosen, imperceptible changes to the input data at test time can have undesirable effects on the predictions of models that otherwise perform well. For example, imperceptible pixel-wise changes to images are known to severely degrade the performance of state-of-the-art image classifiers~\citep{szegedy2013intriguing,biggio2013evasion}.

\emph{Adversarial training} methods \citep[e.g.,][]{goodfellow2014explaining,madry2017towards,wong2017provable,moosavi2016deepfool,zhang2019theoretically}
tackle this problem by seeking models that are \emph{robust} to adversarial attacks.
A common approach is to replace the standard risk used to assess classifier performance
with a \emph{robust risk} that incorporates the possibility of small perturbations to the input.
To illustrate this approach,  consider the classification problem of assigning labels $y \in \clC$ to input vectors (e.g.\revone{,} images) $x \in \bbR^p$.
Traditional, non-adversarial training techniques seek classifiers $\hty: \bbR^p \to \clC$
that minimize the standard risk (misclassification probability)%
\footnote{To simplify the discussion,
we focus here on the 0-1 loss for which the risk corresponds to misclassification.
We consider surrogate losses in \cref{sec:landscape,sec:finite:sample}.}
\begin{equation} \label{eq:stdrisk}
  \stdrisk(\hty)
  \coloneqq
  \Pr_{x,y}\big\{ \hty(x) \neq y \big\}
  = \sum_{c \in \clC} \Pr(y=c) \Pr_{x|y=c} \big\{ \hty(x) \neq c \big\}
  = \E_x \Pr_{y|x} \big\{ \hty(x) \neq y \big\}
  .
\end{equation}
To obtain a classifier robust to $\ep$-perturbations with respect to a given norm $\|\cdot\|$,
one can minimize the corresponding robust risk
\begin{equation} \label{eq:robustrisk:l2}
  \robrisk(\hty,\ep,\|\cdot\|)
  \coloneqq
  \Pr_{x,y}
  \big\{
    \exists_{\delta : \|\delta\| \leq \ep}
    \;\;
    \hty(x+\delta) \neq y
  \big\}
  =
  \sum_{c \in \clC}
  \Pr(y=c)
  \Pr_{x|y=c}
  \big\{
    \exists_{\delta : \|\delta\| \leq \ep}
    \;\;
    \hty(x+\delta) \neq c
  \big\}
  .
\end{equation}
The robust risk penalizes errors on $(x,y)$ pairs from the data distribution, as well as on data \emph{after $\ep$-sized perturbations $\delta \in \bbR^p$}.
Furthermore, the robust risk defined in~\eqref{eq:robustrisk:l2} generalizes the standard risk \cref{eq:stdrisk}
since $\stdrisk(\hty) = \robrisk(\hty,0,\|\cdot\|)$.

While minimizing the robust risk has been shown to indeed improve robustness in practice, this approach is not without its drawbacks.  Numerous works have argued that there may be a fundamental tradeoff between robustness and standard test risk \citep[e.g.,][etc]{tsipras2018robustness,su2018robustness} and that generalization after adversarial training requires significantly more data \citep[e.g.,][etc]{schmidt2018adversarially}.  Moreover, whereas the problem of training a deep neural network typically is overparameterized, finding worst-case perturbations of data as in \eqref{eq:robustrisk:l2} is severely underparameterized and therefore this problem does not benefit from the benign optimization landscape of standard training~\cite{soltanolkotabi2018theoretical, zhang2016understanding, arpit2017closer}.  To this end, a growing body of work has sought to analyze the theoretical properties of these tradeoffs to gain a deeper understanding of the fundamental limits of adversarial robustness~\citep[e.g.,][etc]{tsipras2018robustness,DBLP:conf/icml/ZhangYJXGJ19, raghunathan2019adversarial,javanmard2020precise, chen2020more, min2020tcc:arxiv:v1}.

Despite the progress made toward uncovering the tradeoffs inherent to adversarial training, many fundamental questions remain unresolved.
What do adversarially robust classifiers
that minimize the robust risk \cref{eq:robustrisk:l2}
look like in simple settings?
How do they depend on properties of the data distribution
such as class separation and class imbalance,
as well as the choice of perturbation radius $\ep$ and norm $\|\cdot\|$?
How are they affected when surrogate losses are used
or when the classifier is trained from small numbers of samples?  

% !TEX root = ../adv.tex

\definecolor{wong1}{RGB}{  0,  0,  0}
\definecolor{wong2}{RGB}{230,159,  0}
\definecolor{wong3}{RGB}{ 86,180,233}
\definecolor{wong4}{RGB}{  0,158,115}
\definecolor{wong5}{RGB}{240,228, 66}
\definecolor{wong6}{RGB}{  0,114,178}
\definecolor{wong7}{RGB}{213, 94,  0}
\definecolor{wong8}{RGB}{204,121,167}
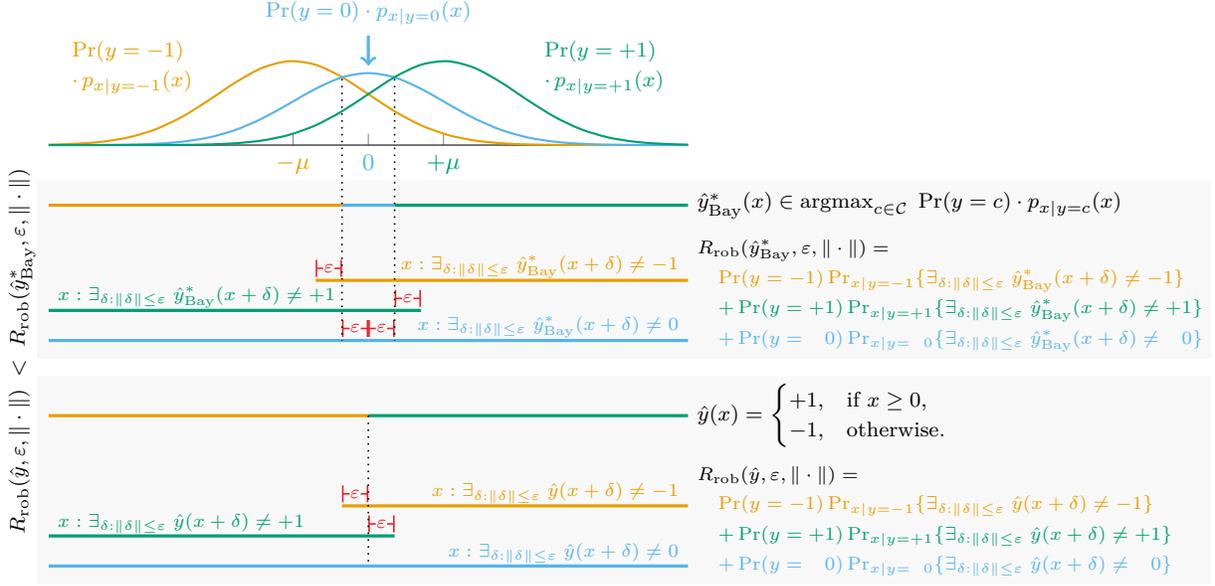
\begin{figure} \centering
  \pgfmathdeclarefunction{gauss}{2}{%
    \pgfmathparse{1/(#2*sqrt(2*pi))*exp(-((x-#1)^2)/(2*#2^2))}%
  }

  \begin{tikzpicture}
  \begin{axis}[
    every axis plot post/.append style={mark=none,domain=-4.25:4.25,samples=25,smooth},
    enlargelimits = false,
    axis x line* = middle,
    axis y line  = none,
    xtick = {-1,0,1},
    xticklabels = {\small\textcolor{wong2}{$-\mu$},\small\textcolor{wong3}{$0$},\small\textcolor{wong4}{$+\mu$}},
    ymin = 0,
    ymax = 0.45,
    x = 1cm,
    y = 4cm,
    clip = false,
  ]
    \addplot[thick,wong2] {0.35*gauss(-1,1)*2}
      node[above left, pos=0.24] {\footnotesize \shortstack[l]{$\Pr(y=-1)$\\$\cdot\, p_{x|y=-1}(x)$}};
    \addplot[thick,wong3] {0.30*gauss( 0,1)*2} coordinate[pos=0.50] (zero);
    \draw[<-,very thick,wong3] ([yshift=1mm]zero) -- ++(axis direction cs:0,0.1)
      node[above] {\footnotesize $\Pr(y= 0) \cdot p_{x|y= 0}(x)$};
    \addplot[thick,wong4] {0.35*gauss( 1,1)*2}
      node[above right, pos=0.76] {\footnotesize \shortstack[l]{$\Pr(y=+1)$\\$\cdot\, p_{x|y=+1}(x)$}};
    
    %% Bayes classifier
    \filldraw[opacity=0.03] (axis cs:-4.4,-0.12) rectangle (11.20,-0.71);

    \draw[wong2,line width=1.2pt] (axis cs:-4.25,-0.20) -- (axis cs:-0.35,-0.20);
    \draw[wong3,line width=1.2pt] (axis cs:-0.35,-0.20) -- (axis cs: 0.35,-0.20);
    \draw[wong4,line width=1.2pt] (axis cs: 0.35,-0.20) -- (axis cs: 4.25,-0.20);
    \node[anchor=west] at (axis cs: 4.25,-0.20)
      {\footnotesize $\bayesclass(x) \in \argmax_{c \in \clC} \; \Pr(y=c) \cdot p_{x|y=c}(x)$};

    \node[anchor=west] at (axis cs: 4.25,-0.35) {\scriptsize $\robrisk(\bayesclass,\ep,\|\cdot\|) =$};

    \draw[wong2,line width=1.2pt] (axis cs:-0.70,-0.45) -- (axis cs: 4.25,-0.45);
    \node[anchor=south east,yshift=-2.5pt,wong2] at (axis cs: 4.25,-0.45)
      {\scriptsize $x : \exists_{\delta : \|\delta\| \leq \ep} \; \bayesclass(x+\delta) \neq -1$};
    \draw[Red,|-,semithick] (axis cs:-0.70,{-0.45+0.04}) -- (axis cs: {-0.70+0.1},{-0.45+0.04});
    \draw[Red,|-,semithick] (axis cs:-0.35,{-0.45+0.04}) -- (axis cs: {-0.35-0.1},{-0.45+0.04});
    \node[Red] at (axis cs: {((-0.70)+(-0.35))/2},{-0.45+0.04}) {\scriptsize $\ep$};
    \node[anchor=west,wong2] at (axis cs: 4.25,-0.45)
      {\scriptsize $\quad  \Pr(y=          -1) \Pr_{x | y=          -1} \{\exists_{\delta : \|\delta\| \leq \ep} \; \bayesclass(x+\delta) \neq           -1\}$};

    \draw[wong4,line width=1.2pt] (axis cs: 0.70,-0.55) -- (axis cs:-4.25,-0.55);
    \node[anchor=south west,yshift=-2.5pt,wong4] at (axis cs:-4.25,-0.55)
      {\scriptsize $x : \exists_{\delta : \|\delta\| \leq \ep} \; \bayesclass(x+\delta) \neq +1$};
    \draw[Red,|-,semithick] (axis cs: 0.35,{-0.55+0.04}) -- (axis cs: { 0.35+0.1},{-0.55+0.04});
    \draw[Red,|-,semithick] (axis cs: 0.70,{-0.55+0.04}) -- (axis cs: { 0.70-0.1},{-0.55+0.04});
    \node[Red] at (axis cs: {(( 0.35)+( 0.70))/2},{-0.55+0.04}) {\scriptsize $\ep$};
    \node[anchor=west,wong4] at (axis cs: 4.25,-0.55)
      {\scriptsize $\quad+ \Pr(y=          +1) \Pr_{x | y=          +1} \{\exists_{\delta : \|\delta\| \leq \ep} \; \bayesclass(x+\delta) \neq           +1\}$};

    \draw[wong3,line width=1.2pt] (axis cs:-0.00,-0.65) -- (axis cs:-4.25,-0.65);
    \draw[wong3,line width=1.2pt] (axis cs: 0.00,-0.65) -- (axis cs: 4.25,-0.65);
    \node[anchor=south east,yshift=-2.5pt,wong3] at (axis cs: 4.25,-0.65)
      {\scriptsize $x : \exists_{\delta : \|\delta\| \leq \ep} \; \bayesclass(x+\delta) \neq 0$};
    \draw[Red,|-,semithick] (axis cs:-0.35,{-0.65+0.04}) -- (axis cs: {-0.35+0.1},{-0.65+0.04});
    \draw[Red,|-,semithick] (axis cs:-0.00,{-0.65+0.04}) -- (axis cs: {-0.00-0.1},{-0.65+0.04});
    \node[Red] at (axis cs: {((-0.35)+(-0.00))/2},{-0.65+0.04}) {\scriptsize $\ep$};
    \draw[Red,|-,semithick] (axis cs: 0.00,{-0.65+0.04}) -- (axis cs: { 0.00+0.1},{-0.65+0.04});
    \draw[Red,|-,semithick] (axis cs: 0.35,{-0.65+0.04}) -- (axis cs: { 0.35-0.1},{-0.65+0.04});
    \node[Red] at (axis cs: {(( 0.00)+( 0.35))/2},{-0.65+0.04}) {\scriptsize $\ep$};
    \node[anchor=west,wong3] at (axis cs: 4.25,-0.65)
      {\scriptsize $\quad+ \Pr(y=\phantom{+}0) \Pr_{x | y=\phantom{+}0} \{\exists_{\delta : \|\delta\| \leq \ep} \; \bayesclass(x+\delta) \neq \phantom{+}0\}$};

    \draw[semithick,dotted] (axis cs:-0.35,{0.11*2}) -- (axis cs:-0.35,-0.65);
    \draw[semithick,dotted] (axis cs: 0.35,{0.11*2}) -- (axis cs: 0.35,-0.65);    

    %% Classifier removing zero class
    \filldraw[opacity=0.03] (axis cs:-4.4,-0.77) rectangle (11.20,-1.46);

    \draw[wong2,line width=1.2pt] (axis cs:-4.25,-0.90) -- (axis cs:-0.00,-0.90);
    \draw[wong4,line width=1.2pt] (axis cs: 0.00,-0.90) -- (axis cs: 4.25,-0.90);
    \node[anchor=west] at (axis cs: 4.25,-0.90)
      {\footnotesize $\hty(x) = \begin{cases} +1 , & \text{if } x \geq 0 , \\ -1 , & \text{otherwise} . \end{cases}$};

    \node[anchor=west] at (axis cs: 4.25,-1.10) {\scriptsize $\robrisk(\hty,\ep,\|\cdot\|) =$};

    \draw[wong2,line width=1.2pt] (axis cs:-0.35,-1.20) -- (axis cs: 4.25,-1.20);
    \node[anchor=south east,yshift=-2.5pt,wong2] at (axis cs: 4.25,-1.20)
      {\scriptsize $x : \exists_{\delta : \|\delta\| \leq \ep} \; \hty(x+\delta) \neq -1$};
    \draw[Red,|-,semithick] (axis cs:-0.35,{-1.20+0.04}) -- (axis cs: {-0.35+0.1},{-1.20+0.04});
    \draw[Red,|-,semithick] (axis cs:-0.00,{-1.20+0.04}) -- (axis cs: {-0.00-0.1},{-1.20+0.04});
    \node[Red] at (axis cs: {((-0.35)+(-0.00))/2},{-1.20+0.04}) {\scriptsize $\ep$};
    \node[anchor=west,wong2] at (axis cs: 4.25,-1.20)
      {\scriptsize $\quad  \Pr(y=          -1) \Pr_{x | y=          -1} \{\exists_{\delta : \|\delta\| \leq \ep} \; \hty(x+\delta) \neq           -1\}$};

    \draw[wong4,line width=1.2pt] (axis cs: 0.35,-1.30) -- (axis cs:-4.25,-1.30);
    \node[anchor=south west,yshift=-2.5pt,wong4] at (axis cs:-4.25,-1.30)
      {\scriptsize $x : \exists_{\delta : \|\delta\| \leq \ep} \; \hty(x+\delta) \neq +1$};
    \draw[Red,|-,semithick] (axis cs: 0.00,{-1.30+0.04}) -- (axis cs: { 0.00+0.1},{-1.30+0.04});
    \draw[Red,|-,semithick] (axis cs: 0.35,{-1.30+0.04}) -- (axis cs: { 0.35-0.1},{-1.30+0.04});
    \node[Red] at (axis cs: {(( 0.00)+( 0.35))/2},{-1.30+0.04}) {\scriptsize $\ep$};
    \node[anchor=west,wong4] at (axis cs: 4.25,-1.30)
      {\scriptsize $\quad+ \Pr(y=          +1) \Pr_{x | y=          +1} \{\exists_{\delta : \|\delta\| \leq \ep} \; \hty(x+\delta) \neq           +1\}$};

    \draw[wong3,line width=1.2pt] (axis cs:-4.25,-1.40) -- (axis cs: 4.25,-1.40);
    \node[anchor=south east,yshift=-2.5pt,wong3] at (axis cs: 4.25,-1.40)
      {\scriptsize $x : \exists_{\delta : \|\delta\| \leq \ep} \; \hty(x+\delta) \neq  0$};
    \node[anchor=west,wong3] at (axis cs: 4.25,-1.40)
      {\scriptsize $\quad+ \Pr(y=\phantom{+}0) \Pr_{x | y=\phantom{+}0} \{\exists_{\delta : \|\delta\| \leq \ep} \; \hty(x+\delta) \neq \phantom{+}0\}$};

    \draw[semithick,dotted] (axis cs: 0.00,-0.90) -- (axis cs: 0.00,-1.40);

    \node[anchor=south,rotate=90] at (axis cs:-4.3,-0.685) {\footnotesize $\robrisk(\hty,\ep,\|\cdot\|) \; < \; \robrisk(\bayesclass,\ep,\|\cdot\|)$};
  \end{axis}
  \end{tikzpicture}
  \caption{Illustration of differences between the standard and robust risk.
    The Bayes classifier $\bayesclass$ minimizes the standard risk
    by maximizing $\Pr(y=c) \cdot p_{x|y=c}(x)$ for each $x$ pointwise,
    so it assigns a nontrivial interval around $x=0$ to the zero class.
    However, it has worse robust risk than an alternative $\hty$ that drops the zero class.
    Minimizing the robust risk does not reduce to making optimal pointwise decisions.
  }
  \label{fig:illustration:threeclass}
\end{figure}

Resolving these issues is complicated by the fact that the robust risk \cref{eq:robustrisk:l2} is significantly more challenging to minimize than the standard risk \cref{eq:stdrisk}.
Indeed, in the standard, non-robust setting,
much of our understanding stems from knowing the optimal classifier, which minimizes the standard risk.
As is well known, e.g., \cite[pg. 216]{anderson1958introduction},
minimizing the standard risk
reduces to making an optimal \emph{pointwise} choice for each $x \in \bbR^p$.  In general the minimizer is given by the Bayes optimal classifier
\begin{equation}
  \bayesclass(x) \in \argmax_{c \in \clC} \; \Pr_{y|x}(y=c)
  = \argmax_{c \in \clC} \; \Pr(y=c) \cdot p_{x|y=c}(x)
  . \label{eq:bayes-opt}
\end{equation}
Unfortunately, an analogous technique has not yet been found for minimizing the robust risk and for deriving expressions for optimal robust classifiers.  This is largely because---unlike minimizing the standard risk---minimizing the robust risk does not reduce to making pointwise decisions depending on the data distribution at each given point individually.  

To elucidate the differences between minimizing the standard and robust risks, consider the simple yet fundamental setting of the classification of points drawn from Gaussian distributions.  In particular, suppose that each of three classes is distributed according to a Gaussian distribution, $\clN(-\mu,1)$, $\clN(0,1)$ and $\clN(\mu,1)$ respectively,
with respective proportions $\Pr(c=-1) = \Pr(c=+1) = 0.35$ and $\Pr(c=0) = 0.30$,
as shown in \cref{fig:illustration:threeclass}.
Deciding how to optimally classify a point like $x=0$ is trivial when minimizing the standard risk.
One simply compares $\Pr(y=c) \cdot p_{x|y=c}(0)$ across $c \in \{-1,0,1\}$
and selects the class corresponding to the largest term
to obtain the Bayes classification of $\bayesclass(0) = 0$.
To minimize the robust risk, however,
one must also consider the behavior of the classifier on the entire $\ep$-neighborhood of $x$.
In this case, it turns out that dropping the zero class altogether engenders a more robust classifier, meaning that
a classifier that assigns $\hty(0) = +1$ has smaller robust risk than the Bayes optimal classifier defined in~\eqref{eq:bayes-opt}.  In this way, minimizing the robust risk does not reduce to a problem depending on the pointwise densities like for the standard risk.

This fundamental difference between the standard and robust risks
means that new techniques are needed for deriving optimal robust classifiers.
Even for the two-class Gaussian classification setting, while it is well-known that linear classifiers minimize the standard risk,
it is not immediately clear that an analogous result holds when minimizing the robust risk.  

To address challenges of this type, in this paper we provide new insights and understanding
by deriving optimal robust classifiers
in the fundamental setting of imbalanced Gaussian distributions.
This precise characterization allows us to \revone{rigorously} investigate
the questions enumerated above,
and moreover reveals \emph{fundamental tradeoffs} that arise
between standard and robust classification.
\revone{Namely, we show that in this foundational setting,
a tradeoff between standard and robust classification arises
not merely because we have not yet managed to find
a classifier (from some appropriately chosen family)
that minimizes both risks.
Rather, no such classifier exists in general.
The tradeoff holds no matter what training methods are used,
how much computation is available,
how many training data points are available,
or what hypothesis class is chosen.}
\revone{An additional feature of our analysis
is that}
the simplicity of a Gaussian mixture
makes it easier to interpret and reason about the results,
helping build intuition about adversarially robust learning.

\noindent\textbf{Contributions.} Our contributions are as follows:

\begin{compactenum}
    \item We find the \emph{optimal robust classifiers} for the foundational setting of two- and three-class imbalanced Gaussian classification with respect to $\ell_2$ and $\ell_\infty$ norm-bounded adversaries in the \emph{imbalanced data setting}. These were previously unknown and involve new fundamental challenges.
    To tackle this problem, we develop a new proof method, carefully combining the characterization of exact Gaussian isoperimetry---i.e., the equality case of Gaussian concentration of measure---with the Neyman-Pearson lemma and Fisher's linear discriminant.\footnote{While Gaussian distributions are in some sense specific, they are also fundamental and broadly applicable. Indeed, many datasets are well-approximated by Gaussian distributions, see textbooks such as \cite{hastie2009elements,bishop2006pattern}, including even data distributions generated by GANs \citep{seddik2020random}.} 
    This introduces a \emph{new and nontrivial theoretical approach}.
    \item We use these results to identify the role of \emph{class imbalance} for tradeoffs between accuracy and robustness for two and three-class $\ell_2$ norm-bounded adversaries.
    In particular, we uncover \emph{fundamental distinctions between balanced and imbalanced classes}.
    Balanced classes experience no tradeoff with respect to the Bayes risk:
    the optimal non-robust classifier also turns out to minimize the robust risk.
    However, an unavoidable tradeoff appears in the imbalanced case:
    the boundary measure of the larger class expands and gets larger weight,
    so the optimal boundary necessarily moves.
    Thus, no classifier simultaneously minimizes both standard and robust risk;
    the optimal \emph{non-robust} classifier is not the optimal \emph{robust} classifier.

    \item We show how the optimal robust classifiers relate to previously proposed randomized classifiers. Additionally, we characterize optimal robust classifiers for data with more general covariances and with low-dimensional structure.

    \item We further show that in certain cases \emph{all robust classifiers are approximately linear},
    characterizing all approximately optimal robust classifiers. This requires a novel approach, which leverages recent breakthroughs from \emph{robust isoperimetry} \citep{cianchi2011isoperimetric,mossel2015robust} and represents one of our main technical innovations. 

    \item We uncover \emph{surprising phase-transitions arising in three-class robust classification}.
    Deriving optimal classifiers for this setting requires delicate analysis,
    but reveals how the optimal classifier can jump discontinuously
    for small changes in the problem parameters.
    This does not occur in the two-class setting
    or for non-robust classification;
    it arises when combining robustness with a third class.

    \item We provide a more comprehensive understanding by analyzing the \emph{non-convex problem of optimizing the robust risk}
      for linear classifiers.
      Specifically, 
      we characterize broad settings where classification calibration holds for convex surrogate losses, so the optimizers of surrogate losses coincide with the optimizers of the original objective.
    \item We connect our findings to empirical robust risk minimization by providing a \emph{finite-sample analysis}
      with respect to 0-1 and surrogate loss functionals,
      which also highlights the key role of geometry in convergence rates.
      This analysis does not rely on Gaussianity.
\end{compactenum}

The remainder of this paper is organized as follows.  In Section~\ref{sec:related-work}, we review related work concerning algorithms and analysis techniques for adversarial robustness.  Next, \revone{after some preliminaries in \cref{sec:notation:prelim}, Section~\ref{sec:opt:l2:twoclass} derives} $\ell_2$ robust optimal classifiers for two-class Gaussian classification.  In Section~\ref{sec:opt:l2:threeclass}, we \revone{tackle} three-class Gaussian classification.  Following this, in Section~\ref{sec:opt:linf}, we derive $\ell_\infty$ robust optimal linear classifiers for two-class Gaussian classification models.  In Section~\ref{sec:landscape}, we study the optimization landscape of the robust risk.  Finally, in Section~\ref{sec:finite:sample}, we connect the results derived in the previous sections to empirical risk minimization by providing a finite sample analysis under broader distributional assumptions.

% !TEX root = ../adv.tex

\section{Related work} \label{sec:related-work}

Adversarial robustness is a very active and rapidly expanding area of research.  For this reason, we can only review a collection of some of the most closely related works.

\textbf{Robustness of linear models.}  Recently,  several papers have studied the robustness of linear models.  \citep{xu2009robustness} shows that certain robust support vector machines (SVM) are equivalent to regularized SVM. They also give bounds on the standard generalization error based on the regularized empirical hinge risk. \citep{xu2009robust} shows equivalences between adversarially robust regression and lasso.
\citep{megyeri2019adversarial} studies the adversarial robustness of linear models, arguing that random hyperplanes are very close to any data point and that robustness requires strong regularization.  Furthermore, the two related works of \citep{chen2020more} and \citep{min2020tcc:arxiv:v1} study Gaussian and Bernoulli models for data and analytically establish a variety of phenomena regarding robust accuracy and the generalization gap for linear models.  They conclude that more data may actually increase the generalization gap.  In this paper, we consider the distinct yet related problem of characterizing optimal classifiers in the population setting rather than determining the finite-sample generalization gap.  Finally, the recent results in \cite{bhattacharjee2020sample} analyze the sample complexity of training adversarially robust linear classifiers on separated data.  We study a similar problem, but in our case the data is not separated.

\textbf{Randomized smoothing.}  The connection between Gaussianity and robustness is one of the key ideas behind randomized smoothing \citep{cohen2019certified,salman2019provably,blum2020random,mohapatra2020rethinking}.  While these works provide an interesting and useful alternative to adversarial training, they have little theoretical overlap with this paper.  In this paper, we study the different problem of deriving classifiers that are optimal with respect to the robust risk.

\textbf{Generalized likelihood ratio testing.}
\citep{puranik2021arc} proposes another interesting and useful alternative to adversarial training.
The approach is based on the generalized likelihood ratio test (GLRT)
and can be applied to general multi-class Gaussian settings.
It remains distinct from this paper
since it does not seek to optimize the robust risk,
instead utilizing a GLRT.

\textbf{Concentration based analyses.}  Several works use various forms of concentration of measure to explain the existence of adversarial examples in high dimensions \citep{gilmer2018trb:arxiv:v3,shafahi2018are}.  Relatedly, \citep{mahloujifar2019empirically} proposes methods for empirically measuring concentration and establishing fundamental limits on intrinsic robustness.  These analyses and empirical results generally rely on concentration on the sphere $S^{p-1}$, whereas we rely on Gaussian concentration in this paper.  More related is the work of  \citep{richardson2020bayes}, which studies the adversarial robustness of Bayes-optimal classifiers in two-class Gaussian classification problems with unequal covariance matrices $\Sigma_1$ and $\Sigma_2$. For instance, when the covariance matrices are strongly asymmetric, so that the smallest eigenvalue of one class tends to zero, they show that almost all points from that class are close to the optimal decision boundary. In contrast, in the symmetric isotropic case, $\Sigma_1=\Sigma_2=\sigma^2 I_p$ and $\sigma\to 0$, they show that with high probability all points in both classes are at distance $p/2$ from the boundary.  While this is consistent with a portion of our findings, we focus on different problems, namely finding the optimal robust classifiers.

\textbf{Trade-offs in adversarial robustness.}  Many works have argued that there are inherent trade-offs between standard and robust accuracy \citep{raghunathan2019adversarial,raghunathan2020understanding,hayes2020provable}.  Among these works, several consider Gaussian models of data.   Notably,  \citep{schmidt2018adversarially} studies the two-class Gaussian classification problem $x\sim \N(y\mu, \sigma^2 I_p)$, focusing on the balanced case $\pi=1/2$, and on signal vectors $\mu$ of norm approximately $\sqrt{p}$. In this setting, unlike in ours, it is possible to construct accurate classifiers even from one training data point $(x_1,y_1)$. They show that such classifiers can have high standard accuracy, but low robust accuracy. In contrast, we attack the more challenging problem of deriving \emph{closed-form} expressions for optimal classifiers without strong assumptions on the signal strength.

\revone{In~\citep{dan2020sharp}, the authors study} a two-class Gaussian classification problem with balanced classes.  They develop sharp minimax bounds on the classification excess risk with a corresponding estimator.
In contrast, we derive optimal classifiers and study trade-offs for the more general, imbalanced-class setting.   \revone{Similarly,~\citep{fawzi2018analysis} studies robustness defined as the average of the norm of the smallest
perturbation that switches the sign of a classifier $f$. They consider labels that are a deterministic function of the datapoints, which differs from our setting.}

\revone{Another related work is that of~\citep{tsipras2018robustness}, in which the authors consider} two-class Gaussian classification where $x=y\cdot (b, \eta 1_p)+\N(0_p,\diag(0,1_{p-1}))$, and $b$ is a random sign variable with $P(b=1)=q\ge1/2$, while $\eta$ is a constant. Thus, the first variable contains the correct class $y$ with probability $q$, while the remaining ``non-robust'' features contain a weak correlation with $y$. Our models are related, but distinct from this model. Their work is closest to our results on the optimal robust classifiers for $\ell_\infty$ perturbations, which are given by soft-thresholding the mean. This will not use the non-robust features, which is consistent with \citep{tsipras2018robustness}. However, their results are different, as they emphasize the robustness-accuracy tradeoff (Theorem 2.1), while we characterize the optimal robust classifiers.

Aside from the trade-off between accuracy and adversarial robustness, a growing body of work has focused on other naturally-arising trade-offs in various problem settings.  Among such works, \citep{sanyal2020benign} studies adversarial robustness in the presence of label noise and explores its relationship to benign overfitting \citep{bartlett2020benign}.  \citep{mehrabi2021fundamental} studies trade-offs in the distributionally robust setting, where robustness is defined with respect to a family of related data distributions.  Finally, in \citep{tramer2020fundamental}, the authors analyze the trade-off between invariance and sensitivity to adversarial examples.  While each of these studies considers trade-offs in adversarially robust machine learning, the settings and results are different from our setting.

\textbf{Distribution-agnostic results.}  One notable recent direction is to study the properties of adversarial learning problems in a distribution-agnostic setting.  Among such works, \citep{cullina2018pac} introduces the ``adversarial VC-dimension'' to study the statistical properties of PAC learning in the presence of an adversary.  The authors extend the fundamental theorem of statistical learning theory to this setting, and provide sample complexity bounds for this distribution-agnostic setting.  These results gave rise to a line of work focusing specifically on the distribution-agnostic setting (see\revone{,} e.g.\revone{,} \citep{yin2019rademacher,attias2019improved,awasthi2020adversarial}). Building on this, \citep{montasser2020efficiently} proposes methods for efficiently PAC learning adversarially robust halfspaces with noise.  By and large, due to the distribution-agnostic assumption, the PAC-style results from these papers are more conservative and distinct from the results we obtain for the Gaussian setting.  However, in very special cases, such as learning in the presence of \emph{random} (e.g.\revone{,} non-adversarial) classification noise, the representation of the risk in terms of the dual norm in \citep{montasser2020efficiently} agrees with our characterization.

\revone{\textbf{Calibration of the adversarial loss.}  Also of note is a recent line of work which considers the calibration of the robust 0-1 loss.  Concretely, a loss is calibrated with respect to a given function class $\mathcal{F}$ if minimizing the excess risk with respect to a surrogate loss over $\mathcal{F}$ implies minimization of the target risk.  Following~\cite{steinwart2007compare}, both~\cite{awasthi2021finer} and~\cite{bao2020calibrated} consider the calibration of the robust 0-1 loss, showing both positive and negative calibration results for a variety of surrogate losses and function classes.  In contrast to these works, the majority of our main results (see Sections~\ref{sec:opt:l2:twoclass},~\ref{sec:opt:l2:threeclass}, and~\ref{sec:opt:linf}) focus directly on minimizing the robust 0-1 loss, rather than minimizing a surrogate loss.  Furthermore, we note that the calibration results presented in Section~\ref{sec:landscape} of this paper are complementary to the results in~\cite{awasthi2021finer,bao2020calibrated} which show that convex surrogates are in general not calibrated for the robust 0-1 loss.  Specifically, we show that surrogate losses can be calibrated in this setting under stronger conditions than convexity.  Also related is the work of~\cite{bhagoji2021lower}, in which the authors derive lower bounds on the cross-entropy loss under adversarially-chosen perturbations.  This differs from our setting, as we do not consider the cross-entropy loss in this work.} 

\textbf{Robustness in non-parametric settings.}  While we study a parametric setting in this paper, there are several notable works that study similar problems in the non-parametric setting.  \citep{wang2018analyzing} analyzes the robustness of nearest neighbor methods to adversarial examples. In this work, the authors introduce and study quantities called ``astuteness'' and ``$r$-optimality'', which are intimately related to the robust risk.  These ideas led to further related studies of attacks and defenses in the non-parametric setting (see\revone{,} e.g.\revone{,} \citep{bhattacharjee2020non,yang2020robustness}).  In \cite{yang2020robustness}, the robust optimal classifier for a non-parametric setting is determined to be the solution of a particular optimization problem.  These results are incomparable with ours, since a mixture of Gaussians is not $r$-separated, and truncating the Gaussians to satisfy $r$-separation would lead to a large test error.  Furthermore, we \emph{explicitly} derive closed-form expressions for optimal robust classifiers in our setting.  More related to our paper is the work of \citep{yang2020adversarial}, in which the authors study robustness through local-Lipschitzness.  However, whereas we seek to find statistically optimal robust classifiers in the classical two- and three-class Gaussian setting, \citep{yang2020adversarial} considers a completely different model of data.

\textbf{Optimal transport based analyses.}  Most related to this paper is the recent work of \citep{bhagoji2019lower}, in which the authors develop a framework connecting adversarial risk to optimal transport. As a special case, for balanced two-class Gaussian classification problems with $\pi=1/2$ and $x_i|y_i \sim \N(\mu y_i,\Sigma)$, and for general perturbations in a closed, convex, absorbing and origin-symmetric set $B$, they show that linear classifiers are optimal, and characterize these optimal classifiers.  Similarly, \citep{pydi2019adversarial} also characterizes optimal classifiers in various settings; in particular, they focus on the balanced case $\pi=1/2$, e.g., for two classes with spherical covariances $\N(\mu_i, \sigma^2 I_p)$ or in 1-D with different means and covariances $\N(\mu_i, \sigma_i^2)$. Complementary to these two works, we focus on identifying trade-offs for \emph{imbalanced} classes.  Furthermore, our analyses rely on entirely different proof techniques, and it is unclear whether our results can be obtained from optimal transport theory, or whether results in the imbalanced setting can be derived using proof techniques which use optimal transport.  Relatedly, \citep{dohmatob2019limitations} provide lower bounds on the adversarial risk for certain multi-class classification problems
whose data distributions satisfy the $W_2$ Talagrand transportation-cost inequality.
In contrast, we find the optimal classifiers for the two- and three-class Gaussian classification problems.    

% !TEX root = ../adv.tex

\section{Notations and preliminaries}
\label{sec:notation:prelim}

Before we begin, it will be helpful to define some notation.
We denote the ball of radius $\ep$
(with respect to the norm $\norm{\cdot}$)
centered at the origin
by $B_{\ep}$,
and the indicator function of a set $A$ by $I(A)$.
Further, if $A$ and $B$ are sets,
then we use the notation $A + B = \{a + b : a\in A, b\in B\}$
for the Minkowski sum;
when $A = \{a\}$ contains a single element,
we abbreviate it to $a + B$.
In these terms,
the robust risk with 0-1 loss \cref{eq:robustrisk:l2}
has another convenient form
that we use heavily in the proofs:
\begin{equation}
  \robrisk(\hty,\ep,\|\cdot\|)
  =
  \E_y
  \Pr_{x|y}
  \big\{
    \exists_{\delta : \|\delta\| \leq \ep}
    \;\;
    \hty(x+\delta) \neq y
  \big\}
  = \E_y \Pr_{x|y} \big\{ S_y^c(\hty) + B_\ep \big\}
  ,
\end{equation}
where $S_y^c(\hty) \coloneqq \{x : \hty(x) \neq y\}$
is the \emph{misclassification set} of classifier $\hty$ for class $y$,
and
\begin{equation*}
  S_y^c(\hty) + B_\ep
  = \{ x + \delta : \hty(x) \neq y \text{ and } \|\delta\| \leq \ep \}
  = \{ x : \exists_{\|\delta\| \leq \ep} \;\; \hty(x+\delta) \neq y \}
\end{equation*}
is the corresponding \emph{robust misclassification set},
illustrated for a single class by the following diagram.
\begin{center}
  \pgfmathdeclarefunction{gauss}{2}{%
    \pgfmathparse{1/(#2*sqrt(2*pi))*exp(-((x-#1)^2)/(2*#2^2))}%
  }

  \begin{tikzpicture}
  \begin{axis}[
    every axis plot post/.append style={mark=none,domain=-3:3,samples=50,smooth},
    enlargelimits = false,
    axis x line* = middle,
    axis y line  = none,
    xtick = {0},
    xticklabels = {$\mu$},
    ymin = 0,
    ymax = 0.45,
    x = 1cm,
    y = 4cm,
    clip = false
  ]
    \addplot[color=black] {gauss(0,1)}
      node[above right,pos=0.65] {$\mathcal{N}(\mu,1)$};

    \filldraw[OliveGreen] (axis cs:-0.75,0) rectangle (axis cs:-0.5,0.015);
    \filldraw[OliveGreen] (axis cs:-3.00,0) rectangle (axis cs:-1.5,0.015)
      node[OliveGreen,above right,at start] {$S_y^c$};

    \filldraw[Red] (axis cs:-1.00,0) rectangle (axis cs:-0.25,-0.015);
    \filldraw[Red] (axis cs:-3.00,0) rectangle (axis cs:-1.25,-0.015)
      node[Red,below right,at start] {$S_y^c + B_\ep$};

    \draw[|-|] (axis cs:-1.00,0.05) -- (axis cs:-0.75,0.05)
      node[above,midway] {$\ep$};
  \end{axis}
  \end{tikzpicture}
\end{center}
Note that $S_y(\hty) \coloneqq \{x : \hty(x) = y\}$ for $y \in \clC$ are, correspondingly,
the classification sets or decision regions of the classifier $\hty$.

% !TEX root = ../adv.tex

\section{Optimal \texorpdfstring{$\ell_2$}{l2} robust classifiers for two classes}
\label{sec:opt:l2:twoclass}

This section considers the fundamental binary classification setting where data is distributed as a Gaussian for each of the classes $y\in \clC = \{+1,-1\}$:
\oneortwo{%
\begin{align} \label{eq:model:twoclass}
  x|y &\sim \clN(y\mu, \sigma^2 I_p)
  , &
  y &=
  \begin{cases}
    +1 & \text{with probability } \pi
    , \\
    -1 & \text{with probability } 1-\pi
    ,
  \end{cases}
\end{align}
}{%
\begin{align} \label{eq:model:twoclass}
  x|y &\sim \clN(y\mu, \sigma^2 I_p), 
\end{align}
}%
where
$\mu \in \bbR^p$ specifies the class means ($+\mu$ and $-\mu$, $\mu\neq 0$),
$\sigma^2 \in \bbR_{>0}$ is the within-class variance,
and $\pi \in [0,1]$ is the proportion of the $y=1$ class.
The means are centered at the origin without loss of generality.
By scaling, we will also take $\sigma^2=1$ to simplify exposition.

In this setting, the Bayes optimal (non-robust) classifier is linear; in particular, the expression for this classifier is given by the following pointwise calculation:
\begin{equation*} \textstyle
  \bayesclass(x) = \argmax_{c \in \clC} \Pr_{y|x}(y=c)
  = \sign(x^\top \mu - q/2)
  ,
\end{equation*}
where $q \coloneqq \ln\{(1-\pi)/\pi\}$ is the log-odds-ratio and we define $\ln(0) \coloneqq -\infty$ and $\sign(0) = 1$.
The classifier is unaffected by any positive rescaling
of the argument of $\sign$.
Denoting the normal cumulative distribution function
$\Phi(x) \coloneqq (2\pi)^{-1/2} \int_{-\infty}^x \exp(-t^2/2) dt$
and $\brPhi \coloneqq 1 - \Phi$,
the corresponding Bayes risk
\oneortwo{%
\begin{equation} \label{eq:twoclass:bayes}
  \bayesrisk(\mu, \pi)
  \coloneqq \stdrisk(\bayesclass)
  =
  \pi
  \cdot
  \Phi  \left(\frac{q}{2\|\mu\|_2} - \|\mu\|_2\right)
  +
  (1-\pi)
  \cdot
  \brPhi\left(\frac{q}{2\|\mu\|_2} + \|\mu\|_2\right)
  ,
\end{equation}
}{%
\begin{align} \label{eq:twoclass:bayes}
  &\bayesrisk(\mu, \pi)
  \coloneqq \stdrisk(\bayesclass)
  \\
  &=
  \pi
  \,
  \Phi  \left[\frac{q}{2\|\mu\|_2} - \|\mu\|_2\right]
  +
  (1-\pi)
  \,
  \brPhi\left[\frac{q}{2\|\mu\|_2} + \|\mu\|_2\right]
  , \nonumber
\end{align}
}%
is the smallest attainable \emph{standard} risk and characterizes the problem difficulty.

\subsection{Optimal classifiers with respect to the robust risk}

With this background in mind,
we now turn to our problem of finding Bayes-optimal \emph{robust} classifiers.
Unlike the non-robust setting, one can no longer simply make optimal pointwise decisions for each $x \in \bbR^p$ depending only on the data distribution at $x$,
because the robust risk is also affected by neighboring datapoints and decisions.
Thus, it is not initially obvious how to find provably optimal robust classifiers.
Moreover, it is not initially clear how such classifiers might differ from the Bayes-optimal (non-robust) classifier $\bayesclass$,
especially in the presence of class imbalance.

The following theorem provides a precise characterization of optimal robust classifiers.
Its proof involves a novel approach that combines
the result that halfspaces are extremal sets with respect to Gaussian isoperimetry
\citep{borell1975brunn,sudakov1978extremal,boucheron2013concentration},
the Neyman-Pearson lemma,
and Fisher's linear discriminant.

\medskip
\begin{theorem}[Optimal $\ell_2$-robust two-class classifiers]
\label{orc}
Suppose the data $(x,y)$ follow the two-class Gaussian model \cref{eq:model:twoclass}
and $\ep < \|\mu\|_2$.
An optimal $\ell_2$ robust classifier
is
\begin{equation} \label{eq:robopt:twoclass}
  \hty^*(x)
  \coloneqq
  \sign\{x^\top \mu(1-\ep/\|\mu\|_2)_+-q/2\}
  ,
\end{equation}
where $q = \ln\{(1-\pi)/\pi\}$ and $(x)_+ = \max(x,0)$.
Moreover, the corresponding optimal robust risk is
\begin{equation} \label{eq:robopt:twoclass:risk}
  \robrisk^*(\mu,\pi;\ep)
  \coloneqq
  \bayesrisk\{\mu(1-\ep/\|\mu\|_2)_+,\pi\}
  ,
\end{equation}
where $\bayesrisk$ is the Bayes risk defined in \cref{eq:twoclass:bayes}.
\end{theorem}

\Cref{orc} reveals several important insights
into the properties of optimal robust classification.

\noindent \textbf{Optimality of linear classifiers.} \Cref{orc} shows the nontrivial result that the Bayes optimal robust classifier \emph{is also linear}. Indeed, it shows that the optimal robust classifier corresponds to the classical (non-robust) Bayes optimal classifier
with a \emph{reduced effective mean} $\mu \mapsto \mu\left(1-\ep/\|\mu\|_2\right)_+$
or, equivalently, an \emph{amplified effective class imbalance}: $q \mapsto q/(1-\ep/\|\mu\|_2)_+$.
Note that if $\ep \ge \|\mu\|_2$, then nontrivial classification is impossible.
The effective signal strength reduction is
consistent with prior arguments
that ``adversarially robust generalization requires more data'' \citep{schmidt2018adversarially}.
However, to the best of our knowledge, such an exact characterization for imbalanced data was previously unknown (see the related work section).

{\bf Optimal Tradeoffs and Pareto-Frontiers.}
When the classes are balanced, i.e., $\pi = 1/2$ (and thus $q=0$),
\cref{orc} shows that
the Bayes optimal classifier $\bayesclass$
and the optimal robust classifier $\hty^*$ \emph{coincide}.
In general, however, there is a \emph{tradeoff}:
neither classifier optimizes both standard and robust risks.
These insights are important since real datasets are often imbalanced. 
Indeed, our analysis (see \cref{thm:opt:twoclass:admissible}) implies that given any classifier, there exists a linear classifier of the form $\text{sign}(x^\top \mu -c)$ with no worse standard risk and no worse robust risk. Using this, we can precisely characterize the \emph{Pareto-frontier (optimal tradeoff)} between the standard risk and the robust risk. 
 Consider a two-dimensional plane in which the $x$-axis represents the robust risk and the $y$-axis represents the standard risk (see \cref{trf:pareto}).  Any classifier can be represented 
as a pair in this plane with its robust risk as the first entry, and  its standard risk as the second 
   entry. We further consider the region of all the possible achievable pairs over all classifiers. The \emph{Pareto-frontier (optimal boundary)} of this region shows the fundamental tradeoff between standard and robust risks. For the setting of \cref{orc} we can precisely characterize the  fundamental tradeoff (Pareto-frontier) between robust and standard risks. Importantly, these tradeoffs hold for any predictor (be it a deep network or a simple linear classifier) including those learned in any way from any size of training data, see \cref{trf:pareto}.

This provides new insights. To the best of our knowledge, this is the first work to illustrate trade-offs due to \emph{class imbalance}, which is prevalent in practice. 
Moreover, this result proves that a \emph{linear} classifier is \emph{optimal} with respect to robust risk; that is, even if one had access to rich model classes such as deep neural networks, massive computational power, or arbitrarily large datasets, this fundamental tradeoff would remain in place.
Such a strong tradeoff (for the Bayes risk) has only been observed for balanced-class settings in a completely separate line of work~\cite{bhagoji2019lower}, as discussed in \cref{sec:related-work}.  Relative to this work, our proof techniques are completely different, and we are also able to map the entire Pareto-frontier for imbalanced classes.

\Cref{trf} illustrates this tradeoff for an example
with a mean having norm $\|\mu\|_2 = 1$,
a positive class proportion $\pi = 0.2$,
and a perturbation radius $\ep = 0.5$.
The first figure plots the two risks
for the linear classifier $\hty(x) = \sign(x^\top\mu-c)$
as a function of the threshold $c$.
This highlights the difference between the two risks
and their corresponding optimal thresholds.
The next figure plots the two risks against each other
for a sweep of the threshold $c$.

\oneortwo{%
\begin{figure} \centering
  \begin{subfigure}{.45\textwidth} \centering
    \includegraphics[scale=0.8]{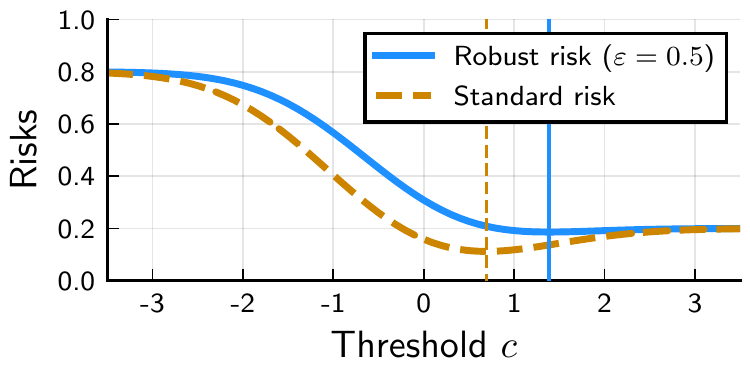}
    \caption{Risks as functions of threshold $c$;
      vertical lines at optimal thresholds.}
  \end{subfigure}
  \;
  \begin{subfigure}{.45\textwidth} \centering
    \includegraphics[scale=0.8]{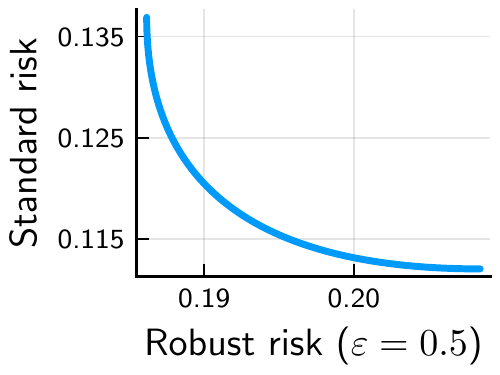}
    \caption{Pareto-frontier: Standard and robust risk plotted against each other as a function of the threshold $c$.}
    \label{trf:pareto}
  \end{subfigure}
\caption{Tradeoffs between optimal classification
  with respect to standard and robust risks.
}
\label{trf}
\end{figure}
}{%
\begin{figure} \centering
  \begin{subfigure}{\linewidth} \centering
    \includegraphics[scale=0.8]{twoclass,thresh.pdf}
    \caption{Risks as functions of threshold $c$;
      vertical lines at minima.}
  \end{subfigure}\\[1em]
  
  \begin{subfigure}{.48\linewidth} \centering
    \includegraphics[scale=0.82,trim=4pt 0 4pt 0]{twoclass,tradeoff,thresh.pdf}
    \caption{Risks for thresholds $c$ in between the minima.}
  \end{subfigure}
\caption{Tradeoffs between standard and robust risks.}
\label{trf}
\end{figure}
}%

\subsection{Proof of \cref{orc}}
We now prove \cref{orc}.
Since we cannot simply optimally classify data points based on the data distributions at each individual point,
new techniques are needed for deriving optimal classifiers with respect to the robust risk.
Here we introduce a novel approach.
First, we prove that there exist optimal linear classifiers
by combining the fact that halfspaces are extremal with respect to Gaussian isoperimetry, i.e., the equality case in the Gaussian concentration of measure
  \citep{borell1975brunn,sudakov1978extremal,boucheron2013concentration},
with the Neyman-Pearson lemma.
Then, we derive the optimal linear classifiers via Fisher's linear discriminant.

Before we begin, note that the robust risk
for the two-class setting can be written as
\begin{align*}
    R(\hat{y}, \ep)
    \coloneqq \robrisk(\hty,\ep,\|\cdot\|)
    &=  \pi \cdot \Pr_{x|y=+1}(S_{-1} + B_{\ep}) +  (1-\pi) \cdot  \Pr_{x|y=-1}(S_{+1} + B_{\ep})
    ,
\end{align*}
where we drop $\hty$ from $S_{-1}$ and $S_{+1}$ for convenience.
This holds for any binary classification problem, and in particular for the two-class Gaussian problem that we consider in this section.

\subsubsection{Optimality of linear classifiers} 
The first step is to prove the following claim:
for any classifier $\hty$,
there exists a linear classifier with robust risk no worse than that of $\hty$.
Precisely put, we prove the following lemma.

\medskip
\begin{lemma}[Existence of optimal linear classifiers]
\label{thm:opt:twoclass:admissible}

Under the assumptions of Theorem \ref{orc}, for any classifier $\hat y$, there exists a linear classifier $\hat y^*(x)=\sign(x^\top w-c)$, for some $w$ and $c$, whose standard risk and robust risk with respect to $\ep$-bounded $\ell_2$ adversaries is less than or equal to that of the original classifier:
\begin{align*}
    R(\hat y^*,\ep) &\leq R(\hat y,\ep)
    , &
    R(\hat y^*,0) &\leq R(\hat y,0)
    .
\end{align*}
Moreover, we can take $w = \mu$.
\end{lemma}

To prove \cref{thm:opt:twoclass:admissible},
we carefully combine the fact that halfspaces are extremal with respect to Gaussian isoperimetry with the Neyman-Pearson lemma
to construct, given any classifier $\hty$, a linear classifier that has robust risk no worse than $\hty$.

\begin{proof}[Proof of \cref{thm:opt:twoclass:admissible}]
  Recall that $x|y\sim \N(y\mu,1)$, and abbreviate $\Pr_{x|y=\pm 1}=\Pr_\pm$.
  Let $\hty$ be an arbitrary classifier
  and denote its decision regions as $S_{\pm1}$. 
  We will show there exists a linear classifier with decision regions $\tilde S_{\pm1}$, such that: 
  \begin{align}\label{obj}
  \begin{split}
  \Pr_+(\tilde S_{-1}+B_\ep) & \le \Pr_+(S_{-1}+B_\ep)\\
  \Pr_-(\tilde S_{1}+B_\ep) & \le \Pr_-(S_{1}+B_\ep)\\
  \Pr_+(\tilde S_{-1}) & = \Pr_+(S_{-1})\\
  \Pr_-(\tilde S_{1}) & = \Pr_-(S_{1}).
  \end{split}
  \end{align}

  Now, the well-known Gaussian concentration of measure (GCM) \citep{borell1975brunn,sudakov1978extremal,boucheron2013concentration}
  states that the sets with minimal ``concentration function'' are the half-spaces. Specifically, the measurable sets solving the problem
  \begin{align}\label{gcmsol}
  \min_S \Pr_\mu(S+B_\ep)\,\,
  s.t.\,\, \Pr_\mu(S)=\alpha
  \end{align}
  are half-spaces (up to measure zero sets).
  
  For simplicity, let us discuss the one-dimensional problem $p=1$. Note that the general $p$ case can be reduced to the one-dimensional case. We can take $w = \mu$ and solve the optimal robust classification problem projected into the 1-dimensional line $\{c\mu,c\in\R\}$. Then, projecting back, we can compute probabilities and distances for the multi-dimensional problem. It is not hard to see that the back-projection of the one-dimensional problem also solves the multi-dimensional problem.
  
  In the one-dimensional case, suppose without loss of generality that $\mu>0$. The GCM states that there is a half-line $\tilde S_{-1} = (-\infty,c]$---which will serve as a new set $\{x:\tilde y^*(x)=-1\}$---such  that
  \begin{align}
  \Pr_+(\tilde S_{-1}+B_\ep) & \le \Pr_+(S_{-1}+B_\ep)\label{gcm1}\\
  \Pr_+(\tilde S_{-1}) & = \Pr_+(S_{-1}).\label{gcm2}
  \end{align}
  Moreover, we have $c = \mu+\Phi^{-1}(\Pr_+(S_{-1}))$.
  
  Similarly, using the GCM symmetrically, we find that there is a half-line $\tilde S_{1} = (d,-\infty)$---which will serve as a new set $\{x:\tilde y^*(x)=1\}$---such  that
  \begin{align*}
  \Pr_-(\tilde S_{1}+B_\ep) & \le \Pr_-(S_{1}+B_\ep)\\
  \Pr_-(\tilde S_{1}) & = \Pr_-(S_{1}).
  \end{align*}
  Now, the question is if the two sets $\tilde S$ can form the classification regions of a classifier. This would be true if they cover the real line. Here, we claim that they overlap. Thus, they can be shrunken to partition the real line, and the values of the objective in \eqref{obj} decrease.
  
  To show that they overlap, it is enough to prove that their total probability under one of the two measures, say $\Pr_-$, is at least unity. Thus we need:
  \begin{align*}
  & \Pr_-(\tilde S_{1}) + \Pr_-(\tilde S_{-1})\ge 1\\
  \iff & \Pr_-(S_{1}) + \Pr_-(\tilde S_{-1})\ge 1\\
  \iff & \Pr_-(\tilde S_{-1})\ge 1-\Pr_-(S_{1})\\
  \iff & \Pr_-(\tilde S_{-1})\ge \Pr_-(S_{-1}).
  \end{align*}
  Now note that $\Pr_+(\tilde S_{-1}) = \Pr_+(S_{-1})$, so the probability of the two sets coincides under $\Pr_+$. Moreover, $\Pr_+ = \N(\mu,1)$, $\Pr_- = \N(-\mu,1)$, $\mu>0$, and $\tilde S_{-1} = (-\infty,c]$. Then, the Neyman-Pearson lemma states that $\tilde S_{-1}$ maximizes the function $S\mapsto \Pr_-(S)$ over measurable sets $S$ (i.e., the power of a hypothesis test of $\Pr_+$ against $\Pr_-$), subject to a fixed value of $\Pr_+(S)$. Therefore, the inequality above is true. This shows that the two sets overlap, and thus finishes the proof.
\end{proof}

\subsubsection{Optimal linear classifiers}
The proof of \cref{orc} now concludes by optimizing among linear classifiers. As in the proof of \cref{thm:opt:twoclass:admissible}, it is enough to solve the one-dimensional problem. Thus, we want to find the value of the threshold $c$ that minimizes
\begin{align*}
  R(\hat y_c,\ep)
  &=
  \Pr(y=1)\Pr_{x|y=1}(x\le c+\ep)
  +\Pr(y=-1)\Pr_{x|y=-1}(x\ge c-\ep)\\
  &=
  \Pr(y=1)\Pr_{\mu}(x\le c+\ep)
  +\Pr(y=-1)\Pr_{-\mu}(x\ge c-\ep)\\
  &=
  \Pr(y=1)\Pr_{\mu-\ep}(x\le c)
  +\Pr(y=-1)\Pr_{-\mu+\ep}(x\ge c).
  \end{align*}
  This is exactly the problem of non-robust classification between two Gaussians with means $\mu' = \mu-\ep$ and $-\mu'$.
  By assumption, $\mu'\ge 0$.

  As is well known, the optimal classifier is Fisher's linear discriminant \citep[see\revone{,} e.g.,][p. 216]{anderson1958introduction}:
  \begin{align*}
  \hat y_\ep^*(x) & =\sign\left[x\cdot(\mu-\ep)- q/2\right]
  \end{align*}
  where $q = \ln[(1-\pi)/\pi]$.
This concludes the proof of \cref{orc}. \qed

\subsection{Extensions of \cref{orc}}

This section briefly describes a few extensions of \cref{orc},
with details given in \cref{app:randomized,app:weighted:risk,app:l2:gencov,app:lowdim}.

\textbf{Connections to randomized classifiers.}
The reduced effect size can also be interpreted as adding noise to the data, which relates to previously proposed algorithms
\citep{xie2018mitigating,athalye2017synthesizing}.

\textbf{Extension to weighted combinations.}
Given the tradeoff between standard risk and robust risk,
one might naturally consider minimizing a weighted combination of the two instead.
The techniques used to prove \cref{orc} can be extended to this setting, leading to new optimally robust classifiers.

\textbf{Data with a general covariance.}
We can extend \cref{orc} to some settings
where the within-class data covariance $I_p$
is replaced with a more general covariance matrix $\Sigma \in \bbR^{p \times p}$.

\textbf{Data on a low-dimensional subspace.}
For data that
lie in a low-dimensional subspace given by some coordinates being equal to zero,
we show the nontrivial result that low-dimensional classifiers
are optimal.
This is a geometric fact that holds for any norm and does not require Gaussianity.
In the Gaussian $\ell_2$-robust case, it implies that low-dimensional linear classifiers are optimal.

\subsection{Approximately optimal robust classifiers via robust isoperimetry}
So far, we have explored the problem of characterizing optimal robust classifiers.
We conclude this section by briefly characterizing
all \textit{approximately} optimal robust classifiers.
In particular, we consider robust classifiers that are approximately optimal, i.e., their robust risk is small but potentially suboptimal,
and show that they are necessarily close to half-spaces.
This is a highly nontrivial question. However, it turns out that we can make some progress by leveraging recent breakthroughs from \textit{robust isoperimetry} \citep[e.g.,][]{cianchi2011isoperimetric,mossel2015robust}. In general, these results show that if a set is approximately isoperimetric (in the sense that its boundary measure is close to minimal for its volume), then it has to be close to a half-space. In what follows, we leverage these powerful tools, which, to the best of our knowledge, have not yet been used in machine learning.

Given the difficulty of the problem, we restrict the setting to one-dimensional data. Let  $\gamma(S)$ be the Gaussian measure of a measurable set $S\subset \R$. Let also $\gamma^*(S)$ be the \textit{Gaussian deficit} of $S$, the measure of the error of approximation with a half-line:
$\gamma^*(S) \coloneqq \inf_H \gamma(S \Delta H)$,
where $S \Delta H = (S\setminus H) \cup (H\setminus S)$ is the symmetric set difference and
the infimum is taken over half-lines. Clearly $\gamma^*(S)\ge 0$, with equality when $S$ is a half-line almost surely.  The following result concerns a broad family of classifiers whose decision regions are unions of intervals. We show that if $\hat y$ has small robust risk, then its decision regions must be close to a half-line; that is, \emph{all robust classifiers are close to linear}.

\medskip
\begin{theorem}[Approximately optimal robust classifiers]
\label{thm:approx:opt:l2:twoclass}
For the two-class Gaussian $\ell_2$-robust problem, consider classifiers whose classification regions $S_+$ and $S_-$ are unions of intervals with endpoints in $[-M,M]$ where $\ep$ is less than the half-width of all intervals.

Define $\tau=\tau(\ep,M,\mu)=\ep \exp{[-((M+\mu)\ep+\ep^2/2)]}$. Then, for some universal constant $c>0$,
\oneortwo{%
\begin{align*}
\robrisk(\hat y, \ep) \ge \bayesrisk
+\tau\cdot c \cdot \left[\pi \cdot \gamma^*(S_- - \mu)^2 +(1-\pi)\cdot  \gamma^*(S_+ +\mu)^2\right].
\end{align*}
}{%
\begin{align*}
\robrisk(\hat y, \ep) &\ge \bayesrisk
\\
&
+\tau c \left[\pi \gamma^*(S_- - \mu)^2 +(1-\pi) \gamma^*(S_+ +\mu)^2\right].
\end{align*}
}%
\end{theorem}
If the robust risk $\robrisk(\hat y, \ep)$ is close to the Bayes risk $\bayesrisk$, then the Gaussian deficits $\gamma^*(S_{\pm}\pm\mu)$ are small, so the decision regions $S_{\pm}$ are near half-lines.

A priori, one might suppose that sophisticated classifiers, e.g., deep neural nets, with complicated classification regions, can benefit robustness. \Cref{thm:approx:opt:l2:twoclass} shows that these classifiers must also essentially be linear to be even approximately optimally robust. Thus, in this particular case, the complex expressivity of deep neural nets does not bring any clear benefits.

\begin{proof}[Proof of \cref{thm:approx:opt:l2:twoclass}]
  Denote by $\phi$ the standard normal density in one dimension, and recall that $\Phi$ is the standard normal cumulative distribution function. Let $\gamma^+$ be the boundary measure of measurable sets, defined precisely in \cite{cianchi2011isoperimetric,mossel2015robust}. While this definition in general poses some technical challenges, we will only use it for unions of intervals $J = \cup_{k\in K} [a_k,b_k]$, where $K$ is a countable index set and $a_k \le b_k < a_{k+1}$ are the endpoints sorted in increasing order.
  For such sets $\gamma^+(J)=\sum_k [\phi(a_k)+\phi(b_k)]$ is simply the sum of the values of the Gaussian density at the endpoints, which can be finite or infinite.
 
 The \textit{Gaussian isoperimetric profile} is commonly defined as $I = \phi \circ \Phi^{-1}$, and in this language the Gaussian isoperimetric inequality states that $I(\gamma(A))\le \gamma^+(A)$, with equality if $A$ is a half-line.
 
  Suppose $J$ is a union of intervals in $\R$ with all interval endpoints contained in $[-M,M]$. Then for $\ep$ small enough that the $\ep$-expansion of $J$ does not merge any intervals, i.e., $2\ep < (a_{k+1}-b_k)$ for all $k$, we have
 \beq\label{eb}
 \gamma(J+B_\ep) \ge \gamma(J)+ \ep \exp{[-(M\ep+\ep^2/2)]} \cdot \gamma^+(J).
 \eeq
 This follows by first considering one interval $J=[a,b]$, and then summing over all intervals, noting that the non-intersection condition on $\ep$ guarantees that all terms are additive. To check the condition for an interval $J=[a,b]$, we write:
 \begin{align*}
 \gamma([a,b]+B_\ep) &\ge \gamma([a,b])+ \ep \exp{[-(M\ep+\ep^2/2)]} \cdot \gamma^+([a,b])\\
 \iff  \gamma([a-\ep,b+\ep]) &\ge \gamma([a,b])+ \ep \exp{[-(M\ep+\ep^2/2)]} \cdot \gamma^+([a,b])\\
 \iff \gamma([a-\ep, a]) + \gamma([b,b+\ep]) &\ge \ep \exp{[-(M\ep+\ep^2/2)]} [\phi(a)+\phi(b)].
 \end{align*}
 Thus, it is enough to verify that $\gamma([a-\ep, a])\ge \ep \exp{[-(M\ep+\ep^2/2)]}\phi(a)$. This follows from
 \begin{align*}
 \gamma([a-\ep, a]) &= \int_{a-\ep}^a \phi(x) dx\\
  &= \phi(a)\int_{a-\ep}^a \phi(x)/\phi(a) dx= \phi(a)\int_{a-\ep}^a \exp{[(a^2-x^2)/2]} dx\\
 &\ge \phi(a)\cdot \ep\cdot \min_{x\in [a-\ep,a]} \exp{[(a^2-x^2)/2]} \\
 &= \phi(a)\cdot \ep\cdot \min_{u\in [-\ep,0]} \exp{[(a^2-(a-u)^2)/2]}.
 \end{align*}
 Now $a^2-(a-u)^2=2au-u^2$. Given that this is a concave function of $u$, the minimum occurs at one of the two endpoints of the interval $[-\ep,0]$. Hence, we have
 \begin{align*}
 \gamma([a-\ep, a])
 &\ge \phi(a)\cdot \ep\cdot \min\{\exp{(a\ep-\ep^2/2)},1\}\\
 &\ge \phi(a)\cdot \ep\cdot \exp{[-(M\ep+\ep^2/2)]}.
 \end{align*}
 This proves the required bound \eqref{eb} when $J$ is an interval. Thus, by additivity, it also holds for unions of intervals when $\ep$ is small enough that the $\ep$-expansion of any two intervals does not merge them. This proves \eqref{eb} for general sets $J$.
 
 Suppose now that we have a classifier whose classification regions are unions of intervals. Suppose that the conditions for \eqref{eb} hold for both $S_1$ and $S_{-1}$. Specifically, suppose that all interval endpoints are contained in $[-M,M]$ and $\ep < (a_{k+1}-b_k)$ for all $k$.
 Recall that the robust risk can be written as
 \begin{align*}
 R(\hat y, \ep) =\pi \cdot \gamma(S_{-1}-\mu) +(1-\pi) \cdot \gamma(S_1+\mu).
 \end{align*}
 Applying \eqref{eb} to both classes, and denoting $M'=M+|\mu|$, we find
 \begin{align*}
 \gamma(S_1+\mu+B_\ep) &\ge \gamma(S_1+\mu)+ \ep \exp{[-(M'\ep+\ep^2/2)]} \cdot \gamma^+(S_1+\mu)\\
 \gamma(S_{-1}-\mu+B_\ep) &\ge \gamma(S_{-1}-\mu)+ \ep \exp{[-(M'\ep+\ep^2/2)]} \cdot \gamma^+(S_{-1}-\mu).
 \end{align*}
 Let $\tau=\tau(\ep,M,\mu)=\ep \exp{[-(M'\ep+\ep^2/2)]}$. Then, by taking a weighted average of the previous two inequalities, we find the bound on the robust risk
 \begin{align*}
 R(\hat y, \ep) \ge R(\hat y, 0) + \tau\cdot [\pi \cdot \gamma^+(S_{-1}-\mu) +(1-\pi) \cdot \gamma^+(S_1+\mu)].
 \end{align*}
We denote the difference between the boundary measure and isoperimetric profile of a set $S$ as $\delta(S) = \gamma^+(S)-I(\gamma(S))$. The Gaussian isoperimetric inequality states $\delta(S)\ge 0$. Defining $\delta_{\pm 1}=\delta(S_{\pm 1}\pm \mu)$ for the two classification regions, we conclude
 \begin{align}\label{rb}
 R(\hat y, \ep) \ge &R(\hat y, 0) + \tau[\pi I(\gamma(S_{-1}-\mu))) +(1-\pi) I(\gamma(S_1+\mu))]\nonumber\\
 &+\tau\cdot [\pi \delta_{-1} +(1-\pi) \delta_{1}].
 \end{align}
 From \cite{cianchi2011isoperimetric}, it follows for the Gaussian deficit  $\gamma^*(S)$ that
 $$\gamma^*(S) \le C \sqrt{\delta(S)}$$
 for a universal constant $C$. Using this inequality for $S_{\pm 1}\pm \mu$, we find that for a universal constant $c$
 \begin{align*}
 \pi \delta_{-1} +(1-\pi) \delta_{1} \ge c\cdot[\pi \gamma^*(S_{-1}-\mu)^2 +(1-\pi) \gamma^*(S_{1}+\mu)^2].
 \end{align*}
 Plugging in to \eqref{rb}, and discarding the second term on the right hand side, we find that
 \begin{align*}
 R(\hat y, \ep) \ge R(\hat y, 0)
 +\tau\cdot c \cdot [\pi \gamma^*(S_{-1}-\mu)^2 +(1-\pi) \gamma^*(S_{1}+\mu)^2].
 \end{align*}
 Since $R(\hat y, 0) \ge R_{\textnormal{Bay}}$, this gives the desired conclusion.
 \end{proof}
 
% !TEX root = ../adv.tex

\section{Optimal \texorpdfstring{$\ell_2$}{l2} robust classifiers for three classes}
\label{sec:opt:l2:threeclass}

Having studied two-class robust classification,
we now turn to the more general setting of three Gaussian classes $\clC = \{-1,0,1\}$:
\begin{align} \label{eq:model:threeclass}
  x|y &\sim \clN( \revone{ \lambda_y } \mu, \sigma^2 I_p)
  , &
  y &=
  \begin{cases}
    +1 & \text{with probability } \pi_+
    , \\
     0 & \text{with probability } \pi_0
    , \\
    -1 & \text{with probability } \pi_-
    ,
  \end{cases}
\end{align}
where \revone{ $\mu \in \bbR^p \setminus \{0\}$
specifies the line along which the Gaussians lie,
$\lambda_y$ specifies the distance along $\mu$ of class $y$, }
$\sigma^2 \in \bbR_{>0}$ is the within-class variance,
and $\pi_+, \pi_0, \pi_- \in [0,1]$ sum to unity and specify the class proportions.
Setting $\pi_0 = 0$ produces the two-class model \cref{eq:model:twoclass}.
As before, we will take $\sigma^2=1$ without loss of generality to simplify the exposition.
\revone{%
Further, without loss of generality,
we also
normalize $\mu$ so that $\|\mu\|_2 = 1$,
center $\lambda_y$ so that $\lambda_0 = 0$,
and order $\lambda_y$ so that $\lambda_- \coloneqq \lambda_{-1} < 0 < \lambda_{1} \eqqcolon \lambda_+$.
}

The Bayes optimal classifier can again
be found via a calculation based on the pointwise  densities 
(see \cref{opt:threeclass:bayesopt}).
Here it takes the form of a \emph{linear interval classifier} (illustrated in \cref{fig:robopt:threeclass:linregions}):
\begin{equation} \label{eq:threeclass:lindef}
  \intclass(x; w,c_+,c_-)
  \coloneqq
  \begin{cases}
    +1 & \text{if } x^\top w \geq c_+
    , \\
     0 & \text{if } c_- < x^\top w < c_+
    , \\
    -1 & \text{if } x^\top w \leq c_-
    ,
  \end{cases}
\end{equation}
with Bayes optimal thresholds
\revone{
\begin{alignat}{2}
  \label{eq:threeclass:bayesopt}
  c_+ &\coloneqq \max && \big\{
    \revone{\lambda_+}/2+\ln(\pi_0/\pi_+)/\revone{|\lambda_+|},
    \revone{(\lambda_+ + \lambda_-)/2 - \ln(\pi_+/\pi_-)/|\lambda_+-\lambda_-|}
  \big\}
  , \\
  \nonumber
  c_- &\coloneqq \min && \big\{
    \revone{\lambda_-}/2-\ln(\pi_0/\pi_-)/\revone{|\lambda_-|},
    \revone{(\lambda_+ + \lambda_-)/2 - \ln(\pi_+/\pi_-)/|\lambda_+-\lambda_-|}
  \big\}
  ,
\end{alignat}
}%
and weights $w = \revone{\mu}$.
As in the two class setting,
the positive and negative classes are half-spaces,
but now the zero class in between is a slab.
As we will see,
this feature produces new phenomena and new challenges in robust classification.
Unlike the standard risk,
for which pointwise calculation of the Bayes classifier trivially generalizes to multi-class settings,
going from two-class robust classification to even three classes
requires some new methods.

\subsection{Optimal linear interval classifiers with respect to the robust risk}
We now present the main result of this section:
optimal linear interval classifiers
with respect to the robust risk.
As in the two-class setting,
this minimization does not simply reduce
to pointwise comparisons of posterior probabilities.
We focus here on the regime $\ep < \revone{\min\{|\lambda_+|,|\lambda_-|\}}/2$
where nontrivial classification to all three classes occurs;
when $\ep \geq \revone{\min\{|\lambda_+|,|\lambda_-|\}}/2$ the problem essentially reduces to two-class problems
(see \cref{opt:threeclass:large:eps}).

\oneortwo{%
\begin{figure}
  \begin{subfigure}{0.31\textwidth} \centering
    \begin{tikzpicture}[scale=0.95]
      \draw[->,>=stealth,thick] (0,0) -- node[text width=5mm,above,near end]{$w$} (0.7,0.7);
      \fill[opacity=0.2,color=Red]
        (0-1.0,+1.7) -- (+1.7,0-1.0) -- (+1.7,+1.7) -- cycle;
      \draw[thick,dashed] (0-1.0,+1.7) -- (+1.7,0-1.0);
      \fill[opacity=0.2,color=OliveGreen]
        (0-1.0,+1.7) -- (+1.7,0-1.0) -- (+1.7,-1.7) --
        (0+0.7,-1.7) -- (-1.7,0+0.7) -- (-1.7,+1.7) -- cycle;
      \draw[thick,dashed] (0+0.7,-1.7) -- (-1.7,0+0.7);
      \fill[opacity=0.2,color=Blue]
        (0+0.7,-1.7) -- (-1.7,0+0.7) -- (-1.7,-1.7) -- cycle;

      \node[rotate=-45] at (+0.95,+0.95) {\small $\primetranspose x'w > c_+$};
      \node[rotate=-45] at (-0.90,-0.90) {\small $\primetranspose x'w < c_-$};
      \node[rotate=-45] at (-0.25,-0.25) {\small $\primetranspose c_- < x'w < c_+$};
      \node[color=Red]        at (+1.4,+1.4) {\large $+$};
      \node[color=OliveGreen] at (+1.4,-1.4) {\large $0$};
      \node[color=Blue]       at (-1.4,-1.4) {\large $-$};
    \end{tikzpicture}
    \caption{Linear interval classifier regions.}
    \label{fig:robopt:threeclass:linregions}
  \end{subfigure}
  \begin{subfigure}{0.32\textwidth} \centering
    \includegraphics[scale=0.53]{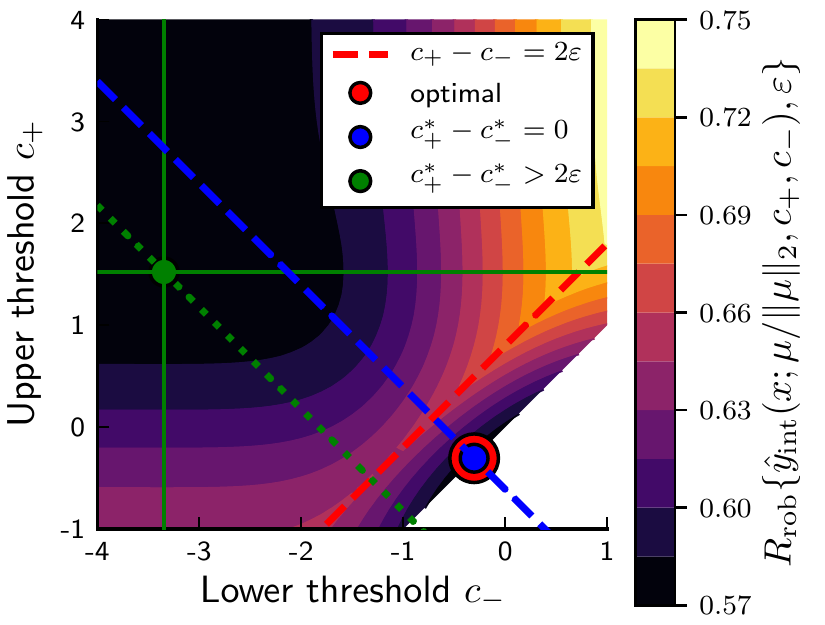}
    \caption{$\pi_0 = 42.00\%$}
    \label{fig:robopt:threeclass:pi0:a}
  \end{subfigure} \quad
  \begin{subfigure}{0.32\textwidth} \centering
    \includegraphics[scale=0.53]{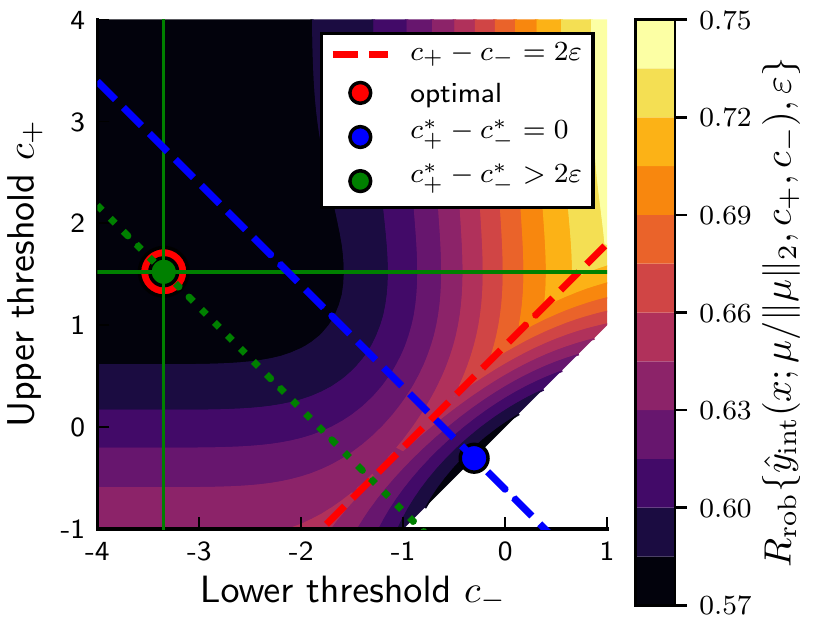}
    \caption{$\pi_0 = 42.01\%$}
    \label{fig:robopt:threeclass:pi0:b}
  \end{subfigure}
  \caption{Optimal linear \revone{interval} $\ell_2$-robust classifiers for three classes;
    (\subref{fig:robopt:threeclass:pi0:a})
    and (\subref{fig:robopt:threeclass:pi0:b})
    show the thresholds
    \cref{eq:robopt:threeclass:1:thresh,eq:robopt:threeclass:2:thresh},
    circling the optimal,
    where $\mu = 1$, \revone{$\lambda_\pm = \pm 1$,} $\gamma = 1.2$, $\ep = 0.4$
    and $\pi_\pm = (1-\pi_0)\{\gamma^{\pm 1}/(\gamma+\gamma^{-1})\}$.}
  \label{fig:robopt:threeclass}
\end{figure}
}{%
\begin{figure} \centering
  \begin{subfigure}{\linewidth} \centering
    \begin{tikzpicture}[scale=0.8]
      \draw[->,>=stealth,thick] (0,0) -- node[text width=5mm,above,near end]{$w$} (0.7,0.7);
      \fill[opacity=0.2,color=Red]
        (0-1.0,+1.7) -- (+1.7,0-1.0) -- (+1.7,+1.7) -- cycle;
      \draw[thick,dashed] (0-1.0,+1.7) -- (+1.7,0-1.0);
      \fill[opacity=0.2,color=OliveGreen]
        (0-1.0,+1.7) -- (+1.7,0-1.0) -- (+1.7,-1.7) --
        (0+0.7,-1.7) -- (-1.7,0+0.7) -- (-1.7,+1.7) -- cycle;
      \draw[thick,dashed] (0+0.7,-1.7) -- (-1.7,0+0.7);
      \fill[opacity=0.2,color=Blue]
        (0+0.7,-1.7) -- (-1.7,0+0.7) -- (-1.7,-1.7) -- cycle;

      \node[rotate=-45] at (+0.95,+0.95) {\small $\primetranspose x'w > c_+$};
      \node[rotate=-45] at (-0.90,-0.90) {\small $\primetranspose x'w < c_-$};
      \node[rotate=-45] at (-0.25,-0.25) {\small $\primetranspose c_- < x'w < c_+$};
      \node[color=Red]        at (+1.4,+1.4) {\large $+$};
      \node[color=OliveGreen] at (+1.4,-1.4) {\large $0$};
      \node[color=Blue]       at (-1.4,-1.4) {\large $-$};
    \end{tikzpicture}
    \caption{Linear interval classifier regions.}
    \label{fig:robopt:threeclass:linregions}
  \end{subfigure}\\[1em]

  \begin{subfigure}{0.49\linewidth} \centering
    \includegraphics[width=3.95cm,trim=6pt 5pt 6.5pt 0]{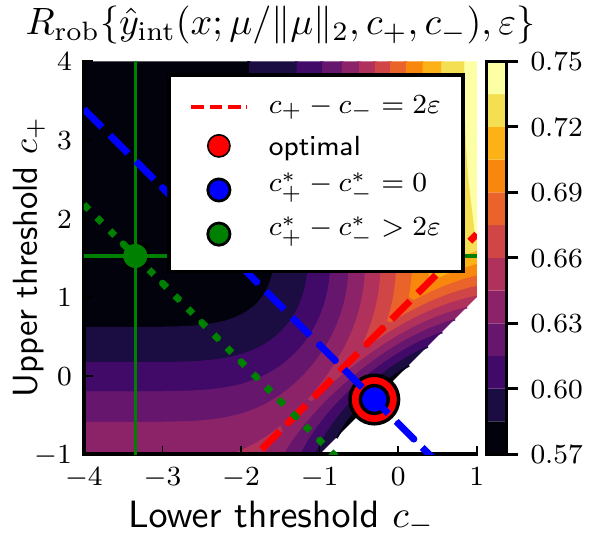}
    \caption{$\pi_0 = 42.00\%$}
    \label{fig:robopt:threeclass:pi0:a}
  \end{subfigure} \hfill
  \begin{subfigure}{0.49\linewidth} \centering
    \includegraphics[width=3.95cm,trim=6pt 5pt 6.5pt 0]{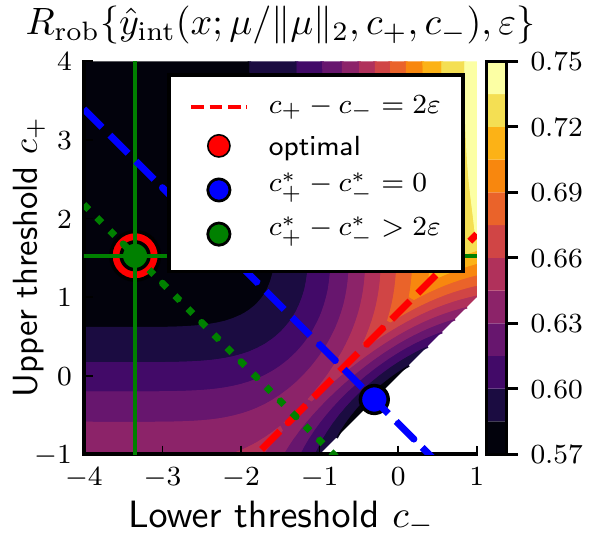}
    \caption{$\pi_0 = 42.01\%$}
    \label{fig:robopt:threeclass:pi0:b}
  \end{subfigure}
  \caption{Optimal linear \revone{interval} $\ell_2$-robust classifiers for three classes;
    (\subref{fig:robopt:threeclass:pi0:a})
    and (\subref{fig:robopt:threeclass:pi0:b})
    show the thresholds
    \cref{eq:robopt:threeclass:1:thresh,eq:robopt:threeclass:2:thresh},
    circling the optimal,
    where $\|\mu\|_2 = 1$, \revone{$\lambda_\pm = \pm 1$,} $\gamma = 1.2$, $\ep = 0.4$
    and $\pi_\pm = (1-\pi_0)\{\gamma^{\pm 1}/(\gamma+\gamma^{-1})\}$.}
  \label{fig:robopt:threeclass}
\end{figure}
}%

\medskip
\begin{theorem}[Optimal linear \revone{interval} $\ell_2$-robust three-class classifiers]
\label{thm:opt:threeclass:threshopt}
Suppose the data $(x,y)$ follow the three-class Gaussian model \cref{eq:model:threeclass}
and $\ep < \revone{\min\{|\lambda_+|,|\lambda_-|\}}/2$.
Among linear interval classifiers, an optimal $\ell_2$ robust \revone{classifier} is:
\begin{equation} \label{eq:robopt:threeclass:lin}
  \intclass^*(x)
  \coloneqq
  \revone{ \intclass(x; \mu, c_+^*, c_-^*) }
  ,
\end{equation}
where the thresholds $c_+^* \geq c_-^*$ are:
\begin{enumerate}
\item[Case 1 (\revone{rare} zero class).]
If $\pi_0 \leq \alpha^* \sqrt{\pi_- \pi_+}$ for $\alpha^*$ specified below in \cref{eq:opt:threeclass:cutoff}, then
the optimal classifier ignores the zero class.
The thresholds are equal,
i.e., $c_+^* = c_-^*$,
with value
\begin{equation} \label{eq:robopt:threeclass:1:thresh}
  c_+^* = c_-^*
  = \revone{ \frac{\lambda_+ + \lambda_-}{2} + } \frac{\ln(\pi_-/\pi_+)}{ \revone{ |\lambda_+ - \lambda_-| } -2\ep}
  ,
\end{equation}
and the corresponding robust risk is
\begin{equation*}
  \robrisk(\intclass^*,\ep)
  = \pi_0 + (\pi_+ + \pi_-) \robrisk^*\bigg( \revone{ \frac{|\lambda_+ - \lambda_-|}{2} } ,\frac{\pi_+}{\pi_+ + \pi_-};\ep\bigg)
  ,
\end{equation*}
where  $\robrisk^*$ is the two-class optimal robust risk in \cref{eq:robopt:twoclass:risk}.

\item[Case 2 (\revone{frequent} zero class).]
Otherwise, the thresholds are at least $2\ep$ apart,
i.e., $c_+^* - c_-^* > 2\ep$,
with
\begin{equation} \label{eq:robopt:threeclass:2:thresh}
  c_\pm^*
  = \frac{ \revone{ \lambda_\pm } }{2} \pm \frac{\ln(\pi_0/\pi_\pm)}{ \revone{ |\lambda_\pm| } -2\ep}
  ,
\end{equation}
and the corresponding robust risk is
\begin{equation*}
  \robrisk(\intclass^*,\ep)
  = (\pi_+ + \pi_0) \robrisk^*\bigg(\frac{ \revone{ |\lambda_+| } }{2},\frac{\pi_+}{\pi_+ + \pi_0};\ep\bigg)
  + (\pi_- + \pi_0) \robrisk^*\bigg(\frac{ \revone{ |\lambda_-| } }{2},\frac{\pi_-}{\pi_- + \pi_0};\ep\bigg)
  .
\end{equation*}
\end{enumerate}
The cutoff $\alpha^*$
is the unique solution
with \revone{ $\alpha \geq \bralpha$ }
to the equation:
\begin{align} \label{eq:opt:threeclass:cutoff}
  &
  (\gamma + \gamma^{-1})
  \robrisk^*\bigg( \revone{ \frac{|\lambda_+ - \lambda_-|}{2} } ,\frac{\gamma}{\gamma + \gamma^{-1}};\ep\bigg)
  \\&\qquad\qquad
  =
  (\gamma + \alpha)
  \robrisk^*\bigg(
    \frac{ \revone{ |\lambda_+| } }{2},\frac{\gamma}{\gamma + \alpha};\ep
  \bigg)
  +
  (\gamma^{-1} + \alpha)
  \robrisk^*\bigg(
    \frac{ \revone{ |\lambda_-| } }{2},\frac{\gamma^{-1}}{\gamma^{-1} + \alpha};\ep
  \bigg)
  - \alpha
  , \nonumber
\end{align}
where $\gamma \coloneqq \sqrt{\pi_+/\pi_-}$\revone{,
and
\begin{equation*}
  \bralpha
  \coloneqq
  \exp\bigg\{
    - \frac{(|\lambda_+|-2\ep)(|\lambda_-|-2\ep)}{2}
    - \frac{\lambda_+ + \lambda_-}{|\lambda_+ - \lambda_-| - 4\ep} \ln\gamma
  \bigg\}
  .
\end{equation*}
}
\end{theorem}

\Cref{thm:opt:threeclass:threshopt}
reveals several properties of optimally robust linear \revone{interval} classification
in the three-class setting.

\textbf{New three-class phenomenon: discontinuity of the optimal thresholds.}
The optimal thresholds in \cref{thm:opt:threeclass:threshopt}
fall into two cases.
In the first case, the thresholds \cref{eq:robopt:threeclass:1:thresh}
coincide with the one from two-class robust classification
between the positive and negative classes,
\emph{ignoring the zero class}.
In the second case, the thresholds \cref{eq:robopt:threeclass:2:thresh}
coincide with two-class robust classification:
i) between the zero and positive classes,
and
ii) between the zero and negative classes.

However, the robust setting with $\ep > 0$
introduces a new and perhaps surprising property:
optimal thresholds can now be \emph{discontinuous}
in the problem parameters,
as they jump between the cases.
For example,
\cref{fig:robopt:threeclass:pi0:a,fig:robopt:threeclass:pi0:b}
show nearly identical setups,
with only a \emph{slight} change in class imbalance,
but the optimal thresholds jump \emph{discontinuously}
from \cref{eq:robopt:threeclass:1:thresh} to \cref{eq:robopt:threeclass:2:thresh}.
To understand why this occurs, 
the optimal thresholds
are never in $0 < c_+ - c_- < 2\ep$
since any such choice can be improved
by moving $c_+$ and $c_-$ closer together.
This discontinuity does \emph{not} arise for the standard risk ($\ep = 0$),
since \cref{eq:robopt:threeclass:1:thresh,eq:robopt:threeclass:2:thresh}
coincide at the cutoff in that case.
It also does \emph{not} arise in the two-class setup discussed previously.

The three-class Gaussian setting \cref{eq:model:threeclass} considered here
is an exemplar for general multi-class settings,
and we expect similar discontinuous transitions to appear more generally.
Essentially, they arise from a general property of the $\ep$-robust risk
that our analysis draws into focus.
Namely, classification regions must either be:
i) empty or
ii) have radius at least $\ep$ (otherwise it is better to remove them).
Jumping between these cases produces the observed discontinuity.

\textbf{Tradeoffs and the impact of class imbalance.}
As in the two-class setting,
we see a tradeoff in general between the standard risk
and the robust risk.
No classifier optimizes both
\emph{even when restricted to linear \revone{interval} classifiers},
which are optimal for the standard risk.
A new three-class manifestation of this phenomenon
arises in the jump
from \cref{eq:robopt:threeclass:1:thresh} to \cref{eq:robopt:threeclass:2:thresh},
i.e., from ignoring to including the zero class.
The class ratios in
\cref{eq:robopt:threeclass:1:thresh,eq:robopt:threeclass:2:thresh}
are also effectively inflated by
\revone{$1/(1-2\ep/|\lambda_+ - \lambda_-|)$} and \revone{$1/(1-2\ep/|\lambda_\pm|)$},
respectively,
amplifying the effects of class imbalance.

\subsection{Proof of \cref{thm:opt:threeclass:threshopt}}
\newcommand{\tlrobrisk}{\tlR_{\mathrm{rob}}}

We now prove \cref{thm:opt:threeclass:threshopt}.
The three-class setting introduces a new challenge in the proof:
the optimization landscape will turn out to have two qualitatively different regions.
Essentially, the robust risk produces a dichotomy between classifiers
where the central zero-class slab in \cref{fig:robopt:threeclass:linregions}
is either ``thin'' or ``thick''.
Optimizing over each region separately yields two candidate optimizers.
We will need to carefully analyze their properties
to characterize when each is globally optimal.

Precisely put, our goal here is to find $w \in \bbR^p \setminus \{0\}$ and $c_+ \geq c_-$
that minimize the robust risk
\begin{equation*}
  \robrisk\{\intclass(\cdot; w, c_+, c_-),\ep,\|\cdot\|_2\}
  = \pi_+ \clM_+(w, c_+, c_-, \ep)
  + \pi_0 \clM_0(w, c_+, c_-, \ep)
  + \pi_- \clM_-(w, c_+, c_-, \ep)
  ,
\end{equation*}
where
\begin{align*}
  \clM_+(w, c_+, c_-, \ep) &\coloneqq
  \Pr_{x|y=+1} \big\{
    \exists_{\delta : \|\delta\|_2 \leq \ep}
    \;\;
    \intclass(x+\delta; w,c_+,c_-) \neq +1
  \big\}
  , \\
  \clM_-(w, c_+, c_-, \ep) &\coloneqq
  \Pr_{x|y=-1} \big\{
    \exists_{\delta : \|\delta\|_2 \leq \ep}
    \;\;
    \intclass(x+\delta; w,c_+,c_-) \neq -1
  \big\}
  , \\
  \clM_0(w, c_+, c_-, \ep) &\coloneqq
  \Pr_{x|y=0} \big\{
    \exists_{\delta : \|\delta\|_2 \leq \ep}
    \;\;
    \intclass(x+\delta; w,c_+,c_-) \neq 0
  \big\}
  ,
\end{align*}
are the associated robust misclassification probabilities
as functions of $w$, $c_+$, $c_-$ and $\ep$.

\subsubsection{Optimizing over \texorpdfstring{$\bmw$}{w}}
\label{proof:thm:opt:threeclass:threshopt:weights}
First we optimize over $w \in \bbR^p \setminus \{0\}$.
Since any positive scaling of $w$ can be absorbed into $c_+$ and $c_-$,
it suffices to consider $\|w\|_2 = 1$.
Then we have
\begin{equation*}
  \clM_0(w, c_+, c_-,\ep)
  = \Pr_{x|y=0}(x^\top w \leq c_- + \ep \text{ or } x^\top w \geq c_+ - \ep)
  = \Pr_{\tlx \sim \clN(0,1)}(\tlx \leq c_- + \ep \text{ or } \tlx \geq c_+ - \ep)
  ,
\end{equation*}
which is a constant with respect to $w$.
Meanwhile,
\begin{align*}
  \clM_+(w, c_+, c_-,\ep)
  &= \Pr_{x|y=+1}(x^\top w < c_+ + \ep)
  = \Pr_{\tlx \sim \clN(\revone{\lambda_+} \mu^\top w,1)}(\tlx < c_+ + \ep)
  , \\
  \clM_-(w, c_+, c_-,\ep)
  &= \Pr_{x|y=-1}(x^\top w > c_- - \ep)
  = \Pr_{\tlx \sim \clN(\revone{\lambda_-} \mu^\top w,1)}(\tlx > c_- - \ep)
  ,
\end{align*}
which are both minimized by taking $w = \revone{\mu}$.
Hence $w^* \coloneqq \revone{\mu}$ optimizes the robust risk.

\subsubsection{Finding candidate optimizers with respect to \texorpdfstring{$c_+$}{c+} and \texorpdfstring{$c_-$}{c-}}
\label{proof:thm:opt:threeclass:threshopt:thresh}
We now proceed to derive optimal thresholds $c_+ \geq c_-$
given optimal weights $w^* = \revone{\mu}$.
Substituting $w^*$ into the robust risk yields
the following function of $c_-$ and $c_+$:
\begin{align*}
  \robrisk(c_-,c_+)
  &\coloneqq
  \robrisk\{\intclass(\cdot; w^*, c_+, c_-),\ep,\|\cdot\|_2\}
  \\&
  = \pi_+ \clM_+(w^*, c_+, c_-, \ep)
  + \pi_0 \clM_0(w^*, c_+, c_-, \ep)
  + \pi_- \clM_-(w^*, c_+, c_-, \ep)
  \\&
  = \pi_+ \Pr_{\tlx \sim \clN(\revone{\lambda_+} \mu^\top w^*,1)}(\tlx < c_+ + \ep)
  + \pi_0 \Pr_{\tlx \sim \clN(0,1)}(\tlx \leq c_- + \ep \text{ or } \tlx \geq c_+ - \ep)
  \\&\qquad
  + \pi_- \Pr_{\tlx \sim \clN(\revone{\lambda_-} \mu^\top w^*,1)}(\tlx > c_- - \ep)
  \\&
  =
    \pi_- \Pr_-(\tlx > c_- - \ep)
  + \pi_+ \Pr_+(\tlx < c_+ + \ep)
  + \pi_0 \Pr_0(\tlx \leq c_- + \ep \text{ or } \tlx \geq c_+ - \ep)
  ,
\end{align*}
where we drop the arguments $\ep$ and $\|\cdot\|_2$ from $\robrisk$ for simplicity,
and use the following shorthands
for the class-conditional probabilities
\begin{align*}
  \Pr_-
  &\coloneqq
  \Pr_{\tlx \sim \clN(\revone{\lambda_-} \mu^\top w^*,1)}
  = \Pr_{\tlx \sim \clN(\revone{\lambda_-},1)}
  , &
  \Pr_+
  &\coloneqq
  \Pr_{\tlx \sim \clN(\revone{\lambda_+} \mu^\top w^*,1)}
  = \Pr_{\tlx \sim \clN(\revone{\lambda_+},1)}
  , &
  \Pr_0
  &\coloneqq
  \Pr_{\tlx \sim \clN(0,1)}
  .
\end{align*}

Next we will explain the unique feature of the robust risk:
the constraint set $\Omega \coloneqq \{(c_-, c_+ : c_- \leq c_+)\}$
contains two qualitatively different regions.
Minimizing over each region results in two candidate optimizers.

First, consider
$\Omega_0 \coloneqq \{(c_-, c_+) : c_- \leq c_+ \leq c_- + 2\ep\}$.
In this region,
$\Pr_0(\tlx \leq c_- + \ep \text{ or } \tlx \geq c_+ - \ep) = 1$
so
the robust risk simplifies to
\begin{equation*}
  \robrisk(c_-,c_+)
  =
    \pi_- \Pr_-(\tlx > c_- - \ep)
  + \pi_+ \Pr_+(\tlx < c_+ + \ep)
  + \pi_0
  .
\end{equation*}
Since this function is decreasing in $c_-$ and increasing in $c_+$,
it is minimized by $c_- = c_+$.
This yields a two-class problem between the negative and positive classes
with minimizer
\begin{equation} \label{eq:threeclassline:threshopt:equal}
  c_- = c_+
  = \tlc \coloneqq \revone{\frac{\lambda_+ + \lambda_-}{2}} + \frac{\ln(\pi_-/\pi_+)}{\revone{|\lambda_+ - \lambda_-|}-2\ep}
  .
\end{equation}
This is the minimizer over $\Omega_0$.

Second, consider
$\Omega_1 \coloneqq \{(c_-, c_+) : c_+ \geq c_- + 2\ep\}$.
In this region,
$\Pr_0(\tlx \leq c_- + \ep \text{ and } \tlx \geq c_+ - \ep) = 0$
so
\begin{equation*}
  \Pr_0(\tlx \leq c_- + \ep \text{ or } \tlx \geq c_+ - \ep)
  = \Pr_0(\tlx \leq c_- + \ep) + \Pr_0(\tlx \geq c_+ - \ep)
  ,
\end{equation*}
and the robust risk simplifies to
$\robrisk(c_-,c_+) = \robrisk^-(c_-) + \robrisk^+(c_+)$,
where
\begin{align*}
  \robrisk^-(c_-) &\coloneqq
    \pi_- \Pr_-(\tlx > c_- - \ep)
  + \pi_0 \Pr_0(\tlx \leq c_- + \ep)
  , \\
  \robrisk^+(c_+) &\coloneqq
    \pi_0 \Pr_0(\tlx \geq c_+ - \ep)
  + \pi_+ \Pr_+(\tlx < c_+ + \ep)
  .
\end{align*}
Now, $\robrisk^-(c_-)$ is proportional to the $\ep$-robust risk
for a two-class problem between the negative and zero classes.
Hence, it is a decreasing function for $c_- < \tlc_-$
and an increasing function for $c_- > \tlc_-$,
where the critical point $\tlc_-$ is
\begin{equation} \label{eq:threeclassline:threshopt:minus}
  \tlc_- \coloneqq \frac{\revone{\lambda_-}}{2} - \frac{\ln(\pi_0/\pi_-)}{\revone{|\lambda_-|}-2\ep}
  .
\end{equation}
Likewise $\robrisk^+(c_+)$ is a decreasing function for $c_+ < \tlc_+$
and an increasing function for $c_+ > \tlc_+$,
where the critical point is
\begin{equation} \label{eq:threeclassline:threshopt:plus}
  \tlc_+ \coloneqq \frac{\revone{\lambda_+}}{2} + \frac{\ln(\pi_0/\pi_+)}{\revone{|\lambda_+|}-2\ep}
  .
\end{equation}
If $(\tlc_-,\tlc_+) \in \Omega_1$,
then it follows that $(\tlc_-,\tlc_+)$ is globally optimal within $\Omega_1$.
On the other hand, if $(\tlc_-,\tlc_+) \notin \Omega_1$
then the optimal value in $\Omega_1$
must occur on the boundary $c_+ = c_- + 2\ep$.
By the monotonicity analysis above,
any point off that boundary can necessarily be improved
either by increasing $c_-$ or by decreasing $c_+$
since either $c_- \leq \tlc_-$ or $c_+ \geq \tlc_+$
for any $(c_-,c_+) \in \Omega_1$
when $(\tlc_-,\tlc_+) \notin \Omega_1$.

\subsubsection{Characterizing when each candidate optimizer is globally optimal}
Now we compare the minimizers from the regions $\Omega_0$ and $\Omega_1$
to find globally optimal thresholds.
For this purpose, it turns out that
$\alpha = \pi_0/\sqrt{\pi_- \pi_+}$
and $\gamma = \sqrt{\pi_+/\pi_-}$
provide a more convenient parameterization
than $\pi_-$, $\pi_0$ and $\pi_+$.
Recall that $\pi_0 + \pi_- + \pi_+ = 1$ necessarily constrains those parameters.
Rewriting
\cref{eq:threeclassline:threshopt:plus,%
  eq:threeclassline:threshopt:minus,%
  eq:threeclassline:threshopt:equal}
in terms of $\alpha$ and $\gamma$ yields
\begin{align*}
  \tlc &= \revone{\frac{\lambda_+ + \lambda_-}{2}} - \frac{\ln\gamma}{\revone{|\lambda_+ - \lambda_-|/2}-\ep}
  , &
  \tlc_{\pm} &= \frac{\revone{\lambda_{\pm}}}{2} \pm \frac{\ln\alpha}{\revone{|\lambda_{\pm}|}-2\ep} - \frac{\ln\gamma}{\revone{|\lambda_{\pm}|}-2\ep}
  .
\end{align*}
Furthermore,
\begin{align*}
  &
  \tlc_+ > \tlc_- + 2\ep
  \\&
  \iff
  \revone{
  0 < \tlc_+ - \tlc_- - 2\ep
  = \frac{|\lambda_+ - \lambda_-| - 4\ep}{2}
  + \frac{|\lambda_+ - \lambda_-| - 4\ep}{(|\lambda_+|-2\ep)(|\lambda_-|-2\ep)} \ln\alpha
  + \frac{ \lambda_+ + \lambda_-        }{(|\lambda_+|-2\ep)(|\lambda_-|-2\ep)} \ln\gamma
  }
  \\&
  \iff
  \revone{
  \alpha
  >
  \bralpha
  \coloneqq
  \exp\bigg\{
    - \frac{(|\lambda_+|-2\ep)(|\lambda_-|-2\ep)}{2}
    - \frac{\lambda_+ + \lambda_-}{|\lambda_+ - \lambda_-| - 4\ep} \ln\gamma
  \bigg\}
  }
  ,
\end{align*}
yielding \revone{an} equivalent condition
for $(\tlc_-,\tlc_+) \in \interior\Omega_1$, 
where $\interior A$ denotes the interior of the set $A$ with respect to the standard topology induced by the Euclidean metric.
Thus, when $\alpha \leq \revone{\bralpha}$,
the optimal value in $\Omega_1$ occurs on the boundary $c_+ = c_- + 2\ep$,
but this boundary is also contained in $\Omega_0$
so it is no worse than \cref{eq:threeclassline:threshopt:equal}.
Namely, \cref{eq:threeclassline:threshopt:equal} is optimal
when $\alpha \leq \revone{\bralpha}$.

Now suppose $\alpha > \revone{\bralpha}$.
In this case,
$(\tlc_-,\tlc_+) \in \interior\Omega_1$ is optimal in $\Omega_1$,
so we compare $\robrisk(\tlc,\tlc)$
with $\robrisk(\tlc_-,\tlc_+) = \robrisk^-(\tlc_-) + \robrisk^+(\tlc_+)$.
For this comparison,
we study the sign of
\begin{equation} \label{eq:opt:threeclass:diff}
  \Delta(\alpha)
  \coloneqq
  \tlrobrisk^-(\tlc_-,\alpha) + \tlrobrisk^+(\tlc_+,\alpha) - \tlrobrisk(\tlc,\alpha)
  ,
\end{equation}
where we define
\begin{align*}
  \tlrobrisk^-(\tlc_-,\alpha)
  &\coloneqq
    \gamma^{-1} \Pr_-(\tlx > \tlc_- - \ep)
  + \alpha      \Pr_0(\tlx \leq \tlc_- + \ep)
  , \\
  \tlrobrisk^+(\tlc_+,\alpha)
  &\coloneqq
    \alpha      \Pr_0(\tlx \geq \tlc_+ - \ep)
  + \gamma      \Pr_+(\tlx < \tlc_+ + \ep)
  , \\
  \tlrobrisk(\tlc,\alpha)
  &\coloneqq
    \gamma^{-1} \Pr_-(\tlx > \tlc - \ep)
  + \gamma      \Pr_+(\tlx < \tlc + \ep)
  + \alpha
  .
\end{align*}
Note first that \revone{$\Delta(\bralpha) \geq 0$}
since, as established above,
$\robrisk(\tlc,\tlc) \leq \robrisk^-(\tlc_-) + \robrisk^+(\tlc_+)$
in this case.

Next, when $\alpha > \revone{\bralpha}$
\begin{align*}
  \frac{\partial \Delta(\alpha)}{\partial \alpha}
  &=
    \frac{\partial \tlrobrisk^-(\tlc_-,\alpha)}{\partial \tlc_-}
      \frac{\partial \tlc_-}{\partial \alpha}
  + \frac{\partial \tlrobrisk^-(\tlc_-,\alpha)}{\partial \alpha}
  + \frac{\partial \tlrobrisk^+(\tlc_+,\alpha)}{\partial \tlc_+}
      \frac{\partial \tlc_+}{\partial \alpha}
  + \frac{\partial \tlrobrisk^+(\tlc_+,\alpha)}{\partial \alpha}
  - \frac{\partial \tlrobrisk(\tlc,\alpha)}{\partial \alpha}
  \\
  &= \Pr_0(\tlx \leq \tlc_- + \ep) + \Pr_0(\tlx \geq \tlc_+ - \ep) - 1
  < 0
  .
\end{align*}
Both $\tlc_-$ and $\tlc_+$ are differentiable functions of $\alpha$ for $\alpha > 0$,
and $\tlrobrisk^-$, $\tlrobrisk^+$ and $\tlrobrisk$
are continuously differentiable functions of $\tlc_-$, $\tlc_+$ and $\alpha$.
Hence, $\Delta$ is also a differentiable function of $\alpha$.
The equality in the second line holds because $\tlc_-$ and $\tlc_+$
are critical points of $\robrisk^-$ and $\robrisk^+$, respectively,
and the inequality holds because $\tlc_- + \ep < \tlc_+ - \ep$
for $\alpha > \revone{\bralpha}$.
Namely, $\Delta(\alpha)$ is a decreasing function for $\alpha > \revone{\bralpha}$.

Next, $\tlc_-$ is a decreasing function of $\alpha$, and $\tlc_+$ is an increasing function of $\alpha$.
Hence, $\Pr_0(\tlx \leq \tlc_- + \ep)$ and $\Pr_0(\tlx \geq \tlc_+ - \ep)$
are both decreasing functions of $\alpha$,
and so $\partial \Delta(\alpha)/\partial \alpha$
is also a decreasing function of $\alpha$.
As a result, $\Delta < 0$ eventually
and $\Delta$ has exactly one root with respect to $\alpha$ in the domain  $\alpha \geq \revone{\bralpha}$.
Specifically, there is a unique $\alpha^* \geq \revone{\bralpha}$ for which $\Delta(\alpha^*) = 0$.
If $\alpha < \alpha^*$ then $\Delta > 0$, and if $\alpha > \alpha^*$ then $\Delta < 0$.

\revone{In particular,} we must solve the equation $\Delta(\alpha) = 0$ to obtain $\alpha^*$.
We conclude by deriving the alternative form \cref{eq:opt:threeclass:cutoff}
for this equation.
First, recall that $\tlc$ is the optimal two-class threshold
between the negative and positive classes so
\begin{align*}
  \robrisk(\tlc,\tlc)
  &=
    \pi_0
  + \pi_- \Pr_-(\tlx > \tlc - \ep)
  + \pi_+ \Pr_+(\tlx < \tlc + \ep)
  \\
  &=
    \pi_0
  + (\pi_+ + \pi_-)
  \bigg\{
    \frac{\pi_-}{\pi_+ + \pi_-} \Pr_-(\tlx > \tlc - \ep)
  + \frac{\pi_+}{\pi_+ + \pi_-} \Pr_+(\tlx < \tlc + \ep)
  \bigg\}
  \\
  &=
    \pi_0
  + (\pi_+ + \pi_-) \robrisk^*\bigg(\revone{\frac{|\lambda_+ - \lambda_-|}{2}},\frac{\pi_+}{\pi_+ + \pi_-};\ep\bigg)
  .
\end{align*}
Next, when $\alpha \geq \revone{\bralpha}$,
we have $\tlc_+ \geq \tlc_- + 2\ep$ so
\begin{align*}
  &\robrisk(\tlc_-,\tlc_+)
  = \robrisk^-(\tlc_-) + \robrisk^+(\tlc_+)
  \\
  &\qquad=
    \pi_+ \Pr_+(\tlx < \tlc_+ + \ep)
  + \pi_0 \Pr_0(\tlx \geq \tlc_+ - \ep)
  + \pi_0 \Pr_0(\tlx \leq \tlc_- + \ep)
  + \pi_- \Pr_-(\tlx > \tlc_- - \ep)
  \\
  &\qquad=
  (\pi_+ + \pi_0)
  \bigg\{
    \frac{\pi_0}{\pi_+ + \pi_0} \Pr_0(\tlx \geq \tlc_+ - \ep)
  + \frac{\pi_+}{\pi_+ + \pi_0} \Pr_+(\tlx < \tlc_+ + \ep)
  \bigg\}
  \\
  &\qquad\qquad
  + (\pi_- + \pi_0)
  \bigg\{
    \frac{\pi_-}{\pi_- + \pi_0} \Pr_-(\tlx > \tlc_- - \ep)
  + \frac{\pi_0}{\pi_- + \pi_0} \Pr_0(\tlx \leq \tlc_- + \ep)
  \bigg\}
  \\
  &\qquad=
    (\pi_+ + \pi_0) \robrisk^*\bigg(\frac{\revone{|\lambda_+|}}{2},\frac{\pi_+}{\pi_+ + \pi_0};\ep\bigg)
  + (\pi_- + \pi_0) \robrisk^*\bigg(\frac{\revone{|\lambda_-|}}{2},\frac{\pi_-}{\pi_- + \pi_0};\ep\bigg)
  ,
\end{align*}
since $\tlc_+$ and $\tlc_-$ are, respectively, optimal two-class thresholds:
i) between the zero and the positive class
and
ii) between the zero and the negative class.
Substituting these into \cref{eq:opt:threeclass:diff} and simplifying yields
\begin{align*}
  \Delta
  &=
  \frac{1}{\sqrt{\pi_- \pi_+}}
  \big\{
    \robrisk^-(\tlc_-) + \robrisk^+(\tlc_+) - \robrisk(\tlc,\tlc)
  \big\}
  \\
  &=
    (\gamma      + \alpha) \robrisk^*\bigg(
      \frac{\revone{|\lambda_+|}}{2}, \frac{\gamma     }{\gamma      + \alpha}; \ep
    \bigg)
  + (\gamma^{-1} + \alpha) \robrisk^*\bigg(
      \frac{\revone{|\lambda_-|}}{2}, \frac{\gamma^{-1}}{\gamma^{-1} + \alpha}; \ep
    \bigg)
  - \alpha
   \\&\qquad
  - (\gamma + \gamma^{-1}) \robrisk^*\bigg(
    \revone{\frac{|\lambda_+ - \lambda_-|}{2}}, \frac{\gamma}{\gamma + \gamma^{-1}}; \ep
  \bigg)
  ,
\end{align*}
and re-arranging gives \cref{eq:opt:threeclass:cutoff}.
This concludes the proof of \cref{thm:opt:threeclass:threshopt}.
\qed

\subsection{\revone{On the optimality} of linear \revone{interval} classifiers}
\label{sec:opt:l2:threeclass:optimality}

\revone{%
\Cref{thm:opt:threeclass:threshopt}
gave an optimal linear interval robust classifier
for the three-class setting \cref{eq:model:threeclass},
i.e.,
we derived a classifier
that minimizes the robust risk
subject to the constraint that
it is a linear interval classifier \cref{eq:threeclass:lindef}.
Naturally, one may wonder if this classifier
is only optimal among linear interval classifiers
or if it is in fact optimal across all classifiers.
In other words,
are linear interval classifiers optimal?
It is important to note here
that answering this question
was not needed to demonstrate a tradeoff
between the standard risk and the robust risk,
as discussed above.
Nevertheless, this is an important and interesting question in its own right.
Indeed,
given the symmetry of the three-class setting we consider,
we expect
that linear interval classifiers are optimal.%
\footnote{\revone{For Gaussians in arbitrary positions,
  i.e., with means that do not lie on a line,
  we do not expect linear interval classifiers to be optimal.
  See \cref{sec:threeclass:beyond:line} for a more detailed discussion
  with some conjectures.
}}
Namely, we have the following \lcnamecref{conj:opt:threeclass:linear}:

\medskip
\begin{conjecture}[Linear interval classifiers are optimal across all classifiers]
\label{conj:opt:threeclass:linear}
Under the assumptions of \cref{thm:opt:threeclass:threshopt},
the optimal linear interval classifier from \cref{thm:opt:threeclass:threshopt}
is also \underline{optimal across all classifiers}.
\end{conjecture}

While this \lcnamecref{conj:opt:threeclass:linear} is very natural,
a proof for it is still unknown
and appears to be surprisingly nontrivial.
The three-class setting introduces new challenges beyond
the two-class setting.
Indeed we are unaware of any existing optimality results like this
for robust classification beyond two-class settings.
The following subsections make progress
towards rigorously establishing the conjecture,
providing a pair of first results on
the optimality of linear interval classifiers
in multi-class settings.
The two results are obtained using two different theoretical approaches
and
are somewhat complementary as a result:
the first applies for any model parameters but does not consider all classifiers,
while
the second considers all classifiers but only applies for certain model parameters.
For each result,
we explain not only the approach used
to prove it
but also why that approach falls short of
fully establishing the conjecture.
Taken together,
we hope these results and insights
will provide a foundation for future work
on proving the conjecture.

\subsubsection{Optimality across ``ignore/separate'' classifiers}
\label{sec:opt:threeclass:ignore:separate}

Recall that the optimal linear interval classifier
derived in \cref{thm:opt:threeclass:threshopt}
takes one of two forms.
It either:
1) \emph{ignores} the zero class
(in the ``rare zero class'' case),
or
2) \emph{separates} the positive and negative classes by at least $2\ep$
(in the ``frequent zero class'' case).
In other words,
the optimal linear interval classifier
belongs to
the following general family of classifiers
that we will refer to as ``ignore/separate'' classifiers.

\medskip
\begin{definition}[Ignore/separate classifiers]
  A classifier $\hty : \bbR^p \to \{-1,0,1\}$
  is an ignore/separate classifier
  if either
  \begin{equation*}
    \forall_{x} \;\; \hty(x) \neq 0
    \qquad\qquad
    \text{or}
    \qquad\qquad
    \inf\big\{\|x_+ - x_-\|_2 : \hty(x_+) = 1 \text{ and } \hty(x_-) = -1\big\}
    > 2\ep
    ,
  \end{equation*}
  i.e., $\hty$ either ignores the zero class
  or separates its positive and negative decision regions
  by at least $2\ep$.
\end{definition}

This is a very large family of classifiers
and importantly includes numerous nonlinear classifiers.
Moreover,
the classifiers it omits
are all unlikely to be optimal.
The omitted classifiers
all include the zero class
but still have positive and negative decision regions
within $2\ep$.
Recall that
for linear interval classifiers,
doing so is always suboptimal;
the robust risk can be improved by absorbing the zero decision region
into one of the others,
i.e., by ignoring the zero class.
While this does not necessarily imply that
the same holds
for nonlinear classifiers,
it does make the family of ignore/separate classifiers
a very natural one to consider.

The main result of this subsection is that
the classifier from \cref{thm:opt:threeclass:threshopt}
is not only optimal among linear interval classifiers
but also across the family of ignore/separate classifiers.
The result places no additional
conditions on the model parameters;
it holds for any class proportions ($\pi_-$, $\pi_0$, $\pi_+$),
any spacing of the means ($\lambda_+$, $\lambda_-$), and so on.

\medskip
\begin{theorem}[Linear interval classifiers are optimal across ignore/separate classifiers]
\label{thm:opt:threeclass:ignore:separate}
Under the assumptions of \cref{thm:opt:threeclass:threshopt},
the optimal linear interval classifier from \cref{thm:opt:threeclass:threshopt}
is also optimal across all ignore/separate classifiers.
\end{theorem}

We prove \cref{thm:opt:threeclass:ignore:separate}
using the same overall strategy as the proof of \cref{orc}.
Namely, \cref{thm:opt:threeclass:ignore:separate}
follows directly from
the optimal linear interval classifier
(\cref{thm:opt:threeclass:threshopt})
combined with the following three-class variant
of \cref{thm:opt:twoclass:admissible}
that shows that linear interval classifiers dominate
all ignore/separate classifiers.

\medskip
\begin{lemma}[Linear interval classifiers dominate ignore/separate classifiers]
\label{thm:opt:threeclass:dominate:ignore:separate}
Under the assumptions of \cref{thm:opt:threeclass:threshopt},
for any ignore/separate classifier $\hty : \bbR^p \to \{-1,0,1\}$,
there exists a linear interval classifier $\tly$
that dominates its robust misclassification on all three classes simultaneously, i.e.,
\begin{align*}
  \Pr_{x|y}\{\exists_{\delta : \|\delta\|_2 \leq \ep} \;\; \tly(x+\delta) \neq y\}
  \leq \Pr_{x|y}\{\exists_{\delta : \|\delta\|_2 \leq \ep} \;\; \hty(x+\delta) \neq y\}
  , \quad
  y \in \{-1,0,1\}
  .
\end{align*}
Thus, $\robrisk(\tly,\ep) \leq \robrisk(\hty,\ep)$,
i.e., the linear interval classifier $\tly$ also dominates the ignore/separate classifier $\hty$
with respect to the robust risk.
\end{lemma}
}

The proof \revone{of this \lcnamecref{thm:opt:threeclass:dominate:ignore:separate}}
follows a similar strategy as the proof of \cref{thm:opt:twoclass:admissible}.

\begin{proof}[Proof of \cref{thm:opt:threeclass:dominate:ignore:separate}]
Let $\hty : \bbR^p \to \{-1,0,1\}$ be
\revone{an ignore/separate classifier}
and define
\begin{align*}
  c_- &\coloneqq
    \Phi^{-1}\Big( 1 - \Pr_-\{\hty(x) \neq -1\} \Big) + \revone{\lambda_-}
  , &
  c_+ &\coloneqq
    \Phi^{-1}\Big(     \Pr_+\{\hty(x) \neq  1\} \Big) + \revone{\lambda_+}
  ,
\end{align*}
where $\Phi$ is the cumulative distribution functions of the standard normal distribution (with $\Phi^{-1}(0) = -\infty$ and $\Phi^{-1}(1) = \infty$),
and where we denote the conditional probabilities
$\Pr_{x | y=-1}$, $\Pr_{x | y= 0}$ and $\Pr_{x | y= 1}$
by $\Pr_-$, $\Pr_0$ and $\Pr_+$, respectively.
As one can readily verify, $c_-$ and $c_+$ match the misclassification probabilities of $\hty$
on the negative and positive classes:
\begin{align} \label{thm:opt:threeclass:dominate:misclass}
  \Pr_-\{x^\top\revone{\mu} > c_-\} &= \Pr_-\{\hty(x) \neq -1\}
  , &
  \Pr_+\{x^\top\revone{\mu} < c_+\} &= \Pr_+\{\hty(x) \neq  1\}
  .
\end{align}
We first show that $c_- \leq c_+$
so that they define a valid linear interval classifier.
Let
\begin{align*}
  S_-^c &\coloneqq \{x : x^\top\revone{\mu} > c_-\}
  , &
  \htS_-^c &\coloneqq \{x : \hty(x) \neq -1\}
  ,
\end{align*}
where we view $S_-^c$ as the rejection region
of a hypothesis test of the negative class against the alternative of the zero class;
likewise for $\htS_-^c$.
Under this interpretation,
it follows from \cref{thm:opt:threeclass:dominate:misclass}
that the two tests have matching significance levels
$\Pr_-(S_-^c) = \Pr_-(\htS_-^c)$.
However, $S_-^c$ is the rejection region of the likelihood ratio test here.
Thus, it follows from the Neyman-Pearson lemma
that the test corresponding to $\htS_-^c$ must have power
less than or equal to that of the test corresponding to $S_-^c$.
Namely,
\begin{equation*}
  \Pr_0\{x^\top\revone{\mu} > c_-\}
  = \Pr_0(S_-^c)
  \geq
  \Pr_0(\htS_-^c)
  = \Pr_0\{\hty(x) \neq -1\}
  .
\end{equation*}
Repeating the analogous argument for $c_+$ yields
\begin{equation*}
  \Pr_0\{x^\top\revone{\mu} < c_+\} \geq \Pr_0\{\hty(x) \neq  1\}
  .
\end{equation*}
As a result,
\begin{equation*}
  \Pr_0\{x^\top\revone{\mu} > c_-\} + \Pr_0\{x^\top\revone{\mu} < c_+\}
  \geq \Pr_0\{\hty(x) \neq -1\} + \Pr_0\{\hty(x) \neq  1\}
  \geq 1
  ,
\end{equation*}
from which we conclude that $c_- \leq c_+$.

Since $c_- \leq c_+$,
the following linear \revone{interval} classifier
is well defined:
\begin{equation*}
  \tly(x)
  =
  \begin{cases}
    -1 & \text{if } x^\top\revone{\mu} \leq c_-, \\
     0 & \text{if } c_- < x^\top\revone{\mu} < c_+, \\
    +1 & \text{if } x^\top\revone{\mu} \geq c_+,
  \end{cases}
\end{equation*}
\revone{and} it remains to show that
$\Pr_{x|y}\{S_y^c(\tly) + B_\ep\} \leq \Pr_{x|y}\{S_y^c(\hty) + B_\ep\}$
for $y \in \{-1,0,1\}$.

Applying the equality case of the Gaussian concentration of measure, similar to equations \cref{gcm1,gcm2} in the proof of \cref{thm:opt:twoclass:admissible} in the two-class case, to the construction of $c_-$ and $c_+$
immediately yields
\begin{align*}
  \Pr_-\{S_-^c(\tly) + B_\ep\}
    &\leq \Pr_-\{S_-^c(\hty) + B_\ep\}
  , &
  \Pr_+\{S_+^c(\tly) + B_\ep\}
    &\leq \Pr_+\{S_+^c(\hty) + B_\ep\}
  ,
\end{align*}
so the result is shown for the negative and positive classes.

\revone{
For the zero class,
suppose first that $\hty : \bbR^p \to \{-1,0,1\}$
ignores the zero class,
i.e., $\forall_{x} \;\; \hty(x) \neq 0$.
Then, the decision region $S_0(\hty)$ is empty
so we immediately have
\begin{equation*}
  \Pr_0\{S_0^c(\tly) + B_\ep\}
  \leq 1
  = \Pr_0\{S_0^c(\hty) + B_\ep\}
  .
\end{equation*}
So it only remains to consider the case
where $\hty$ separates its positive and negative decision regions
by at least $2\ep$.
For this case,}
using $\Pr\{C \cup D\} \le \Pr\{C\} + \Pr\{D\}$ for any measurable sets $C,D$,
\revone{yields}
\begin{gather*}
  \Pr_0\{S_0^c(\tly) + B_\ep\}
  = \Pr_0[ \{S_-(\tly) + B_\ep\} \cup \{S_+(\tly) + B_\ep\} ]
  \le
    \Pr_0\{S_-(\tly) + B_\ep\}
  + \Pr_0\{S_+(\tly) + B_\ep\}
  ,
\end{gather*}
and another application of the equality case in \revone{the} Gaussian concentration of measure yields
\begin{align*}
  \Pr_0\{S_-(\tly) + B_\ep\}
    &\leq \Pr_0\{S_-(\hty) + B_\ep\}
  , &
  \Pr_0\{S_+(\tly) + B_\ep\}
    &\leq \Pr_0\{S_+(\hty) + B_\ep\}
  .
\end{align*}
Putting these together yields
\begin{align*}
  \Pr_0\{S_0^c(\tly) + B_\ep\}
  &\leq
    \Pr_0\{S_-(\tly) + B_\ep\}
  + \Pr_0\{S_+(\tly) + B_\ep\}
  \leq
    \Pr_0\{S_-(\hty) + B_\ep\}
  + \Pr_0\{S_+(\hty) + B_\ep\}
  \\&
  = \Pr_0\{S_-(\hty) + B_\ep\}
  + \Pr_0\{S_+(\hty) + B_\ep\}
  - \Pr_0\{(S_-(\hty) + B_\ep) \cap (S_+(\hty) + B_\ep)\}
  \\&
  = \Pr_0\{(S_-(\hty) + B_\ep) \cup (S_+(\hty) + B_\ep)\}
  = \Pr_0\{S_0^c(\hty) + B_\ep\}
  ,
\end{align*}
where the first equality holds because \revone{the} $2\ep$-separation of
\revone{the positive and negative decision regions of} $\hty$
implies that $(S_-(\hty) + B_\ep) \cap (S_+(\hty) + B_\ep) = \emptyset$,
the second equality is the inclusion-exclusion principle,
and
\begin{equation*}
  (S_-(\hty) + B_\ep) \cup (S_+(\hty) + B_\ep)
  = (S_-(\hty) \cup S_+(\hty)) + B_\ep
  = S_0^c(\hty) + B_\ep
  .
\end{equation*}
yields the third equality.
\end{proof}

\revone{
Rigorously extending \cref{thm:opt:threeclass:ignore:separate}
beyond ignore/separate classifiers
turns out to be quite nontrivial.
The proof technique used here encounters a major obstacle,
as we will explain in the remainder of this subsection.
The main issue is that
the proof of the core lemma (\Cref{thm:opt:threeclass:dominate:ignore:separate})
crucially uses the fact that
linear interval classifiers can match the robust misclassification
of any ignore/separate classifier
with respect to \emph{all three classes simultaneously}.
It is tempting to expect this fact to extend
beyond ignore/separate classifiers to all classifiers.
However, this is surprisingly not the case.
}%
\revone{Namely}, there \revone{can exist a (nonlinear) classifier}
for which \emph{all} linear \revone{interval} classifiers
have worse robust misclassification on at least one class,
as shown by the following counter-example.

\medskip
\begin{example}[A classifier for which \revone{no linear interval classifier has matching robust classification performance on all three classes simultaneously}]
\label{ex:not:dominate:all:classes}
Let $p=1$, $\mu = 1$, \revone{$\lambda_\pm = \pm 1$,} $\ep = 0.3$, \revone{$\pi_\pm = 1/4$ and $\pi_0 = 1/2$,}
and consider the classifier
\begin{equation*}
  \hty(x)
  =
  \begin{cases}
    -1, & \text{if } x < 1
    , \\
    1,  & \text{if } 1 \leq x < 2.15
    , \\
    0,  & \text{if } 2.15 \leq x < 4
    , \\
    1,  & \text{if } x \geq 4
    .
  \end{cases}
\end{equation*}
For this classifier, the robust misclassification probabilities are:
\begin{align*}
  \clM_- &=
  \Pr_{x|y=-1}\{\exists_{\delta : \|\delta\|_2 \leq \ep} \;\; \hty(x+\delta) \neq -1\}
  = \Pr_{x|y=-1}(x \geq 1-0.3)
  \approx 0.0446
  , \\
  \clM_0 &=
  \Pr_{x|y=0}\{\exists_{\delta : \|\delta\|_2 \leq \ep} \;\; \hty(x+\delta) \neq 0\}
  = \Pr_{x|y=0}(x < 2.15+0.3) + \Pr_{x|y=0}(x \geq 4-0.3)
  \approx 0.9930
  , \\
  \clM_+ &=
  \Pr_{x|y=1}\{\exists_{\delta : \|\delta\|_2 \leq \ep} \;\; \hty(x+\delta) \neq 1\}
  = \Pr_{x|y=1}(x < 1+0.3) + \Pr_{x|y=1}(2.15-0.3 \leq x < 4+0.3)
  \approx 0.8151
  ,
\end{align*}
where we use $\approx$ to indicate that the calculated values have been rounded to the digits shown.

Now for a linear \revone{interval} classifier
\begin{equation*}
  \tly(x)
  =
  \begin{cases}
    -1, & \text{if } x \leq c_-
    , \\
    0,  & \text{if } c_- < x < c_+
    , \\
    1,  & \text{if } x \geq c_+
    ,
  \end{cases}
\end{equation*}
to match the robust misclassification probabilities $\clM_-$ and $\clM_+$,
i.e., to have robust misclassification no worse on the negative and positive classes,
the thresholds must satisfy
\begin{align*}
  c_- &\geq \tlc_- \coloneqq \brPhi^{-1}(\clM_-) + \revone{\lambda_-} + \ep
  \approx 1.000
  , &
  c_+ &\leq \tlc_+ \coloneqq \Phi^{-1}(\clM_+) + \revone{\lambda_+} - \ep
  \approx 1.597
  ,
\end{align*}
where $\Phi$ is the cumulative distribution function of the standard normal distribution,
$\brPhi \coloneqq 1 - \Phi$,
and $\Phi^{-1}$ and $\brPhi^{-1}$ are their inverses.
Otherwise, if $c_- < \tlc_-$ then
\begin{equation*}
  \Pr_{x|y=-1}\{\exists_{\delta : \|\delta\|_2 \leq \ep} \;\; \tly(x+\delta) \neq -1\}
  = \Pr_{x|y=-1}(x > c_- - \ep)
  > \Pr_{x|y=-1}(x > \tlc_- - \ep)
  = \clM_-
  ,
\end{equation*}
and likewise if $c_+ > \tlc_+$ then
\begin{equation*}
  \Pr_{x|y=1}\{\exists_{\delta : \|\delta\|_2 \leq \ep} \;\; \tly(x+\delta) \neq 1\}
  = \Pr_{x|y=1}(x < c_+ + \ep)
  > \Pr_{x|y=-1}(x < \tlc_+ + \ep)
  = \clM_+
  .
\end{equation*}
However, if $c_- \geq \tlc_-$ and $c_+ \leq \tlc_+$ then
\begin{align*}
  &\Pr_{x|y=0}\{\exists_{\delta : \|\delta\|_2 \leq \ep} \;\; \tly(x+\delta) \neq 0\}
  \\
  &\quad
  = \Pr_{x|y=0}(x \leq c_- + \ep \text{ or } x \geq c_+ - \ep)
  \geq \Pr_{x|y=0}(x \leq \tlc_- + \ep \text{ or } x \geq \tlc_+ - \ep)
  = 1
  > \clM_0
  ,
\end{align*}
since $\tlc_- + \ep \geq \tlc_+ - \ep$ here.
Hence, there is no choice of $c_-$ and $c_+$,
i.e., there is no linear \revone{interval} classifier $\tly$,
that matches
the robust misclassification probabilities of $\hty$
for \emph{all} classes simultaneously.
\end{example}

\revone{
As a result,
the strategy used to prove \cref{thm:opt:threeclass:dominate:ignore:separate}
(and consequently \cref{thm:opt:threeclass:ignore:separate})
cannot be directly used to go beyond ignore/separate classifiers.
Note, however, that this does not mean that linear interval classifiers
are not in fact optimal;
it simply shows that the approach used before cannot be used here.
Indeed, the nonlinear classifier $\hty$
considered in \cref{ex:not:dominate:all:classes}
has robust risk
\begin{equation*}
  \robrisk(\hty, \ep)
  = \pi_- \clM_- + \pi_0 \clM_0 + \pi_+ \clM_+
  \approx 0.7114
  ,
\end{equation*}
while the corresponding optimal linear interval classifier $\intclass^*$
from \cref{thm:opt:threeclass:threshopt}
has worse robust misclassification probabilities
on the positive and negative classes
but still a better robust risk of $\robrisk(\intclass^*, \ep) \approx 0.4953$.
}

\newcommand{\robriskmean}[1]{R_{\mathrm{rob},#1}}
\newcommand{\stdriskmean}[1]{R_{\mathrm{std},#1}}
\newcommand{\bayesclassmean}[1]{\hty^*_{\mathrm{Bay},#1}}

\revone{
\subsubsection{Optimality in the ``sufficiently rare zero class'' case}
\label{sec:opt:threeclass:suff:small}

This subsection derives an optimality result
that does not restrict the family of classifiers,
but instead considers a restricted subset of model parameters.
In particular,
we show that
the classifier from \cref{thm:opt:threeclass:threshopt}
is optimal across all classifiers
if the zero class is sufficiently rare.

\medskip
\begin{theorem}[Linear interval classifiers are optimal if the zero class is sufficiently rare]
\label{thm:opt:threeclass:suff:small}
Under the assumptions of \cref{thm:opt:threeclass:threshopt},
suppose further that
$\pi_0 \leq \htalpha \sqrt{\pi_- \pi_+}$,
where
\begin{equation*}
  \htalpha
  \coloneqq
  \exp\bigg\{
    - \frac{(|\lambda_+|-\ep)(|\lambda_-|-\ep)}{2}
    - \frac{\lambda_+ + \lambda_-}{|\lambda_+ - \lambda_-| - 2\ep} \ln\gamma
  \bigg\}
  .
\end{equation*}
Then the optimal linear interval classifier from \cref{thm:opt:threeclass:threshopt}
(using the ``rare zero class'' case)
is also optimal across all classifiers.
\end{theorem}

This result does not have a restriction on the classifier family
like \cref{thm:opt:threeclass:ignore:separate};
we find the optimal classifier across all classifiers.
However,
the additional condition that $\pi_0 \leq \htalpha \sqrt{\pi_- \pi_+}$
essentially limits this result to a subset of
the ``rare zero class'' case in \cref{thm:opt:threeclass:threshopt};
it does not apply to cases where we expect the optimal robust classifier
to assign points to all three classes.
As we explain at the end of this subsection
(after proving \cref{thm:opt:threeclass:suff:small}),
extending this result to remove the condition
turns out to be quite nontrivial.

We obtain \cref{thm:opt:threeclass:suff:small}
through a different approach
than \cref{orc,thm:opt:threeclass:ignore:separate}.
In particular,
we prove \cref{thm:opt:threeclass:suff:small}
by using the same overall strategy
as was used in \cite{dan2020sharp}
for two-class settings.
In \cite{dan2020sharp},
an optimal robust classifier
was derived by identifying
a careful $\ep$-perturbation of the means
for which the corresponding Bayes optimal classifier
has standard risk
matching the robust risk with respect to the unperturbed means.
Optimality then followed
by exploiting the insight
that the $\ep$-robust risk
of any classifier
is lower bounded by its standard risk
with respect to any $\ep$-perturbed means,
which is in turn lower bounded
by the standard risk of the corresponding Bayes optimal classifier.
The proof of \cref{thm:opt:threeclass:suff:small}
follows the same approach.

\begin{proof}[Proof of \cref{thm:opt:threeclass:suff:small}]
  Consider the following $\ep$-perturbed spacings for the means:
  \begin{align*}
    \lambda_-' &\coloneqq \lambda_- + \ep
    = -(|\lambda_-| - \ep)
    < 0
    , &
    \lambda_0' &\coloneqq \lambda_0 = 0
    , &
    \lambda_+' &\coloneqq \lambda_+ - \ep
    = |\lambda_+| - \ep
    > 0
    .
  \end{align*}
  We will first show that
  the Bayes optimal classifier
  $\bayesclassmean{\lambda'}(x)$
  for the resulting perturbed means
  has standard risk
  $\stdriskmean{\lambda'}(\bayesclassmean{\lambda'})$
  with respect to the perturbed means
  matching its robust risk
  $\robrisk(\bayesclassmean{\lambda'}, \ep)$
  with respect to the unperturbed means.
  Note first that
  $\bayesclassmean{\lambda'}(x)$ ignores the zero class
  when $\pi_0 \leq \tlalpha \sqrt{\pi_- \pi_+}$.
  Indeed,
  if $\pi_0 \leq \tlalpha \sqrt{\pi_- \pi_+}$,
  then
  \begin{equation*}
    \ln\bigg( \frac{\pi_0}{\sqrt{\pi_- \pi_+}} \bigg)
    \leq
    \ln \tlalpha
    =
    - \frac{(|\lambda_+|-\ep)(|\lambda_-|-\ep)}{2}
    - \frac{\lambda_+ + \lambda_-}{|\lambda_+ - \lambda_-| - 2\ep} \ln\gamma
    =
    \frac{\lambda_+' \lambda_-'}{2}
    - \frac{\lambda_+' + \lambda_-'}{\lambda_+' - \lambda_-'} \ln\gamma
    .
  \end{equation*}
  Now,
  adding $\ln\gamma$ then simplifying,
  and likewise subtracting $\ln\gamma$ then simplifying,
  yields the following two inequalities
  \begin{align*}
    \ln\bigg( \frac{\pi_0}{\pi_+} \bigg)
    &
    =
    \ln\bigg( \frac{\pi_0}{\sqrt{\pi_- \pi_+}} \bigg)
    - \ln\gamma
    \leq
    \frac{\lambda_+' \lambda_-'}{2}
    - \frac{\lambda_+' + \lambda_-'}{\lambda_+' - \lambda_-'} \ln\gamma
    - \ln\gamma
    =
    +
    \bigg[
      \frac{\lambda_-'}{2}
      - \frac{\ln(\pi_+/\pi_-)}{|\lambda_+' - \lambda_-'|}
    \bigg]
    |\lambda_+'|
    ,
    \\
    \ln\bigg( \frac{\pi_0}{\pi_-} \bigg)
    &
    =
    \ln\bigg( \frac{\pi_0}{\sqrt{\pi_- \pi_+}} \bigg)
    + \ln\gamma
    \leq
    \frac{\lambda_+' \lambda_-'}{2}
    - \frac{\lambda_+' + \lambda_-'}{\lambda_+' - \lambda_-'} \ln\gamma
    + \ln\gamma
    =
    -
    \bigg[
      \frac{\lambda_+'}{2}
      - \frac{\ln(\pi_+/\pi_-)}{|\lambda_+' - \lambda_-'|}
    \bigg]
    |\lambda_-'|
    .
  \end{align*}
  Next, dividing the first inequality by $+|\lambda_+'|$ then adding $\lambda_+'/2$,
  and dividing the second inequality by $-|\lambda_-'|$ then adding $\lambda_-'/2$,
  yields
  \begin{align*}
    \lambda_+'/2
    +
    \ln(\pi_0/\pi_+)/|\lambda_+'|
    &
    \leq
    (\lambda_+' + \lambda_-')/2
    - \ln(\pi_+/\pi_-)/|\lambda_+' - \lambda_-'|
    ,
    \\
    \lambda_-'/2
    -
    \ln(\pi_0/\pi_-)/|\lambda_-'|
    &
    \geq
    (\lambda_+' + \lambda_-')/2
    - \ln(\pi_+/\pi_-)/|\lambda_+' - \lambda_-'|
    ,
  \end{align*}
  so the Bayes optimal classifier for the perturbed means
  is the following linear interval classifier
  with Bayes optimal thresholds from \cref{eq:threeclass:bayesopt}:
  \begin{equation*}
    \bayesclassmean{\lambda'}(x)
    =
    \intclass(x; \mu, c_+', c_-')
    ,
    \quad
    \text{where }
    c_+' = c_-'
    = \frac{\lambda_+' + \lambda_-'}{2} - \frac{\ln(\pi_+/\pi_-)}{|\lambda_+' - \lambda_-'|}
    = \frac{\lambda_+ + \lambda_-}{2} + \frac{\ln(\pi_-/\pi_+)}{|\lambda_+ - \lambda_-| - 2\ep}
    .
  \end{equation*}
  This classifier is exactly the optimal linear interval robust classifier
  from \cref{thm:opt:threeclass:threshopt}
  in the ``rare zero class'' case,
  and it ignores the zero class.
  Since $\bayesclassmean{\lambda'}(x)$ ignores the zero class,
  we finally have
  \begin{align}
    \label{eq:opt:threeclass:perturbed:bayes:match}
    \robrisk(\bayesclassmean{\lambda'}, \ep)
    &
    = \pi_0
    + \pi_+ \Pr_{x | y=+1}(x^\top \mu < c_+' + \ep)
    + \pi_- \Pr_{x | y=-1}(x^\top \mu > c_-' - \ep)
    \\& \nonumber
    = \pi_0
    + \pi_+ \Pr_{\tlx \sim \clN(\lambda_+,1)}(\tlx < c_+' + \ep)
    + \pi_- \Pr_{\tlx \sim \clN(\lambda_-,1)}(\tlx > c_-' - \ep)
    \\& \nonumber
    = \pi_0
    + \pi_+ \Pr_{\tlx \sim \clN(\lambda_+',1)}(\tlx < c_+')
    + \pi_- \Pr_{\tlx \sim \clN(\lambda_-',1)}(\tlx > c_-')
    =
    \stdriskmean{\lambda'}(\bayesclassmean{\lambda'})
    .
  \end{align}
  Namely,
  the standard risk
  $\stdriskmean{\lambda'}(\bayesclassmean{\lambda'})$
  of $\bayesclassmean{\lambda'}(x)$
  with respect to the perturbed means
  matches the corresponding robust risk
  $\robrisk(\bayesclassmean{\lambda'}, \ep)$
  with respect to the unperturbed means.

  The proof now concludes
  by observing that
  for any classifier $\hty : \bbR^p \to \{-1,0,1\}$
  \begin{align}
    \label{eq:opt:threeclass:perturbed:bayes:bound}
    &
    \robrisk(\hty,\ep)
    = \pi_+
    \Pr_{x \sim \clN(\lambda_+\mu, I_p)} \big\{
    \exists_{\delta : \|\delta\|_2 \leq \ep}
      \;\;
      \hty(x+\delta) \neq +1
    \big\}
    + \pi_0
    \Pr_{x \sim \clN(0, I_p)} \big\{
      \exists_{\delta : \|\delta\|_2 \leq \ep}
      \;\;
      \hty(x+\delta) \neq 0
    \big\}
    \\&\qquad\qquad\qquad\qquad
    \nonumber
    + \pi_-
    \Pr_{x \sim \clN(\lambda_-\mu, I_p)} \big\{
      \exists_{\delta : \|\delta\|_2 \leq \ep}
      \;\;
      \hty(x+\delta) \neq -1
    \big\}
    \\&\qquad
    \nonumber
    \geq
    \pi_+ \Pr_{x \sim \clN(\lambda_+\mu, I_p)} \big\{ \hty(x-\ep\mu) \neq +1 \big\}
    + \pi_0 \Pr_{x \sim \clN(0, I_p)} \big\{ \hty(x) \neq 0 \big\}
    + \pi_- \Pr_{x \sim \clN(\lambda_-\mu, I_p)} \big\{ \hty(x+\ep\mu) \neq -1 \big\}
    \\&\qquad
    \nonumber
    =
    \pi_+ \Pr_{\tlx \sim \clN(\lambda_+'\mu, I_p)} \big\{ \hty(\tlx) \neq +1 \big\}
    + \pi_0 \Pr_{\tlx \sim \clN(0, I_p)} \big\{ \hty(\tlx) \neq 0 \big\}
    + \pi_- \Pr_{\tlx \sim \clN(\lambda_-'\mu, I_p)} \big\{ \hty(x) \neq -1 \big\}
    \\&\qquad
    \nonumber
    =
    \stdriskmean{\lambda'}(\hty)
    ,
  \end{align}
  so for any classifier $\hty : \bbR^p \to \{-1,0,1\}$
  we finally have that
  \begin{equation*}
    \robrisk(\hty,\ep)
    \geq
    \stdriskmean{\lambda'}(\hty)
    \geq
    \stdriskmean{\lambda'}(\bayesclassmean{\lambda'})
    =
    \robrisk(\bayesclassmean{\lambda'}, \ep)
    ,
  \end{equation*}
  where the first inequality is \cref{eq:opt:threeclass:perturbed:bayes:bound},
  the second inequality is from the definition of the Bayes optimal classifier,
  and the final equality is from \cref{eq:opt:threeclass:perturbed:bayes:match}.
\end{proof}

As mentioned above,
rigorously extending \cref{thm:opt:threeclass:suff:small}
beyond the ``sufficiently rare zero class'' case
turns out to be quite nontrivial.
The proof technique used here encounters a major obstacle,
as we will explain in the remainder of this subsection.
The main issue is that a crucial step in the above proof
was to find $\ep$-perturbations of the means
for which the corresponding Bayes optimal classifier
has standard risk
matching the robust risk with respect to the unperturbed means.
Such perturbations always exist in the two-class setting
studied by \cite{dan2020sharp},
and it is tempting to hope that the same may hold
for the three-class setting we consider.
However,
this is not the case,
as the following counter-example illustrates.

\medskip
\begin{example}[A case where no set of $\ep$-perturbed means produce matching robust and standard risk.]
  \label{ex:noaltmean}
  Let
  $p = 1$,
  $\mu = 1$,
  $\lambda_\pm = \pm 1$,
  $\ep = 0.3$,
  and $\pi_\pm = \pi_0 = 1/3$.
  Setting $\mu = 1$ is without loss of generality.
  Then for any
  $\ep$-perturbed means
  \begin{align*}
    \lambda'_-
    &\in
    [-1.3,-0.7]
    , &
    \lambda'_0
    &\in
    [-0.3,0.3]
    , &
    \lambda'_+
    &\in
    [0.7,1.3]
    ,
  \end{align*}
  the robust risk $\robrisk(\bayesclassmean{\lambda'},\ep)$
  is lower bounded by the robust risk of the optimal linear interval classifier
  $\intclass^*$
  derived in \cref{thm:opt:threeclass:threshopt},
  i.e.,
  \begin{align*}
    \forall_{\lambda'_- \in [-1.3,-0.7]}
    \;\;
    \forall_{\lambda'_0 \in [-0.3,0.3]}
    \;\;
    \forall_{\lambda'_+ \in [0.7,1.3]}
    \quad
    \robrisk(\bayesclassmean{\lambda'},\ep)
    \geq
    \robrisk(\intclass^*,\ep)
    \approx 0.56
    ,
  \end{align*}
  since $\bayesclassmean{\lambda'}$ is always a linear interval classifier here
  and hence sub-optimal with respect to $\intclass^*$.

  However, the standard risk
  $\stdriskmean{\lambda'}(\bayesclassmean{\lambda'})$
  with respect to the $\ep$-perturbed means $\lambda'$
  is upper bounded as
  \begin{align*}
    &
    \stdriskmean{\lambda'}(\bayesclassmean{\lambda'})
    \leq \stdriskmean{\lambda'}(\bayesclassmean{\lambda})
    \\&\quad
    = \pi_- \Pr_{x \sim \clN(\lambda'_-,1)} \bigg\{x > -\frac{1}{2}\bigg\}
    + \pi_0 \Pr_{x \sim \clN(\lambda'_0,1)} \bigg\{x < -\frac{1}{2} \text{ or } x > +\frac{1}{2}\bigg\}
    + \pi_+ \Pr_{x \sim \clN(\lambda'_+,1)} \bigg\{x < +\frac{1}{2}\bigg\}
    \\&\quad
    \leq \frac{1}{3} \Pr_{x \sim \clN(-0.7,1)} \bigg\{x > -\frac{1}{2}\bigg\}
    + \frac{1}{3} \Pr_{x \sim \clN(0.3,1)} \bigg\{x < -\frac{1}{2} \text{ or } x > +\frac{1}{2}\bigg\}
    + \frac{1}{3} \Pr_{x \sim \clN(+0.7,1)} \bigg\{x < +\frac{1}{2}\bigg\}
    \\&\quad
    \approx 0.49
    .
  \end{align*}
  Thus,
  there is no choice of $\ep$-perturbed means
  $\lambda_-'$, $\lambda_0'$, and $\lambda_+'$
  for which
  $\robrisk(\bayesclassmean{\lambda'},\ep)
  = \stdriskmean{\lambda'}(\bayesclassmean{\lambda'})$.
\end{example}

Essentially,
the issue is that
the robust misclassification probabilities for the positive and negative classes
can be matched by perturbing the positive and negative means,
respectively,
but the same cannot be done for the zero class in general.
Finding these perturbations is central to the approach
used to prove \cref{thm:opt:threeclass:suff:small}.
As a result,
this approach cannot be directly used to
generalize beyond the sufficiently rare zero class case.

}

% !TEX root = ../adv.tex

\section{Optimal \texorpdfstring{$\ell_\infty$}{linfty} robust classifiers}
\label{sec:opt:linf}

We now shift our attention from $\ell_2$ to $\ell_\infty$ adversaries,
i.e., perturbations up to an $\ell_\infty$ radius,
and seek to minimize the robust risk $\robrisk(\hty,\ep,\|\cdot\|_\infty)$.
Doing so introduces new challenges:
the rotational invariant geometry of $\ell_2$
allowed a reduction to the simpler one-dimensional case,
but this does not apply here.
However, the next result captures one setting
where the geometry is favorable
and the $\ell_2$ findings of \cref{sec:opt:l2:twoclass}
extend to $\ell_\infty$ robustness.
We provide its proof in \cref{app:linf:proof}.

\medskip
\begin{corollary}[Optimal $\ell_\infty$ robust classifiers for one-sparse means]
\label{linf}
Let the data $(x,y)$ follow the two-class Gaussian model \cref{eq:model:twoclass},
and let $\mu$ have exactly one non-zero coordinate $\mu_j > 0$
and $\ep < \mu_j$.
An optimal $\ell_\infty$ robust classifier
is
\begin{equation} \label{eq:robopt:twoclass:linf}
  \hty^*(x)
  \coloneqq
  \sign\{x_j (\mu_j-\ep)-q/2\}
  ,
\end{equation}
where $q = \ln\{(1-\pi)/\pi\}$.
\end{corollary}
In essence, the $\ell_2$ and $\ell_\infty$ norms agree when restricted to the nonzero coordinate,
enabling us to extend \cref{orc}.
The same applies to the three-class setting of \cref{sec:opt:l2:threeclass}
with a similar extension of \cref{thm:opt:threeclass:threshopt} as follows.
We provide its proof in \cref{proof:linf:threeclass}.

\medskip
\begin{corollary}[Optimal \revone{linear} interval $\ell_\infty$ robust classifiers
  for one-sparse means -- three classes]
  \label{linf:threeclass}
Suppose data $(x,y)$ are from
the three-class Gaussian model of \cref{sec:opt:l2:threeclass},
$\mu$ has exactly one non-zero coordinate $\mu_j > 0$
and $\ep < \mu_j/2$.
An optimal \revone{linear} interval $\ell_\infty$ robust classifier is:
\begin{equation*}
  \intclass^*(x)
  \coloneqq
  \intclass(x_j; 1, c_+^*, c_-^*)
  ,
\end{equation*}
where the thresholds $c_+^* \geq c_-^*$ are as follows:
\begin{enumerate}
\item[Case 1.]
If $\pi_0 \leq \alpha^* \sqrt{\pi_- \pi_+}$, then
\begin{equation*}
  c_+^* = c_-^*
  = \ln(\pi_-/\pi_+)/(2\mu_j-2\ep)
  .
\end{equation*}

\item[Case 2.]
Otherwise,
$c_+^* - c_-^* > 2\ep$,
with
\begin{align*}
  c_+^*
  &= + \mu_j/2 + \ln(\pi_0/\pi_+)/(\mu_j-2\ep)
  , &
  c_-^*
  &= - \mu_j/2 - \ln(\pi_0/\pi_-)/(\mu_j-2\ep)
  .
\end{align*}
\end{enumerate}
The cutoff $\alpha^*$
is the unique solution to the equation:
\begin{align*}
  &
  (\gamma + \gamma^{-1})
  \robrisk^*\big\{\mu_j,\gamma/(\gamma + \gamma^{-1});\ep\big\}
  \\
  &\quad=
  (\gamma + \alpha)
  \robrisk^*\big\{
    \mu_j/2,\gamma/(\gamma + \alpha);\ep
  \big\}
  +
  (\gamma^{-1} + \alpha)
  \robrisk^*\big\{
    \mu_j/2,\gamma^{-1}/(\gamma^{-1} + \alpha);\ep
  \big\}
  - \alpha
  , \nonumber
\end{align*}
in the domain $\alpha \geq \exp\{-(\mu_j-2\ep)^2/2\}$
with $\gamma \coloneqq \sqrt{\pi_+/\pi_-}$;
$\alpha^* = \exp(-\mu_j^2/2)$ when $\ep = 0$.
\end{corollary}

Removing the restriction that $\mu$ be one-sparse
is highly nontrivial in general,
but it turns out to be possible if we instead consider only \emph{linear} classifiers:
$\linclass(x;w,c) = \sign(x^\top w - c)$.
This statement for the balanced case has been derived in \revone{\citep[Lemma 1]{goibert2019adversarial}}. However our result generalizes it to the imbalanced case.

\medskip
\begin{theorem}[Optimal linear $\ell_\infty$ robust classifiers]
\label{linflin}
Suppose data $(x,y)$ are from the two-class Gaussian model \cref{eq:model:twoclass}.
An optimal linear $\ell_\infty$ robust classifier is:
$\hty^*(x) \coloneqq \sign\{ x^\top \eta_\ep(\mu) - q/2 \}$,
where $q = \ln\{(1-\pi)/\pi\}$
and the soft-thresholding operator
\begin{equation} \label{eq:softthresholding}
  \eta_\ep(x) \coloneqq
  \begin{cases}
    x-\ep, & \text{if } x \geq \ep
    , \\
    0,     & \text{if } x \in (-\ep,\ep)
    , \\
    x+\ep, & \text{if } x \leq -\ep
    ,
  \end{cases}
\end{equation}
is applied element-wise to the vector $\mu \in \bbR^p$.
\end{theorem}
The proof is
based on a connection to the well-known water-filling optimization problem, and is
 provided in \cref{proof:linflin}.
The analogous extension again holds for three classes.

\medskip
\begin{theorem}[Optimal linear \revone{interval} $\ell_\infty$ robust classifiers -- three classes]
\label{linflin:threeclass}
Suppose data $(x,y)$ are from
the three-class Gaussian model of \cref{sec:opt:l2:threeclass}
and $\ep < \|\mu\|_\infty/2$.
An optimal \revone{linear} interval $\ell_\infty$ robust classifier is either:
\begin{enumerate}
\item $\intclass\{x; \eta_\ep(\mu), c^*, c^*\}$,
where $c^* = \ln(\pi_-/\pi_+)/2$,
or
\item $\intclass\{x; \eta_{2\ep}(\mu), c_+^*, c_-^*\}$,
where $c^*_\pm = \pm \eta_{2\ep}(\mu)^\top\mu/2 \pm \ln(\pi_0/\pi_\pm)$,
\end{enumerate}
where the second case applies only when $c_+^* \geq c_-^*$.
\end{theorem}

We provide its proof in \cref{proof:linflin:threeclass}.

% !TEX root = ../adv.tex

\section{Landscape of the robust risk}
\label{sec:landscape}

\cref{sec:opt:l2:twoclass,sec:opt:l2:threeclass,sec:opt:linf} theoretically optimized the robust risk,
but left open important questions about its \emph{optimization landscape},
which can be non-convex and challenging to optimize.
For example, what happens if we use surrogate losses as is commonly done in practice?
This section makes progress on this question.

Consider data $(x,y)$ from the two-class Gaussian model \cref{eq:model:twoclass}
with linear classifiers
and corresponding $\ell$-robust risk
as a function of weights $w \in \bbR^p$ and bias $c \in \bbR$:
\begin{equation} \label{eq:landscape:risk}
  \tlR_{\ep,\|\cdot\|,\ell}(w,c)
  \coloneqq
  \E_{x,y} \sup_{\|\delta\| \leq \ep}
  \ell[\{w^\top(x+\delta)-c\} \cdot y]
  .
\end{equation}
The 0-1 loss $\smash{\brell(z) = I(z \leq 0)}$ yields
$\smash{\tlR_{\ep,\|\cdot\|,\brell}(w,c)}$
$=$ $\smash{\robrisk\{\sign(w^\top x - c), \ep, \|\cdot\|\}}$.

It is common to
use surrogate losses $\ell$
in \cref{eq:landscape:risk} such as
the logistic loss $\ell(z) = \log(1+\exp(-z))$,
the exponential loss $\ell(z)=\exp(-z)$,
or
the hinge loss $\ell(z) = (1-z)_+$.
The impact of doing so is well-studied
in standard settings
\citep{bartlett2006convexity},
but has remained an important open problem in the adversarial setting.
Minimizing a surrogate loss here does not in general produce optimal weights
for the 0-1 loss,
but it does so in a few settings which the next result describes.

\medskip
\begin{theorem}[Classification calibration]
\label{cons}
Let $w^* \in \bbR^p$ be the optimal weights for a linear classifier with no bias term,
i.e., $w^*$ minimizes $\smash{\tlR_{\ep,\|\cdot\|,\brell}(w,0)}$ with the 0-1 loss $\brell$.
Any strictly decreasing surrogate loss $\ell$ is classification calibrated; 
minimizing the $\ell$-robust risk $\tlR_{\ep,\|\cdot\|,\ell}(w,0)$
recovers $w^*$.

Furthermore, calibration extends to the case with bias, i.e., jointly minimizing $\smash{\tlR_{\ep,\|\cdot\|,\ell}(w,c)}$ produces $(w^*,0)$
if either:
i) $\ell$ is convex,
or
ii) the classes are balanced, i.e., $\pi=1/2$.
\end{theorem}

\Cref{cons} partially extends to surrogate losses $\ell$
that are decreasing but not strictly so.
\revone{In this case,} $w^*$ still minimizes the $\ell$-robust risks
but might not do so uniquely.

\revone{
\begin{proof}[Proof of \cref{cons}]
Let $\|\cdot\|_*$ be the dual norm of $\|\cdot\|$. This is defined as $\|w\|_*=\sup w^\top z$, subject to $\|z\|\le1$. Since $\ell$ is decreasing, as is well known \revone{(}see\revone{,} e.g., \cite{khim2018adversarial}\revone{)}, we have
\begin{align*}
R(\ell,w,b,\ep,\|\cdot\|) &= \E_{x,y}\sup_{\|\delta\|\le \ep} \ell([w^\top (x+\delta)+b]\cdot y)\\
 &= \E_{x,y}\ell(y\cdot [w^\top x+b]-\ep\cdot \|w\|_*)
\end{align*}
This shows that for any candidate $w$, the worst-case perturbations are equal to the conjugate of $w$, with respect to the $\|\cdot\|$ norm, namely  $\delta^*(x)=-\hat y(x) \cdot\ep\cdot w^*$, where $w^*$ solves  $\|w\|_*=\sup w^\top z$, subject to $\|z\|\le1$. 

Now, in our case, due to the distributional assumption on the data, we have $y\cdot x\sim \N(\mu, I_p)$. Moreover, $y\cdot w^\top  x\sim \N(w^\top \mu,\|w\|_2^2 I_p)$. It is readily verified that $y\cdot w^\top  x$ is probabilistically independent of $y$. Therefore, we can write, for some $z\sim \N(0,1)$ independent of $y$
\begin{align*}
R(\ell,w,b,\ep,\|\cdot\|)
&= \E_{z,y}\ell(w^\top \mu -\ep\cdot \|w\|_*+ by+  \sigma\cdot \|w\|_2 \cdot z).
\end{align*}

Now we discuss the cases considered in the theorem.

\benum

\item If minimizing restricted to $b=0$, the inner term reduces to $\ell(w^\top \mu -\ep\cdot \|w\|_*+  \sigma\cdot \|w\|_2 \cdot z)$.
\item When the loss is strictly convex, then by Jensen's inequality we obtain
$$\E_{y}\ell(w^\top \mu -\ep\cdot \|w\|_*+ by+  \sigma\cdot \|w\|_2 \cdot z) \ge \ell(w^\top \mu -\ep\cdot \|w\|_*+  \sigma\cdot \|w\|_2 \cdot z).$$
\eenum
In both cases it is enough to minimize the objective
\begin{align*}
R(\ell,w,\ep,\|\cdot\|)
&= \E_{z,y}\ell(w^\top \mu -\ep\cdot \|w\|_*+   \sigma\cdot \|w\|_2 \cdot z).
\end{align*}

Now fix $\|w\|_2=1$. It is readily verified that, when the loss is strictly decreasing and as the normal random variable is symmetric, this is equivalent to maximizing the inner argument. When the loss is decreasing but not necessarily strictly monotonic, maximizing the inner argument is still a sufficient condition that guarantees the risk is minimized; however in this latter case there may be other minimizers of the risk. Therefore, it is enough to maximize the inner argument.

That is, we study maximizing, subject to $\|w\|_2=c>0$,
\begin{align*}
w^\top \mu -\ep\cdot \|w\|_*.
\end{align*}
Given the homogeneity of the norms, we thus conclude that the optimal $w$  minimizing the robust $\ell$-risk
\beq\label{ellrisk}
R(\ell,w,b,\ep,\|\cdot\|) = \E_{x,y}\sup_{\|\delta\|\le \ep} \ell([w^\top (x+\delta)+b]\cdot y).
\eeq
maximize
\begin{align}\label{robo}
\frac{w^\top \mu -\ep\cdot \|w\|_*}{\|w\|_2}.
\end{align}
Next, we study how to minimize the true robust risk. This is similar to the derivation for the optimal robust classifier. We will assume without loss of generality that $\sigma=1$.  As above, recall our general formula:
\begin{align*}
R(\hat y,\ep) &=
\pi\cdot P_{x|y=1}(S_{-1}+B_\ep)
+(1-\pi)\cdot P_{x|y=-1}(S_{1}+B_\ep).
\end{align*}
For a linear classifier $\hat y^*(x)=\sign(x^\top w+b)$, we can restrict without loss of generality to $w$ such that $\|w\|=1$. The classifiers are scale invariant, and so we get the same predictions for all scaled versions of the weights $w$, by changing $b$ appropriately. Then $S_{1}+B_\ep$ is the set of datapoints such that $x^\top w+b\ge -\ep \|w\|_*$. Thus,
\begin{align*}
R(w,b;\ep) &=
\pi\cdot P_{\N(\mu,I)}(x^\top w+b\le \ep \|w\|_*)
+(1-\pi)\cdot P_{\N(-\mu,I)}(x^\top w+b\ge -\ep \|w\|_*)\\
=&\pi \cdot \Phi\left(\ep \|w\|_*-b-\mu^\top w\right)
+(1-\pi)\cdot \Phi\left(\ep \|w\|_*+b-\mu^\top w\right).
\end{align*}
Now we examine the cases of unrestricted bias (general $b$), and zero bias ($b$ constrained to zero) in turn. For the zero bias case we find
\begin{align*}
R(w;\ep)
=&\Phi\left(\ep \|w\|_*-\mu^\top w\right).
\end{align*}
Another way to put this is that for a weight $w$ with unit norm $\|w\|=1$, a linear classifier reduces the effect size from $\mu^\top w$ (which we can assume to be positive, without loss of generality, by flipping the sign if needed), to $\mu^\top w-\ep \|w\|_*$. So the optimal $w$ minimizing the true robust risk solves
\begin{align*}
\sup_w\,\, \mu^\top w-\ep \|w\|_*\,\,\textnormal{ s.t.}\,\,  \|w\|_2=1.
\end{align*}
Recalling again that the original problem is scale-invariant, it follows that this is equivalent to maximizing \eqref{robo}.
Therefore, the optimal linear classifier for the true and surrogate robust risks coincide.

For the general bias case, we recall that the minimizer of $b\to 
\pi \cdot \Phi(c-b)
+(1-\pi)\cdot \Phi(c+b)$ occurs at $b=\ln[(1-\pi)/\pi]/c$. Plugging back, we find that the ``profile risk'', minimized over $b$, equals, with $c(w) \coloneqq \ep \|w\|_*-\mu^\top w$, and $q \coloneqq \ln[(1-\pi)/\pi]$,
\begin{align*}
R_{prof}(w;\ep)
=&\pi \cdot \Phi\left(c(w)-q/c(w)\right)
+(1-\pi)\cdot \Phi\left(c(w)+q/c(w)\right).
\end{align*}
Clearly, this may in general minimizers other than the ones above. This shows that in general, surrogate loss minimization is not consistent. An exception is when $\pi=1/2$, in which case $q=0$, and the optimal bias in the robust risk is $b=0$. This finishes the proof.
\end{proof}
}
% !TEX root = ../adv.tex

\section{Finite sample analysis}
\label{sec:finite:sample}

Having studied optimal population robust classifiers, we now consider
robust linear classifiers learned from finitely many samples
$(x_1,y_1),\dots,(x_n,y_n) \in \bbR^p \times \{\pm 1\}$. 
This section does not assume Gaussianity;
much of our subsequent analysis turns out to not rely on it.
Here, we learn classifiers by minimizing the \emph{empirical} $\ell$-robust risk
with a decreasing loss functional $\ell$:
\oneortwo{%
\begin{equation} \label{eq:empirical:risk}
  \htR^{(n)}_{\ep,\|\cdot\|,\ell}(w,c)
  \coloneqq
  \frac{1}{n}
  \sum_{i=1}^n \sup_{\|\delta\| \leq \ep}
  \ell[\{w^\top(x_i+\delta)-c\} \cdot y_i]
  =
  \frac{1}{n}
  \sum_{i=1}^n
  \ell\{(w^\top x_i - c) \cdot y_i - \ep \|w\|_*\}
  ,
\end{equation}
}{%
\begin{align}
  &
  \htR^{(n)}_{\ep,\|\cdot\|,\ell}(w,c)
  \coloneqq
  \frac{1}{n}
  \sum_{i=1}^n \sup_{\|\delta\| \leq \ep}
  \ell[\{w^\top(x_i+\delta)-c\} \cdot y_i]
  \nonumber \\
  &\qquad=
  \frac{1}{n}
  \sum_{i=1}^n
  \ell\{(w^\top x_i - c) \cdot y_i - \ep \|w\|_*\}
  , \label{eq:empirical:risk}
\end{align}
}%
where $\|\cdot\|_*$ is the dual norm
and the equality holds because $\ell$ is decreasing;
see\revone{,} e.g., \citep{khim2018adversarial}.
Using the 0-1 loss $\smash{\brell(z) = I(z \leq 0)}$ yields
a non-convex and discontinuous empirical robust risk $\smash{\htR^{(n)}_{\ep,\|\cdot\|,\brell}(w,c)}$, 
making optimization challenging.
So one often uses \emph{convex} surrogates instead,
making $\smash{\htR^{(n)}_{\ep,\|\cdot\|,\ell}(w,c)}$ convex
(decreasing convex functions of concave functions are convex).

Hence, the empirical $\ell$-robust risk \cref{eq:empirical:risk}
can be efficiently minimized for $\ell_\infty$ adversaries
with convex decreasing surrogates such as the linear and hinge losses.  Given $n$ samples, this gives optimal weights $\htw_n \in \bbR^p$  and a classifier $\hty_n(x) = \sign(x^\top \htw_n)$, where throughout we will fix the bias $c = 0$.  We will study these classifiers for the linear and hinge losses.  

To study the tradeoff between standard and robust classifiers in finite samples, inspired by \citep{chen2020more}, in \cref{fig:opt-surrogate-gaps} we plot the mean gaps
between the population robust and standard risks,
i.e., $\robrisk(\hty_n,\ep,\|\cdot\|_\infty) - \stdrisk(\hty_n)$,
as a function of the number of samples $n$,
in the two-class Gaussian model \cref{eq:model:twoclass}.
If the gap is large, then the robust risk is much greater than the standard risk for the optimal robust classifiers, yielding an unfavorable tradeoff.
For the linear loss (constrained to $\|w\|_2 \leq 1$ to ensure boundedness),
the gap between the standard and robust risks
is large as $n$ grows, consistent with \citep{chen2020more}.
However under the hinge loss, regardless of the value of $\ep$,
the gap decreases, which had not been investigated in~\citep{chen2020more}.
This shows that the loss functional matters in robust risk minimization,
consistent with our landscape results and expanding on the observations in \citep{chen2020more}.

\oneortwo{%
\begin{figure}
\centering
\begin{minipage}{0.63\textwidth}
    \begin{subfigure}{0.48\textwidth} \centering
        \includegraphics[width=\textwidth]{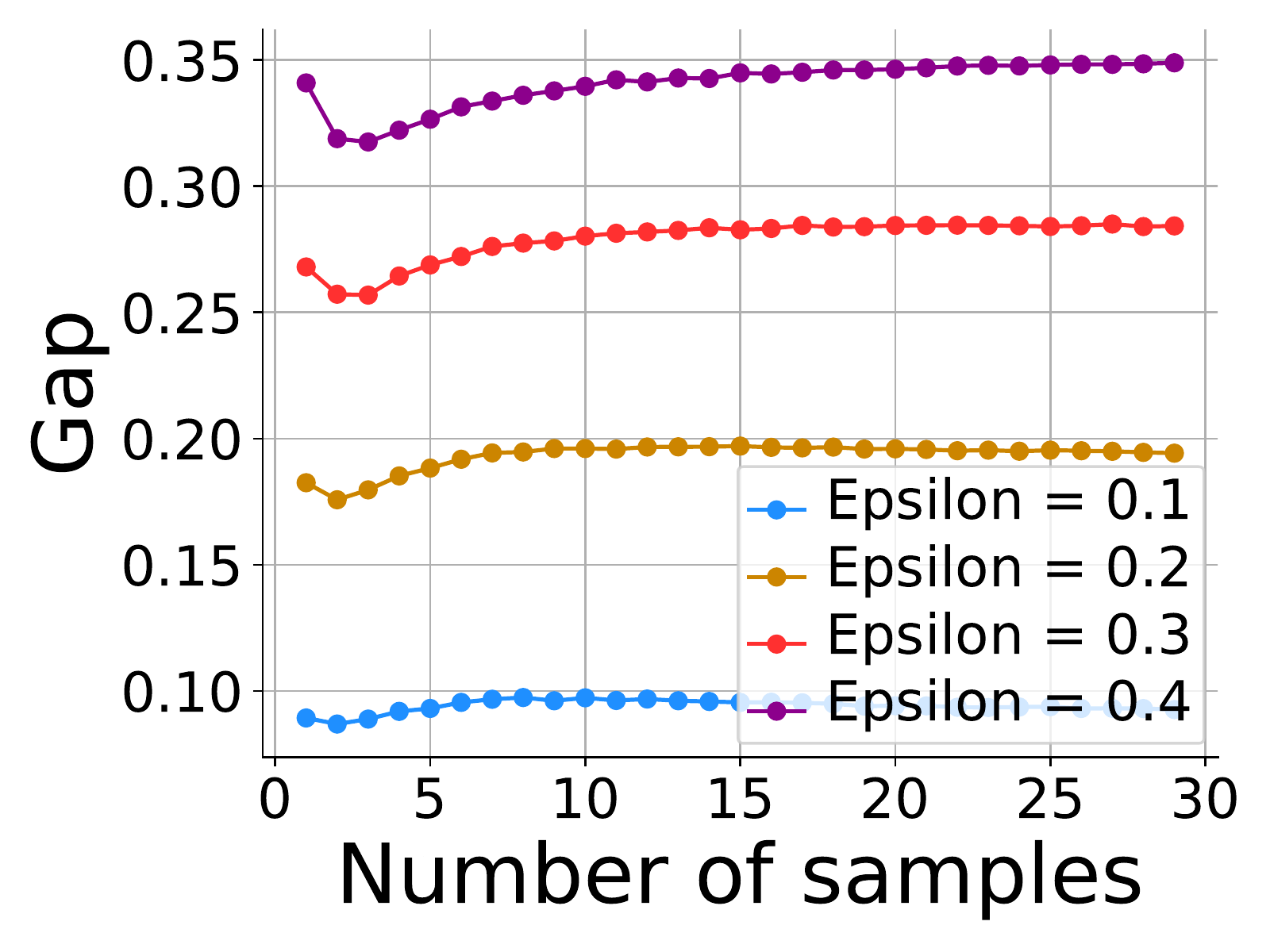}
        \caption{Linear surrogate loss.}
    \end{subfigure}\quad
    \begin{subfigure}{0.48\textwidth} \centering
        \includegraphics[width=\textwidth]{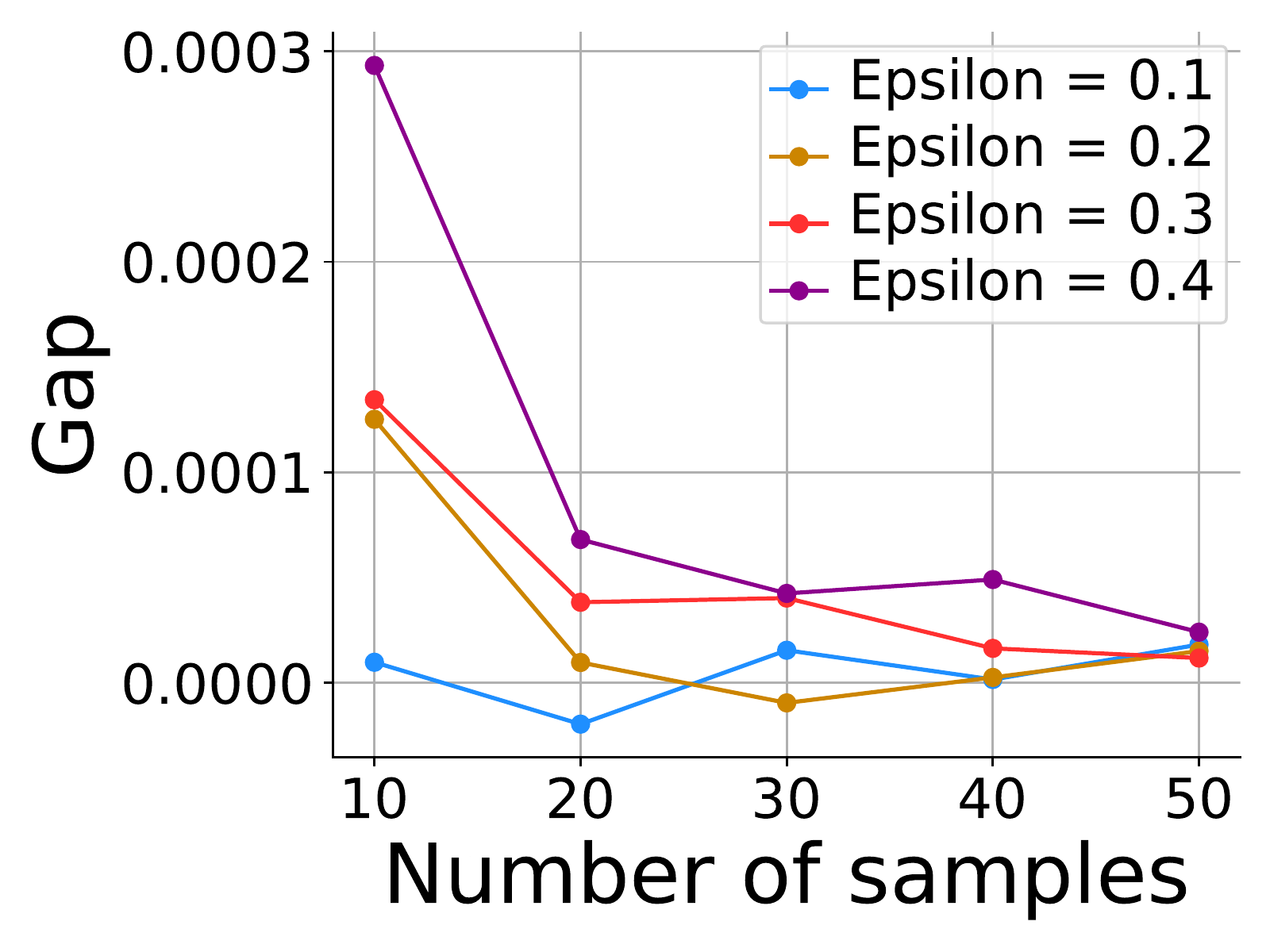}
        \caption{Hinge surrogate loss.}
    \end{subfigure}
    \caption{Mean gap between robust and standard risks
    of optimal finite-sample $\ell_\infty$ robust classifiers obtained via empirical robust risk minimization.  Here we set the dimension $p = 5$, mean vector $\mu = 1/2\cdot \mathbbm{1}$, and class proportion $\pi = 1/2$.}
    \label{fig:opt-surrogate-gaps}
\end{minipage} \quad
\begin{minipage}{0.33\textwidth} \centering
    \includegraphics[width=0.95\textwidth]{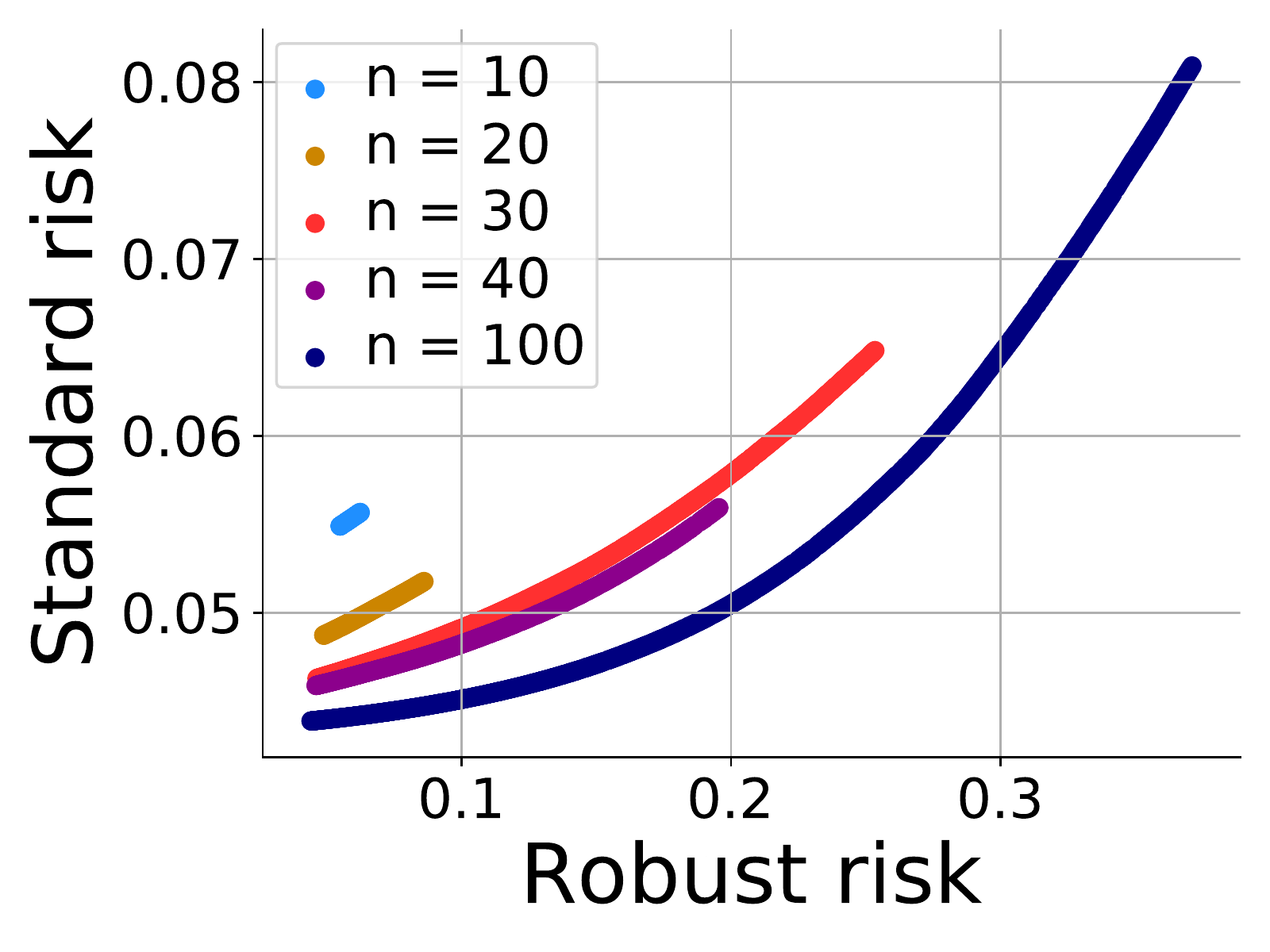}
    \caption{Trade-off between (population) standard and robust risk for $\ep\in[0,1]$ for classifiers obtained via Prop \ref{prop:empirical:risk:opt:linloss}.  Here we set $p = 5$, $\mu = 1/2\cdot \mathbbm{1}$, $\pi = 1/2$.}
\end{minipage}
\end{figure}
}{%
\begin{figure} \centering
  \begin{subfigure}{0.49\linewidth} \centering
    \includegraphics[width=\linewidth]{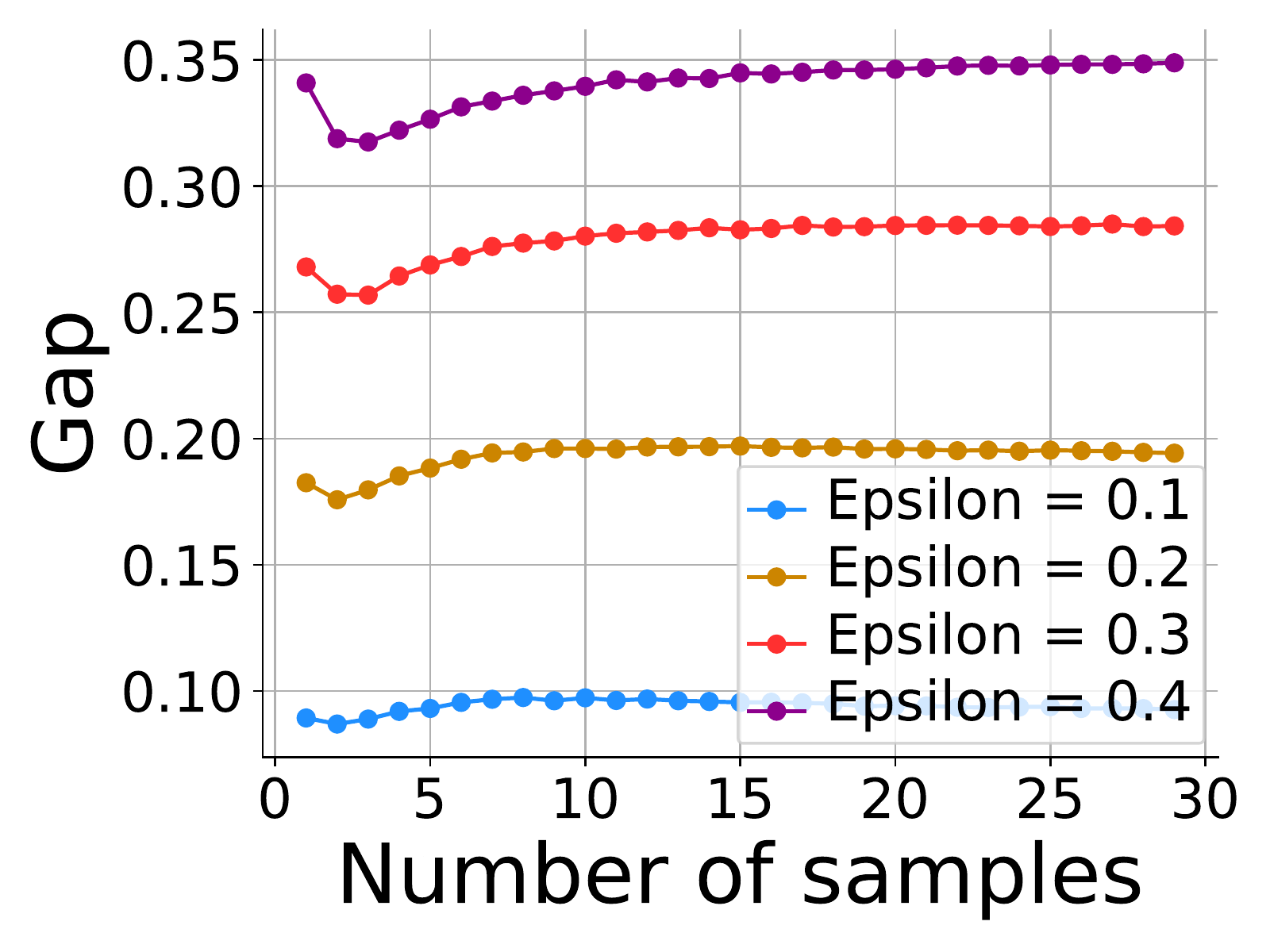}
    \caption{Linear surrogate loss.}
  \end{subfigure} \hfill
  \begin{subfigure}{0.49\linewidth} \centering
    \includegraphics[width=\linewidth]{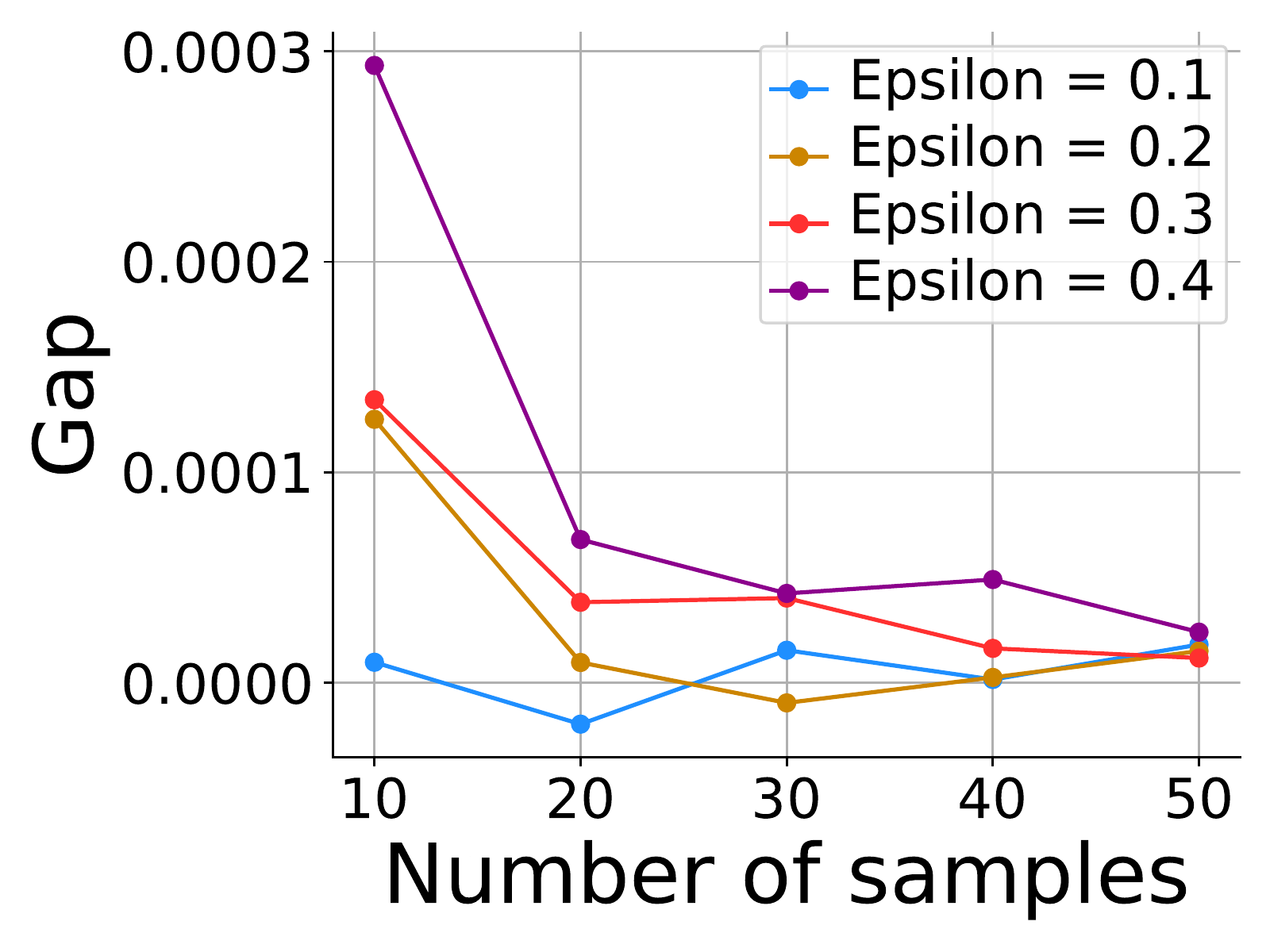}
    \caption{Hinge surrogate loss.}
  \end{subfigure}
  \caption{Mean gap between robust and standard risks
  of optimal finite-sample $\ell_\infty$ robust ERM classifiers.  Here we set the dimension $p = 5$, mean vector $\mu = 1/2\cdot \mathbbm{1}$, and class proportion $\pi = 1/2$.}
  \label{fig:opt-surrogate-gaps}
\end{figure}
\begin{figure} \centering
  \includegraphics[width=0.48\linewidth,trim=0 20pt 0 10pt]{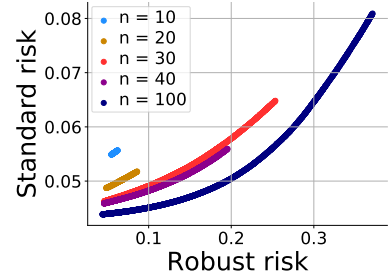}
  \caption{Trade-off between (population) standard and robust risk for $\ep\in[0,1]$ for classifiers obtained via Prop \ref{prop:empirical:risk:opt:linloss}.  Here we set $p = 5$, $\mu = 1/2\cdot \mathbbm{1}$, $\pi = 1/2$.}
\end{figure}
}%

\textbf{Optimal empirical robust classifiers.} The empirical risk-minimization perspective we have described gives an effective procedure for obtaining robust classifiers.  Moreover, in some special cases we can also derive explicit optimal empirical $\ell$-robust classifiers.  The next \lcnamecref{prop:empirical:risk:opt:linloss}
does so for $\ell_\infty$ adversaries with linear loss where we again drop the bias term, i.e.\revone{,} $c = 0$.

\medskip
\begin{proposition}
  \label{prop:empirical:risk:opt:linloss}
  The empirical $\ell_\infty$ robust risk 
  $\smash{\htR^{(n)}_{\ep,\|\cdot\|_\infty,\ell}(w,0)}$
  constrained to $\|w\|_2 \leq 1$
  is minimized for the linear loss $\ell(z) = -z$
  by $w^* \coloneqq \eta_\ep(\htmu) / \norm{\eta_\ep(\htmu)}_2$.  Here $\eta$ is the soft-thresholding operator \cref{eq:softthresholding}, which
  is applied element-wise to the empirical mean vector
  $\htmu \coloneqq (1/n) \sum_{i=1}^n y_i x_i \in \bbR^p$.
\end{proposition}

Interestingly,
these finite-sample weights can be viewed as plug-in estimates of
the \emph{population} optimal weights $\eta_\ep(\mu)$ from \cref{linflin}
for the two-class Gaussian model \cref{eq:model:twoclass},
where the \emph{empirical} mean $\htmu$
is substituted for the population mean $\mu$.  In Figure 4, we illustrate the tradeoff between population standard and robust risk for classifiers obtained via \cref{prop:empirical:risk:opt:linloss}.

\textbf{Convergence of robust risk minimization.}
Here
we quantify the concentration of
the empirical robust risk
$\smash{\htR^{(n)}_{\ep,\|\cdot\|,\brell}(w,c)}$
around its population analogue $\smash{\tlR_{\ep,\|\cdot\|,\brell}(w,c)}$,
where $\brell$ is again the 0-1 loss.  
Notably, in this result $x_i | y_i$
need not be Gaussian.
\medskip
\begin{theorem}[Convergence of empirical robust risk for linear classifiers]
\label{1dfs2}
For any $\delta > 0$,
\oneortwo{%
\begin{equation*}
  \Pr\Big\{
    \forall {(w,c) \in \bbR^p \times \bbR} \quad
    \Big| \htR^{(n)}_{\ep,\|\cdot\|,\brell}(w,c) - \tlR_{\ep,\|\cdot\|,\brell}(w,c) \Big|
    \leq \delta
  \Big\} \geq 1-\exp(C(p -\delta^2n))
\end{equation*}
}{%
\begin{align*}
  &\Pr\Big\{
    \forall {(w,c) \in \bbR^p \times \bbR} \quad
    \Big| \htR^{(n)}_{\ep,\|\cdot\|,\brell}(w,c) - \tlR_{\ep,\|\cdot\|,\brell}(w,c) \Big|
    \leq \delta
  \Big\}
  \\
  &\quad\geq 1-\exp(C(p -\delta^2n))
\end{align*}
}%
where $C$ is a constant independent of $n,d$, and the probability is with respect to the $n$ independent identically distributed samples
$(x_1,y_1),\dots,(x_n,y_n) \in \bbR^p \times \{\pm 1\}$.
\end{theorem}

Put another way, the empirical robust risk concentrates
\emph{uniformly} across all linear classifiers
at a rate $\smash{O(\sqrt{p/n})}$.
The proof uses that the $\ep$-expansions of half-spaces are still half-spaces, enabling arguments by VC-dimension.
Characterizing more general classifiers is highly nontrivial
since the $\ep$-expansion of a finite VC dimension hypothesis class
can have infinite VC dimension \citep{montasser2019vc}.
However, for one-dimensional data it turns out that
we can generalize
to classifiers that assign \emph{finite unions of intervals} to each class.

\medskip
\begin{theorem}[Convergence rate of empirical robust risk in 1D]
\label{1dfs}
In the setting of \cref{1dfs2} 
for any $\delta>0$, we have uniformly over all classifiers $\hty$
whose classification regions are unions of at most $2k$ intervals that
$|R_n(\hty,\ep) - R(\hat y,\ep)| \le \delta$
with probability at least $1-4\exp(-2n\delta^2/k^2)$, for the empirical robust risk $R_n$ and the population robust risk $R$.
\end{theorem}
Thus uniform concentration for $k$ intervals occurs at rate $\smash{O(k/\sqrt{n})}$.
The proof is based on the Dvoretzky–Kiefer–Wolfowitz inequality.

\revone{
\begin{proof}[Proof of \cref{1dfs}]

As we have previously argued, the robust risk can be expressed as follows
$$R(\hat y,\ep)
=
P(y=1)P_{x|y=1}(S_{-1}+B_\ep)
+P(y=-1)P_{x|y=-1}(S_{1}+B_\ep).
$$
Now let $(x_i,y_i)$ for $i=1,\ldots,n$ be sampled iid from a joint distribution $P_{x,y}$ for $i=1,\ldots,n$. Let the fraction of 1-s be $\pi_n\in[0,1]$. Let $P_{n\pm}$ be the empirical distributions of $x_i$ given $y_i=1$ and $-1$, respectively. We can write the finite sample robust risk as 
$$R_n(\hat y,\ep)
=
\pi_n \cdot P_{n+}(S_{-1}+B_\ep)
+(1-\pi_n) \cdot P_{n-}(S_{1}+B_\ep).
$$
Now  $P_{n\pm}$ are empirical distributions that will converge to the limiting distributions under certain conditions.

Consider classifiers $\hat y$ whose decision boundaries are at most $k$ points. For instance, if $k=1$, then these are linear classifiers. Then $S_{\pm 1}$ each consist of a union of at most $j=\lceil k/2 \rceil$ disjoint intervals (finite or semi-infinite). Let $I_j$ denote the collection of all such subsets of the real line, unions of at most $j$ disjoint finite or semi-infinite intervals. Thus $S_{\pm 1} \in I_{j}$. Critically, the $\ep$-expansions also have this property: by expanding the intervals, we still get intervals, merging them as needed. Thus, $S_{\pm 1}+B_\ep \in I_{j}$.  

Now, the classical Dvoretzky–Kiefer–Wolfowitz inequality \citep{dvoretzky1956asymptotic,kosorok2007introduction,shorack2009empirical} states the following. Let $F_n$ be the CDF of $n$ iid samples with CDF $F$.  For every $\delta>0$, 
$$\Pr(\sup_x |F_n(x)-F(x)|>\delta)\le 2\exp(-2n\delta^2).$$
Let $\delta_n(x) = P_{n+}(-\infty,x]-P_{+}(-\infty,x]$. Consider the event $\sup_{c}|\delta_n(c)|\le \delta$, which happens with probability at least $1-2\exp(-2n\delta^2)$. On this event, we have
\begin{align*}
&|P_{n+}(S_{-1}+B_\ep)-P_{+}(S_{-1}+B_\ep)| \le
\sup_{A\in I_{j}}|P_{n+}(A)-P_{+}(A)|\\
&=
\sup_{c_1<c_2<\ldots < c_{j}}|\delta_n(c_1)-\delta_n(c_2)+\delta_n(c_3)-\ldots+(-1)^{j-1}\delta_n(c_j)|\\
&\le j\cdot
\sup_{c}|\delta_n(c)| \le j\delta.
\end{align*}
A similar argument applies to $S_1$. Then, on the intersection of the two events, which happens with probability $1-4\exp(-2n\delta^2)$, we find that
\begin{align*}
|R_n(\hat y,\ep) - R(\hat y,\ep)|
&\le \max_i |P_{n+}(S_{i}+B_\ep)-P_{+}(S_{i}+B_\ep)|\le j\delta.
\end{align*}
as was to be shown.
\end{proof}
}

% !TEX root = ../adv.tex

\section{Conclusion}

In this paper, we studied the tradeoffs inherent to robust classification in the fundamental setting of two- and three-class Gaussian classification models.  In particular, we leveraged that half-spaces are extremal sets with respect to Gaussian isoperimetry to derive $\ell_2$ and $\ell_\infty$ optimal robust classifiers in the imbalanced data setting.  This analysis revealed a fundamental trade-off between accuracy and robustness, which depends on the level of class imbalance in the data.  Indeed, we showed that in this setting, no classifier minimizes both the standard and robust risks simultaneously.  Furthermore, we analyzed the optimization landscape of the robust risk, demonstrating that the optimizers of various convex surrogate losses coincide with the nonconvex robust 0-1 loss.  Finally, we connected our results to empirical robust risk minimization by providing a finite-sample analysis with respect to the 0-1 and surrogate loss functionals.

\bibliographystyle{plainnat-abbrev}
\bibliography{references}

\clearpage
\appendix
\begin{center} \Large \bf
  Appendices
\end{center}

% !TEX root = ../adv.tex

\section{Extensions of \cref{orc}}
\label{app:opt:l2:twoclass:proofs}

\subsection{Connections to randomized classifiers}
\label{app:randomized}

Adding random noise has been used as a heuristic to obtain robust classifiers \cite[see e.g][]{xie2018mitigating,athalye2017synthesizing}. While it has been shown to be attackable via gradient based methods \citep{athalye2018obfuscated}, we can still study it as a heuristic. It turns out that it has connections to optimal robust classifiers in our models.

 In this section, we suppose that the noise level in the data is $\sigma^2$, so $x_i|y_i\sim \N(y\mu,\sigma^2 I_p)$. Suppose we add noise $Z\sim \N(0,\tau^2 I_p)$, for some $\tau^2 >0$, and then train a standard classifier. Note that the Bayes-optimal classifier depends only on the SNR $s(\mu,\sigma^2) = \|\mu\|_2/\sigma$. Thus we get that the Bayes-optimal classifier with noise is the optimal $\ep$-robust classifier if (assuming $\ep<\|\mu\|_2$)
$$s(\|\mu\|_2,\sigma^2+\tau^2)=s(\|\mu\|_2-\ep,\sigma^2)$$
or equivalently if
$$\tau = \sigma \sqrt{\frac{\|\mu\|_2^2}{(\|\mu\|_2-\ep)^2}-1}.$$
Put it another way, our results show that robust classifiers reduce the signal strength. Equivalently, randomized classifiers increase the noise level. However, note that for this the added noise level has to be tuned very carefully.

\subsection{Extension to weighted combinations}
\label{app:weighted:risk}

Given a distribution $Q$ over $\ep$, we can try to minimize $R(\hat y,Q) = \E_{\ep\sim Q} R(\hat y,\ep)$. This leads to classifiers that can achieve various trade-offs between robustness to different sizes of perturbations. For instance, we can minimize $R(\hat y,0)+\lambda \cdot R(\hat y,\ep)$ for some $\lambda>0$. 

It is readily verified that \cref{thm:opt:twoclass:admissible} still holds for $R(\hat y,Q)$, as long as $Q$ is supported on $[0,\|\mu\|)$. This is because the linear classifier $\hat y$ found in the proof of that result does not depend on $\ep$, and reduces the $\ep$ robust risk for all $\ep<\|\mu\|$. Hence, linear classifiers are admissible for $R(\hat y,Q)$.

However, in general there is no analytical expression for the optimal linear classifier. Following \cref{orc}, it is readily verified that  the threshold $c$ in the optimal linear classifier is the unique solution of the equation $\E_{\mu'}\exp(-\mu'^2/2)[\pi \exp(c\mu') + (1-\pi) \exp(-c\mu')]=0$, where $\mu' = \mu-\ep$, and $\ep\sim Q$.

As before, it is enough to solve the 1-D problem. Thus, we want to find the value of the threshold $c$ that minimizes
  \begin{align*}
  R(\hat y_c,Q)
  &=
  P(y=1) \E_{\ep\sim Q} P_{\mu-\ep}(x\le c)
  +P(y=-1)\E_{\ep\sim Q} P_{-\mu+\ep}(x\ge c)\\
    &=
  \pi\cdot \E_{\ep\sim Q} P_{\mu-\ep}(x\le c)
  +(1-\pi)\cdot \E_{\ep\sim Q} P_{-\mu+\ep}(x\ge c)\\
      &=
  \pi\cdot \E_{\mu'} P_{\mu'}(x\le c)
  +(1-\pi)\cdot \E_{\mu'} P_{-\mu'}(x\ge c).
  \end{align*}
Differentiating with respect to $c$, we find that
  \begin{align*}
  R'(c) &\coloneqq dR(\hat y_c,Q)/dc
  \\
  &=
  \pi\cdot \E_{\mu'} \phi(c-\mu')
  -(1-\pi)\cdot \E_{\mu'} \phi(c+\mu')\\
    &=
  (2\pi)^{-1/2}[\pi\cdot \E_{\mu'} \exp[-(c-\mu')^2/2]
  -(1-\pi)\cdot \E_{\mu'} \exp[-(c-\mu')^2/2]].
  \end{align*}
Up to the factor $(2\pi)^{-1/2}$, and also factoring out the term $\exp[-c^2/2]$, which cannot be zero, we find that $R'(c)=0$ iff
  \begin{align*}
  a(c)=\E_{\mu'}  \exp[-\mu'^2/2][\pi\cdot  \exp(c\mu')
  -(1-\pi)\cdot\exp(-c\mu')]=0.
  \end{align*}
  This is exactly the claimed equation for $c$. Now, it is not hard to see that $a(c)$ is strictly increasing with limits $\pm\infty$ at $\pm\infty$. Hence, the solution $c$ exists and is unique.

\subsection{Data with a general covariance}
\label{app:l2:gencov}

A natural question is whether optimality extends to data with general covariance.  To this end, suppose that the data is distributed according to the two-class Gaussian model with an invertible covariance matrix $\Sigma$ so that $x_i\sim\mathcal{N}(y_i\mu, \Sigma)$.  In this setting, we can study the setting when the difference between the population means aligns with the smallest eigenvectors of $\Sigma$.

\medskip
\begin{theorem}[Optimal robust classifiers, general covariance]\label{gencov}  Consider finding $\ell_2$ robust classifiers in the two-class Gaussian classification problem with data $(x_i,y_i)$, $i=1,\ldots,n$, where $y_i=\pm 1$, $x_i\sim \N(y_i \mu,\Sigma)$, where $\Sigma$ is an invertible covariance matrix.  Let $V$ be the span of eigenvectors of $\Sigma$ corresponding to its smallest eigenvalue, and note that this is a nonempty linear space. Suppose that $\mu \in V$.  The optimal $\ell_2$ robust classifiers are linear classifiers
\begin{align*}
    \hat y^*(x)=\sign\left(x^\top \mu\left[1-\frac{\ep}{\|\mu\|}\right]-\lambda^{1/2}\frac{q}{2}\right),
\end{align*}
where $\lambda$ is the smallest eigenvalue of $\Sigma$, and the other symbols are as in \cref{orc}.
\label{thm:opt:twoclass:gencov}
\end{theorem}

This theorem generalizes the result of \cref{orc}.  When the covariance matrix $\Sigma = \sigma^2 I_p$ is diagonal, the optimal robust $\ell_2$ classifier from \cref{thm:opt:twoclass:gencov} is identical to that from \cref{orc}. 

\begin{proof}[Proof of \cref{gencov}]
The proof proceeds along the lines of \cref{orc}, checking that it extends to this setting. We will only sketch the key steps.

We define the $\ep$-expansion of a set $A$ in a norm $\|\cdot\|$ to be the Minkowski sum $A+B_\ep = \{a+b: \, a\in A,\, b\in B_\ep\}$, where $B_\ep=\{x: \, \|x\| \le \ep\}$ is the $\ep$-ball in the given norm.

The key insight is that $\ep$-expansions in $\ell_2$ norm will turn into $\ep$-expansions in the Mahalanobis metric $d_\Sigma(a,b) = [(a-b)^\top \Sigma (a-b)]^{1/2}$. To put it another way, by changing coordinates from $x\to \Sigma^{-1/2}x$, the $\ell_2$ ball transforms to a Mahalanobis ball, i.e., an ellipsoid. We explain this below.

The first critical step was the existence of optimal linear classifiers (\cref{thm:opt:twoclass:admissible}). For any fixed set $S$, the key is to be able to solve the problem \eqref{gcmsol}. This requires us to find the optimal isoperimetric set, i.e., the one with minimal probability under $\ep$-expansion, with respect to the probability measure $\N(\mu,\Sigma)$. 

Now, letting $z-\mu = \Sigma^{-1/2}(x-\mu)$, we can write 
$$P_{x\sim \N(\mu,\Sigma)}(x\in S) = P_{z\sim \N(\mu,I)}(z \in \mu + \Sigma^{-1/2}(S-\mu))$$
and
$$P_{x\sim \N(\mu,\Sigma)}(x\in S+B_\ep) = P_{z\sim \N(\mu,I)}(z \in \mu + \Sigma^{-1/2}(S+B_\ep-\mu)).$$
Let $S' = \mu + \Sigma^{-1/2}(S-\mu)$.
We can write 
$$\mu + \Sigma^{-1/2}(S+B_\ep-\mu) = S' + \Sigma^{-1/2}B_\ep.$$
Now, $\Sigma^{-1/2}B_\ep$ is precisely the $\ep$-ball in the Mahalanobis metric, $\Sigma^{-1/2}B_\ep = \{\Sigma^{-1/2}x:\|x\|_2\le \ep\} = \{z: \|z\|_{\Sigma} \le \ep\}$. Let us call this set $B_{\Sigma,\ep}$.

So problem \eqref{gcmsol} can equivalently be written as
\begin{align*}
\min_{S'} P_{z\sim \N(\mu,I)}(S'+B_{\Sigma,\ep})\,\,
s.t.\,\, P_{z\sim \N(\mu,I)}(S')=\alpha.
\end{align*}

Consider any set $S'$ of the form $v^\top z\le c$, with $\|v\|_2=1$ (i.e., a hyperplane). Then (as can be readily seen by drawing a picture)
$$S'+B_{\Sigma,\ep} = \{z+z': v^\top z\le c, \|\Sigma^{-1/2} z'\|_2\le \ep\}$$
equals the set 
$$S'' = \{v^\top z\le c + \ep\cdot\|\Sigma^{1/2} v\|_2\}.$$
For fixed $c,\ep,\Sigma$, the size of this set (with respect to any probability measure absolutely continuous with respect to Lebesgue measure) can be minimized in $v$ by taking $v$ to lie in the span of the eigenvectors of $\Sigma$ with smallest eigenvalue. 

Thus the sets $S'$ minimizing the expansion are hyperplanes orthogonal to the eigenvectors with \emph{smallest} eigenvalues of $\Sigma$. Then $S = \Sigma^{1/2}\cdot \{z: v^\top z\le c\} = \{x: v^\top \Sigma^{-1/2} x\le c'\}$. Now, since $v$ is an eigenvector of $\Sigma$, i.e., $\Sigma v = \lambda v$, we have that $v'=\Sigma^{-1/2} v  = \lambda^{-1/2} v$ is still a scaled version of $v$. Hence, even in the original coordinate system, the sets $S'$ have the same interpretation.

Next, the second critical step in the proof was to reduce the problem to a one-dimensional classification along the direction of $\mu$. 
This is the case if the eigenvectors align with $\mu$. In that case, after the one-dimensional projection, we have the same problem that we already solved in Theorem \ref{orc}. One can easily verify that the remaining steps go through. This finishes the proof.
\end{proof}

\subsection{Data on a low-dimensional subspace}
\label{app:lowdim}

We can extend the above analysis to low-dimensional data. Suppose that the data $x_i,y_i$ live in a lower dimensional linear space. For simplicity, suppose that $x_i=(x_i^1,0_d)$, so only the first $p'$ coordinates are nonzero, and the remaining $d \revone{\coloneqq} p-p'$ dimensions are zero. This is a model of a low-dimensional manifold. For rotationally invariant problems like $\ell_2$ norm robustness, we can consider instead any low-dimensional affine space, and the same conclusions apply. However, for non-rotationally invariant problems like $\ell_\infty$ norm robustness (studied in detail later), the conclusions only apply to this specific space.

Intuitively, decision boundaries that are not perpendicular to the manifold $M=(x,0_d)$ can have a larger ``expansion'' projected down into the manifold. Hence, their adversarial risk can be larger. This will imply that we can restrict to decision boundaries perpendicular to the manifold, and thus reduce the problem to the previous case.  To this end, we provide the following lemma. We emphasize that this is a purely geometric fact, and holds for any classification problem (not just Gaussian), and any norm (not just $\ell_2$.)

\medskip
\begin{lemma}[Low-dimensional classifiers are admissible]\label{low-dim} Consider any classification problem and robust classifiers for the low-dimensional data model above. For any classifier $\hat y$, the low-dimensional classifier  $\hat y^*(x_i^1,x_i^2)=\hat y(x_i^1,0_d)$ has robust risk is less than or equal to that of the original classifier with respect to any norm $\|\cdot\|$:
$$R(\hat y,\ep)\ge R(\hat y^*,\ep).$$
\end{lemma}

The above claim shows that for low-dimensional data as above, even if we have the data represented as full-length vectors, we can restrict to classifiers that depend only on the first coordinates. This reduces the problem to the one considered before, and all the results derived above are applicable.  In particular, for a low-dimensional two-class Gaussian mixture, low-dimensional linear classifiers are optimal, under the previous conditions.

\begin{proof}[Proof of \cref{low-dim}]
Suppose $S_1$ is the decision region $x:\hat y(x)=1$ where the original classifier outputs the first class. The modified classifier $\hat y^*$ makes the same decision as $\hat y$ restricted to the first $p'$ coordinates.

Then the decision region $S^*_1$ where the modified classifier outputs the first class is the set of vectors $x=(x^1,x^2)$ such that $(x^1,0)\in S_1 \cap M$. We can write $S^*_1$ as the direct product $S^*_1=S^{*,p'}_1\times \R^d$.

Then, the $\ep$-expansion of $S_1^*$ within $M$ is $S_1^{*,p'}+B_\ep^{p'}$, where $B_\ep^{p'}$ is a $p'$-dimensional ball, and we can compute the sum in $p'$-dimensional space. Then, it is readily verified that, by denoting $R_d$ the restriction to the first $p'$ coordinates of a subset of $M$ (i.e., ignoring the last $d$ zero coordinates),

$$S_1^{*,p'}+B_\ep^{p'} \subset R_d[(S_1+B_\ep^{p})\cap M],$$

or equivalently, viewing this as embedded in the $p$-dimensional space,

$$[S_1 \cap M]+(B_\ep^{p'},0_d) \subset (S_1+B_\ep^{p})  \cap M.$$

Indeed, if $z\in [S_1 \cap M]+(B_\ep^{p'},0)$, then $z = x+\delta$, where $x\in S_1 \cap M$ and $\delta\in (B_\ep^{p'},0)$. Then it is clear that $z\in (S_1+B_\ep^{p})  \cap M$. Here we only use that $B_\ep^{p'}$ is the restriction of the $p$-dimensional $\ep$-ball $B_\ep^{p}$ onto the first $p'$ coordinates.  This shows that the $\ep$-expansion of $S_1$ is contained within the $\ep$-expansion of $S_1^*$.  The same reasoning applies to $S_{-1}$.

This shows that the classifier $\hat y^*$ has robust risk at most as large as that of the original classifier $\hat y$. This finishes the proof.
\end{proof}

% !TEX root = ../adv.tex

\section{Pointwise calculation of a Bayes optimal classifier} \label{opt:threeclass:bayesopt}

Here we derive a Bayes optimal classifier,
i.e., a classifier that minimizes the standard non-robust risk $\stdrisk$
for the three-class setting \cref{eq:model:threeclass} of \cref{sec:opt:l2:threeclass}.
Recall that Bayes optimal classification is achieved by maximizing the posterior probability pointwise:
\begin{equation*}
  \bayesclass(x)
  \in \argmax_{c \in \clC} \Pr_{y|x}(y=c)
  = \argmax_{c \in \clC} \frac{p_{x|y=c}(x) \Pr(y=c)}{p(x)}
  = \argmax_{c \in \clC} \Pr(y=c) p_{x|y=c}(x)
  .
\end{equation*}
Hence it remains to identify the associated classification regions
\begin{align*}
  S_+ &\coloneqq \Big\{ x \in \bbR^p :
    \Pr(y=+1) p_{x|y=+1}(x) \geq \max\{\Pr(y= 0) p_{x|y= 0}(x), \Pr(y=-1) p_{x|y=-1}(x)\}
  \Big\}
  , \\
  S_- &\coloneqq \Big\{ x \in \bbR^p :
    \Pr(y=-1) p_{x|y=-1}(x) \geq \max\{\Pr(y= 0) p_{x|y= 0}(x), \Pr(y=+1) p_{x|y=+1}(x)\}
  \Big\}
  ,
\end{align*}
where the complementary region $(S_+ \cup S_-)^c$
will be classified as the remaining zero class.

\revone{Starting with $S_+$,}
note that
\begin{align*}
  &
  \Pr(y=+1) p_{x|y=+1}(x) \geq \Pr(y= 0) p_{x|y= 0}(x)
  \\&\qquad
  \iff
  \pi_+ \exp(-\|x-\revone{\lambda_+}\mu\|_2^2/2) \geq \pi_0 \exp(-\|x\|_2^2/2)
  \iff
  \|x-\revone{\lambda_+}\mu\|_2^2-\|x\|_2^2 \leq 2\ln(\pi_+/\pi_0)
  \\&\qquad \primetranspose
  \iff
  -2\revone{\lambda_+}x'\mu+ \revone{\lambda_+^2} \leq 2\ln(\pi_+/\pi_0)
  \iff
  x'\revone{\mu} \geq \revone{\lambda_+}/2-\ln(\pi_+/\pi_0)/\revone{|\lambda_+|}
  ,
\end{align*}
and similarly
\begin{align*}
  &
  \Pr(y=+1) p_{x|y=+1}(x) \geq \Pr(y=-1) p_{x|y=-1}(x)
  \\&
  \iff
  \pi_+ \exp(-\|x-\revone{\lambda_+}\mu\|_2^2/2) \geq \pi_- \exp(-\|x-\revone{\lambda_-}\mu\|_2^2/2)
  \iff
  \|x-\revone{\lambda_+}\mu\|_2^2-\|x-\revone{\lambda_-}\mu\|_2^2 \leq 2\ln(\pi_+/\pi_-)
  \\& \primetranspose
  \iff
  -\revone{2(\lambda_+-\lambda_-)}x'\mu +\revone{\lambda_+^2 - \lambda_-^2} \leq 2\ln(\pi_+/\pi_-)
  \iff
  x'\revone{\mu} \geq \revone{(\lambda_+ + \lambda_-)/2} - \ln(\pi_+/\pi_-)/\revone{|\lambda_+-\lambda_-|}
  .
\end{align*}
Therefore,
\revone{$\primetranspose S_+ = \{ x \in \bbR^p : x'\revone{\mu} \geq c_+ \}$
where}
\begin{equation*} \primetranspose
  c_+ \coloneqq \max\big\{
    \revone{\lambda_+}/2+\ln(\pi_0/\pi_+)/\revone{|\lambda_+|},
    \revone{(\lambda_+ + \lambda_-)/2 - \ln(\pi_+/\pi_-)/|\lambda_+-\lambda_-|}
  \big\}
  .
\end{equation*}
In the same way, we obtain
\revone{$\primetranspose S_- = \{ x \in \bbR^p : x'\revone{\mu} \leq c_- \}$
where}
\begin{equation*} \primetranspose
  c_- \coloneqq \min\big\{
    \revone{\lambda_-}/2-\ln(\pi_0/\pi_-)/\revone{|\lambda_-|},
    \revone{(\lambda_+ + \lambda_-)/2 - \ln(\pi_+/\pi_-)/|\lambda_+-\lambda_-|}
  \big\}
  .
\end{equation*}
Putting these together yields
\begin{equation*}
  \bayesclass(x)
  =
  \begin{cases}
    +1 & \text{if } x \in S_+
    , \\
    -1 & \text{if } x \in S_-
    , \\
    0 & \text{otherwise}
    ,
  \end{cases}
  =
  \begin{cases}
    +1 & \text{if } x^\top \revone{\mu} \geq c_+
    , \\
     0 & \text{if } c_- < x^\top \revone{\mu} < c_+
    , \\
    -1 & \text{if } x^\top \revone{\mu} \leq c_-
    ,
  \end{cases}
\end{equation*}
which is a linear interval classifier \cref{eq:threeclass:lindef},
as illustrated in \cref{fig:robopt:threeclass:linregions}.
At the boundaries between regions,
the posterior probabilities of the two corresponding classes are equal.
Indeed, assigning the boundary to either class
yields a Bayes optimal classifier, since the boundaries have zero probability under the marginal distribution of $x$.

\section{\revone{T}hree-class linear \revone{interval} classification for \texorpdfstring{$\ep \geq \revone{\min\{|\lambda_+|,|\lambda_-|\}}/2$}{ep >= min(|lambda+|,|lambda-|)/2}} \label{opt:threeclass:large:eps}

Here we show that three-class linear interval classification reduces \revone{down to several} two-class problems
when $\ep \geq \revone{\min\{|\lambda_+|,|\lambda_-|\}}/2$.
Specifically, we show the following \lcnamecref{prop:threeclass:large:eps}.

\medskip
\begin{proposition} \label{prop:threeclass:large:eps}
  Suppose $\ep \geq \revone{\min\{|\lambda_+|,|\lambda_-|\}}/2$.
  Then for any $w \neq 0$
  and $c_+ \geq c_-$,
  \begin{equation*}
    \robrisk\{\intclass(\cdot; w, c_+, c_-),\ep\}
    \geq
    \min\Big\{
      \robrisk(\hty_{-1,+1}^*,\ep),
      \robrisk(\hty_{ 0,+1}^*,\ep),
      \robrisk(\hty_{ 0,-1}^*,\ep)
    \Big\},
  \end{equation*}
  where
  \begin{align}
    \label{eq:threeclass:twoclass:mp}
    \hty_{-1,+1}^*(x)
    &\coloneqq
    \begin{cases}
      +1 , & \text{if }
        \revone{
        [x - \mu(\lambda_+ + \lambda_-)/2]^\top [\mu(\lambda_+ - \lambda_-)/2]
        (1 - 2\ep/|\lambda_+ - \lambda_-|)_+
        }
        -
        \ln(\pi_-/\pi_+)/2
        \geq 0
      , \\
      -1 , & \text{otherwise}
      ,
    \end{cases}
    \\
    \label{eq:threeclass:twoclass:p0}
    \hty_{ 0,+1}^*(x)
    &\coloneqq
    \begin{cases}
      +1 , & \text{if }
        \revone{
        [x - \mu\lambda_+/2]^\top [\mu\lambda_+/2]
        (1 - 2\ep/|\lambda_+|)_+
        }
        -
        \ln(\pi_0/\pi_+)/2
        \geq 0
      , \\
       0 , & \text{otherwise}
      ,
    \end{cases}
    \\
    \label{eq:threeclass:twoclass:m0}
    \hty_{ 0,-1}^*(x)
    &\coloneqq
    \begin{cases}
      -1 , & \text{if }
        \revone{
        [x - \mu\lambda_-/2]^\top [\mu\lambda_-/2]
        (1 - 2\ep/|\lambda_-|)_+
        }
        -
        \ln(\pi_0/\pi_-)/2
        \geq 0
      , \\
       0 , & \text{otherwise}
      ,
    \end{cases}
  \end{align}
  are the optimal robust two-class classifiers
  obtained by applying \cref{eq:robopt:twoclass} from \cref{orc}
  to each pair of classes.
\end{proposition}

\begin{proof}[Proof of \cref{prop:threeclass:large:eps}]
Let $\ep \geq \revone{\min\{|\lambda_+|,|\lambda_-|\}}/2$ be arbitrary.
Considering the negative and positive classes alone
(i.e., ignoring the zero class)
and applying \cref{eq:robopt:twoclass}
yields the optimal robust two-class classifier \cref{eq:threeclass:twoclass:mp}.
Similarly, considering the zero and positive classes alone
and applying \cref{eq:robopt:twoclass}
to $x - \revone{\mu\lambda_+/2}$ (for which the two classes have means $\pm \revone{\mu\lambda_+/2}$)
yields the optimal robust two-class classifier \cref{eq:threeclass:twoclass:p0}.
\revone{Likewise with $x - \mu\lambda_-/2$
for the optimal robust two-class classifier \cref{eq:threeclass:twoclass:m0}.}

Now, let $w \neq 0$ and $c_+ \geq c_-$ be arbitrary.
Recall first that
$\intclass(x; w, c_+, c_-) = \intclass(x; \tlw, \tlc_+, \tlc_-)$,
where
\begin{align*}
  \tlw &= w/\|w\|_2
  , &
  \tlc_+ &= c_+/\|w\|_2
  , &
  \tlc_- &= c_-/\|w\|_2
  ,
\end{align*}
and so
$\robrisk\{\intclass(\cdot; w, c_+, c_-),\ep\}
= \robrisk\{\intclass(\cdot; \tlw, \tlc_+, \tlc_-),\ep\}$.
Next, recall from \cref{proof:thm:opt:threeclass:threshopt:weights} that
for any $\tlc_+ \geq \tlc_-$,
the robust risk is optimized subject to $\|\tlw\|_2 = 1$
by $\tlw = \revone{\mu}$
so we have
\begin{align*}
  \robrisk\{\intclass(\cdot; \tlw, \tlc_+, \tlc_-),\ep\}
  &
  \geq
  \robrisk\{\intclass(\cdot; \revone{\mu}, \tlc_+, \tlc_-),\ep\}
  \\&
  = \pi_- \Pr_-(\tlx > \tlc_- - \ep)
  + \pi_+ \Pr_+(\tlx < \tlc_+ + \ep)
  + \pi_0 \Pr_0(\tlx \leq \tlc_- + \ep \text{ or } \tlx \geq \tlc_+ - \ep)
\end{align*}
where the shorthand notations
$\Pr_- \coloneqq \Pr_{\tlx \sim \clN(\revone{\lambda_-},1)}$,
$\Pr_+ \coloneqq \Pr_{\tlx \sim \clN(\revone{\lambda_+},1)}$,
and $\Pr_0 \coloneqq \Pr_{\tlx \sim \clN(0,1)}$
are taken from \cref{proof:thm:opt:threeclass:threshopt:thresh}.
It now remains to bound \revone{the robust risk} $\robrisk\{\intclass(\cdot; \revone{\mu}, \tlc_+, \tlc_-),\ep\}$.

\revone{To begin, consider the case where} $\tlc_- \leq \tlc_+ \leq \tlc_- + 2\ep$.
In this case,
note that
\revone{$\Pr_0(\tlx \leq \tlc_- + \ep \text{ or } \tlx \geq \tlc_+ - \ep) = 1$,
so} the robust risk \revone{can be rewritten and bounded as}
\revone{%
\begin{align*}
  \robrisk\{\intclass(\cdot; \mu, \tlc_+, \tlc_-),\ep\}
  &
  = \pi_- \Pr_-(\tlx > \tlc_- - \ep)
  + \pi_+ \Pr_+(\tlx < \tlc_+ + \ep)
  + \pi_0
  \\&
  \geq \pi_- \Pr_-(\tlx > \tlc_+ - \ep)
  + \pi_+ \Pr_+(\tlx < \tlc_+ + \ep)
  + \pi_0
  .
\end{align*}
}%
The final expression is the robust risk of a linear two-class classifier
that assigns points to only the positive and negative classes
with threshold $\tlc_+$,
so is no better than that of $\hty_{-1,+1}^*$.
As a result, we have that
\begin{align*}
  \robrisk\{\intclass(\cdot; \tlw, \tlc_+, \tlc_-),\ep\}
  &\geq
  \pi_- \Pr_-(\tlx > \tlc_+ - \ep)
  + \pi_+ \Pr_+(\tlx < \tlc_+ + \ep)
  + \pi_0
  \geq
  \robrisk(\hty_{-1,+1}^*,\ep)
  ,
\end{align*}
which completes the proof in this case.

\revone{Next, consider the case where} $\tlc_+ \geq \tlc_- + 2\ep$\revone{,
in which case we can}
rewrite the robust risk as
\begin{equation*}
  \revone{\robrisk\{\intclass(\cdot; \mu, \tlc_+, \tlc_-),\ep\}}
  =
  1 -
  \Big\{
    \pi_- \Pr_-(\tlx \leq \tlc_- - \ep)
  + \pi_+ \Pr_+(\tlx \geq \tlc_+ + \ep)
  + \pi_0 \Pr_0(\tlc_- + \ep < \tlx < \tlc_+ - \ep)
  \Big\}
  .
\end{equation*}
\revone{
Now note that
$\ep \geq |\lambda_+|/2$ or $\ep \geq |\lambda_-|/2$
since $\ep \geq \min\{|\lambda_+|,|\lambda_-|\}/2$.
Suppose first that $\ep \geq |\lambda_+|/2$,
and note that
if $\pi_+ \leq \pi_0$ then
\begin{align*}
  &
  \pi_+ \Pr_+(\tlx \geq \tlc_+ + \ep)
  + \pi_0 \Pr_0(\tlc_- + \ep < \tlx < \tlc_+ - \ep)
  \leq
  \pi_0 \Big\{
    \Pr_+(\tlx \geq \tlc_+ + \ep)
    + \Pr_0(\tlc_- + \ep < \tlx < \tlc_+ - \ep)
  \Big\}
  \\&\qquad\qquad
  =
  \pi_0 \Big\{
    \Pr_0(\tlx \geq \tlc_+ - (\lambda_+ - \ep))
    + \Pr_0(\tlc_- + \ep < \tlx < \tlc_+ - \ep)
  \Big\}
  \leq
  \pi_0 \Pr_0(\tlx > \tlc_- + \ep)
  ,
\end{align*}
where the final inequality
used the fact that $\tlc_+ - (\lambda_+ - \ep) \geq \tlc_+ - \ep$
when $\ep \geq |\lambda_+|/2$.
Consequently,
\begin{equation*}
  \robrisk\{\intclass(\cdot; \mu, \tlc_+, \tlc_-),\ep\}
  \geq
  \pi_+
  + \pi_- \Pr_-(\tlx > \tlc_- - \ep)
  + \pi_0 \Pr_0(\tlx \leq \tlc_- + \ep)
  \geq
  \robrisk(\hty_{ 0,-1}^*,\ep)
  ,
\end{equation*}
where the final inequality holds because
the middle term is
the robust risk of a linear two-class classifier
with threshold $\tlc_-$
that assigns points to only the negative and zero classes,
so is no better than that of $\hty_{0,-1}^*$.
Likewise,
if $\pi_+ > \pi_0$
then
$\pi_+ \Pr_+(\tlx \geq \tlc_+ + \ep) + \pi_0 \Pr_0(\tlc_- + \ep < \tlx < \tlc_+ - \ep)
\leq \pi_+ \Pr_+(\tlx > \tlc_- + \ep)$
and
\begin{equation*}
  \robrisk\{\intclass(\cdot; \mu, \tlc_+, \tlc_-),\ep\}
  \geq
  \pi_0
  + \pi_- \Pr_-(\tlx > \tlc_- - \ep)
  + \pi_+ \Pr_+(\tlx \leq \tlc_- + \ep)
  \geq
  \robrisk(\hty_{-1,+1}^*,\ep)
  .
\end{equation*}
As a result,
whether $\pi_+ \leq \pi_0$ or $\pi_+ > \pi_0$,
\begin{equation*}
  \robrisk\{\intclass(\cdot; \tlw, \tlc_+, \tlc_-),\ep\}
  \geq
  \min\Big\{
    \robrisk(\hty_{ 0,-1}^*,\ep),
    \robrisk(\hty_{-1,+1}^*,\ep)
  \Big\}
  ,
\end{equation*}
completing the proof when $\ep \geq |\lambda_+|/2$.

Repeating the analogous argument for $\ep \geq |\lambda_-|/2$
yields that in that case
\begin{equation*}
  \robrisk\{\intclass(\cdot; \tlw, \tlc_+, \tlc_-),\ep\}
  \geq
  \min\Big\{
    \robrisk(\hty_{ 0,+1}^*,\ep),
    \robrisk(\hty_{-1,+1}^*,\ep)
  \Big\}
  ,
\end{equation*}
completing the proof.
}%
\end{proof}
% !TEX root = ../adv.tex

\revone{
\section{Beyond Gaussians that lie on a line}
\label{sec:threeclass:beyond:line}

For the one-dimensional setting $p = 1$,
the three Gaussian are always as in \cref{eq:model:threeclass},
i.e., their means lie along a line.
However, when $p > 1$,
the means can be in more general locations
and one naturally wonders what the optimal robust classifier
would be in such settings.

When the means do not lie on a line,
the Bayes optimal classifier
is no longer linear in general
and so the optimal robust classifier
is likely no longer linear as well.
Moreover,
as shown above,
the optimal robust classifier
does not coincide with the Bayes optimal classifier
in general,
making it nontrivial to even
conjecture what the optimal robust classifier is in such cases.
However, we do expect
them to coincide when
the location of the means is symmetric,
as described by the following \lcnamecref{conj:gaussian:triangle}.

\medskip
\begin{conjecture}[Optimal robust classifier for means arranged in a triangle]
\label{conj:gaussian:triangle}
Consider three balanced Gaussian classes $\clC = \{0,1,2\}$
with means at the corners of an equilateral triangle:
\begin{align*}
  x|y &\sim
  \clN\Big(
    \mu \cdot [\cos(y \cdot 2\pi/3) e_1 + \sin(y \cdot 2\pi/3) e_2]
    ,
    I_p
  \Big)
  , &
  y &=
  \begin{cases}
    0 & \text{with probability } 1/3
    , \\
    1 & \text{with probability } 1/3
    , \\
    2 & \text{with probability } 1/3
    ,
  \end{cases}
\end{align*}
where $\mu \in \bbR_{>0}$ defines the size of the triangle,
and $e_1, e_2 \in \{0,1\}^p$ are the canonical basis vectors.

We conjecture that the optimal robust classifier
is the same as the Bayes optimal classifier,
i.e.,
\begin{equation*}
  \hty^*(x)
  \coloneqq
  \argmax_y
  \|x - \mu \cdot [\cos(y \cdot 2\pi/3) e_1 + \sin(y \cdot 2\pi/3) e_2]\|_2
  ,
\end{equation*}
as shown in the following diagram:
\begin{center}
\begin{tikzpicture}
  \draw[Gray,->,thick] (-2.25,0) -- (2.5,0) node [Black,anchor=west] {$e_1$};
  \draw[Gray,->,thick] (0,{-2.5*0.866}) -- (0,{2.5*0.866}) node [Black,anchor=south] {$e_2$};

  \filldraw[OliveGreen,opacity=0.2] (0,0) -- ({2.5*0.5},{2.5*0.866})
    -- (2.5,{2.5*0.866}) -- (2.5,{-2.5*0.866})
    -- ({2.5*0.5},{-2.5*0.866}) -- cycle;
  \filldraw[OliveGreen] (1,0) circle (2pt) node [anchor=south west,outer sep=0pt] {$y=0$};
  
  \filldraw[Red,opacity=0.2] (0,0) -- ({2.5*0.5},{2.5*0.866})
    -- (-2.25,{2.5*0.866}) -- (-2.25,0) -- cycle;
  \filldraw[Red] (-0.5,+0.866) circle (2pt) node [anchor=south,outer sep=2pt] {$y=1$};
  
  \filldraw[Blue,opacity=0.2] (0,0) -- ({2.5*0.5},{-2.5*0.866})
    -- (-2.25,{-2.5*0.866}) -- (-2.25,0) -- cycle;
  \filldraw[Blue] (-0.5,-0.866) circle (2pt) node [anchor=north,outer sep=2pt] {$y=2$};

  \draw[Black,<->] (0,0) -- (0.5,0.866) node [midway, sloped, below] {$\mu$};
  \draw[Black,dashed] (0,0) circle (1);
\end{tikzpicture}
\end{center}
\end{conjecture}

As before,
we have taken the within-class variance to be unity
without loss of generality.
Furthermore,
we have centered the triangle at the origin
without loss of generality.

This \lcnamecref{conj:gaussian:triangle} is natural,
but turns out to be nontrivial to rigorously establish.
In particular, the Gaussian concentration of measure approach
used to prove optimality in \cref{orc,thm:opt:threeclass:ignore:separate}
does not directly apply here
since the conjectured optimal classifier is not composed of half-spaces.
It would appear that
the method used in \cite{dan2020sharp}
may be used here
since the Bayes classifier
for perturbed means $\mu' = \mu - \ep/\sin(\pi/3)$
has standard risk
with respect to the perturbed means
matching the robust risk
with respect to the unperturbed means.
However, the $\ep/\sin(\pi/3)$ size of the perturbation
is too large and cannot be used.

Going beyond three-classes to settings
with even more classes,
we conjecture that the same holds
in corresponding symmetric settings,
as described by the following \lcnamecref{conj:gaussian:moreclasses}

\medskip
\begin{conjecture}[Optimal robust classifier for $n$ means arranged in a regular polygon]
\label{conj:gaussian:moreclasses}
Consider $n$ balanced Gaussian classes $\clC = \{0,1,\dots,n-1\}$
with means at the corners of a regular $n$-sided polygon:
\begin{align*}
  x|y &\sim
  \clN\Big(
    \mu \cdot [\cos(y \cdot 2\pi/n) e_1 + \sin(y \cdot 2\pi/n) e_2]
    ,
    I_p
  \Big)
  , &
  y &=
  \begin{cases}
    0 & \text{with probability } 1/n
    , \\
    \vdots & \vdots \\
    n-1 & \text{with probability } 1/n
    ,
  \end{cases}
\end{align*}
where $\mu \in \bbR_{>0}$ defines the size of the polygon,
and $e_1, e_2 \in \{0,1\}^p$ are the canonical basis vectors.

We conjecture that the optimal robust classifier
is the same as the Bayes optimal classifier,
i.e.,
\begin{equation*}
  \hty^*(x)
  \coloneqq
  \argmax_y
  \|x - \mu \cdot [\cos(y \cdot 2\pi/n) e_1 + \sin(y \cdot 2\pi/n) e_2]\|_2
  .
\end{equation*}
\end{conjecture}

As with \cref{conj:gaussian:triangle},
this \lcnamecref{conj:gaussian:moreclasses}
is natural but highly nontrivial to rigorously establish.
Once again,
the Gaussian concentration of measure approach
used to prove optimality in \cref{orc,thm:opt:threeclass:ignore:separate}
does not directly apply
since the conjectured optimal classifier is not composed of half-spaces.
Likewise,
the method used in \cite{dan2020sharp}
does not seem to apply
since the Bayes classifier
for perturbed means $\mu' = \mu - \ep/\sin(\pi/n)$
has standard risk
with respect to the perturbed means
matching the robust risk
with respect to the unperturbed means.
However, the $\ep/\sin(\pi/n)$ size of the perturbation
is again too large and cannot be used.

Developing new theoretical approaches
that can prove \cref{conj:gaussian:triangle,conj:gaussian:moreclasses}
is an exciting direction for future work.
}

% !TEX root = ../adv.tex

\section{Proofs for optimal \texorpdfstring{$\ell_\infty$}{linfty} robust classifiers}
\label{app:opt:linf:proofs}

\subsection{Proof of \cref{linf}}
\label{app:linf:proof}
This follows because $\ell_\infty$ norm is upper bounded by the $\ell_2$ norm. Thus for any fixed $\ep$, the $\ell_\infty$ robust risk is upper bounded by the $\ell_2$ robust risk:

\begin{align*}
R(\hat y,\ep,\|\cdot\|_\infty) &\ge R(\hat y,\ep,\|\cdot\|_2)\\
R^*(\ep,\|\cdot\|_\infty)&\ge R^*(\ep,\|\cdot\|_2)
.
\end{align*}
From \cref{orc}, we know the optimal $\ell_2$ robust classifiers, i.e., the ones minimizing the upper bound, are based on $x_j\cdot [\mu_j-\ep]$.
Now, it follows that for the decision sets $S_i$ of these classifiers (axis aligned half-planes), $S_i+B_{2,\ep}=S_i+B_{\infty,\ep}$, where $B_{q,\ep}$ denotes the $\ep$-ball in the $\ell_q$ norm. Thus, the $\ell_\infty$ robust risk is \emph{equal} to the $\ell_2$ risk for this specific classifier. This implies that it also minimizes the $\ell_\infty$ risk, and that the two risks are the same:
\begin{align*}
R^*(\ep,\|\cdot\|_\infty)&=R^*(\ep,\|\cdot\|_2).
\end{align*}
This finishes the proof.
\qed

\subsection{Proof of \cref{linf:threeclass}}
\label{proof:linf:threeclass}

As before,
$\robrisk(\hty,\ep,\|\cdot\|_\infty) \geq \robrisk(\hty,\ep,\|\cdot\|_2)$
for any classifier $\hty$ and radius $\ep$.
Now, by \cref{thm:opt:threeclass:threshopt}
the weights $w^* \coloneqq \mu/\|\mu\|_2$ optimize
$\robrisk\{\intclass(x;w^*,c_+^*,c_-^*),\ep,\|\cdot\|_2\}$
where the formulae for the two cases of $c_+^*$ and $c_-^*$,
as well as the cutoff $\alpha^*$,
are simplified by noting that $\|\mu\|_2 = \mu_j$.
Moreover,
\begin{equation*}
  \intclass^*(x)
  \coloneqq \intclass(x;w^*,c_+^*,c_-^*)
  = \intclass(x^\top w^*;1,c_+^*,c_-^*)
  = \intclass(x_j;1,c_+^*,c_-^*)
  ,
\end{equation*}
since $w^*$ has one non-zero coordinate $w_j^* = 1$ ($w^*$ is 1-sparse).
Finally,
\begin{equation*}
  \robrisk(\intclass^*,\ep,\|\cdot\|_\infty)
  = \robrisk(\intclass^*,\ep,\|\cdot\|_2)
\end{equation*}
since $S_y^c(\intclass^*) + B_{2,\ep} = S_y^c(\intclass^*) + B_{\infty,\ep}$
for $y \in \{-1,0,1\}$;
the misclassification sets are coordinate-aligned.
Thus, it follows that $\intclass^*$ also optimizes $\robrisk(\intclass^*,\ep,\|\cdot\|_\infty)$.
\qed

\subsection{Proof of \cref{linflin}}
\label{proof:linflin}

Recall our general formula:
\begin{align*}
R(\hat y,\ep) &=
\pi\cdot P_{x|y=1}(S_{-1}+B_\ep)
+(1-\pi)\cdot P_{x|y=-1}(S_{1}+B_\ep).
\end{align*}
Take a linear classifier $\hat y^*(x)=\sign(x^\top w-c)$ for some some $w,c$, and $P_{x|y}=\N(y\mu,I_p)$. Then $S_1$ is the set of datapoints such that $x^\top w-c\ge 0$. So, $S_{1}+B_\ep$ is the set of datapoints such that $x^\top w-c\ge -\ep \|w\|_1$. Thus, restricting without loss of generality to $w$ such that $\|w\|_2=1$,
\begin{align*}
R(w,c;\ep) &=
\pi\cdot P_{\N(\mu,I)}(x^\top w-c\le \ep \|w\|_1)
+(1-\pi)\cdot P_{\N(-\mu,I)}(x^\top w-c\ge -\ep \|w\|_1)\\
=&\pi \cdot \Phi\left(c+ \ep \|w\|_1-\mu^\top w\right)
+(1-\pi)\cdot \Phi\left(-c+ \ep \|w\|_1-\mu^\top w\right).
\end{align*}
The minimizer is
$$c^* = \frac{q}{2\cdot (\mu^\top w-\ep \|w\|_1)}$$
where recall that $q=\log[(1-\pi)/\pi]$. This applies when $\mu^\top w-\ep \|w\|_1> 0$. If that does not happen, then the weight $w$ is not aligned properly with the problem, in the sense that it reduces the ``effective'' effect size to a negative value. Thus, we do not need to consider those cases.

Another way to put this is that for a weight $w$ with unit norm $\|w\|_2=1$, a linear classifier reduces the effect size from $\mu^\top w$ (which we can assume to be positive, without loss of generality, by flipping the sign if needed), to $\mu^\top w-\ep \|w\|_1$. So we can solve the problem:
\begin{align*}
\sup_w\,\, & \mu^\top w-\ep \|w\|_1\\
s.t.\,\, & \|w\|_2=1.
\end{align*}

First, we can WLOG restrict to weights $w$ which have the same sign as $\mu$, because for any $w$, flipping a sign of a coordinate such that it has the same sign as $\mu_i$ increases (or does not decrease, in the extreme case where $\mu_i$ or $w_i$ are zero), the objective. Moreover, we can also solve first the problem where all coordinates of $\mu$ are non-negative. (Then we can flip the signs of $w$ according to the sign of $\mu$ to recover the solution).

These simplifications lead to the problem with $\mu_i\ge0$
\begin{align*}
\sup_w\,\, & \sum_i[\mu_i-\ep]\cdot w_i\\
s.t.\,\, & \|w\|_2=1, \,\, w_i\ge 0.
\end{align*}
If, for some $i$, $\mu_i-\ep\le0$, then we need to set $w_i=0$. For the remaining coordinates, we can upper bound the objective value by the Cauchy-Schwarz inequality: $v^\top w\le \|v\|_2 \cdot \|w\|_2 = \|v\|_2$; with $v = \mu - \ep \cdot 1$ restricted to the positive coordinates. Moreover, to satisfy the unit norm constraint, we need to set $w^* = v/\|v\|_2$.

More generally, with negative coordinates, the solution will depend on the \emph{soft thresholding} operator $v = \eta(\mu,\lambda)$ well known in signal processing and statistics.

Specifically, we will have $v = \eta(\mu,\ep)$, and $w = v/\|v\|_2$. Then we also get 
\begin{align*}
c^* &= \frac{q}{2\cdot (\mu^\top w-\ep \|w\|_1)} = \frac{q}{2\cdot \|\eta(\mu,\ep)\|}
\end{align*}
This shows that the optimal classifier is $\sign\{\eta(\mu,\ep)^\top x - q/2\}$, as desired.
\qed

\subsection{Proof of \cref{linflin:threeclass}}
\label{proof:linflin:threeclass}

If $\|w\|_2 = 1$,
then the \revone{linear} interval classifier is
$\intclass(x;w,c_+,c_-) = \intclass(x^\top w;1,c_+,c_-)$,
and the problem effectively reduces
to a one-dimensional problem with new variable $\tlx_w \coloneqq x^\top w \in \bbR$,
which is the mixture of Gaussians $\tlx_w | y \sim \clN(y w^\top \mu,1)$,
where $\ep \|w\|_1$ is the corresponding one-dimensional perturbation.

Hence, the robust risk to minimize
with respect to weights $\|w\|_2 = 1$ and thresholds $c_+ \geq c_-$
is
\begin{align*}
  \tlR(w,c_+,c_-)
  &\coloneqq \robrisk\big\{ \intclass(\tlx_w;1,c_+,c_-), \ep\|w\|_1 \big\}
  \\
  &=
    \pi_- \Pr_{\tlx_w | y=-1}(\tlx_w > c_- - \ep\|w\|_1)
  + \pi_+ \Pr_{\tlx_w | y= 1}(\tlx_w < c_+ + \ep\|w\|_1)
  \\
  &\qquad
  + \pi_0 \Pr_{\tlx_w | y= 0}(
      \tlx_w \leq c_- + \ep\|w\|_1 \text{ or } \tlx_w \geq c_+ - \ep\|w\|_1
    )
  \\
  &=
    \pi_- \Pr_{\tlx_w | y=-1}(\tlx_w > c_- - \ep\|w\|_1)
  + \pi_+ \Pr_{\tlx_w | y= 1}(\tlx_w < c_+ + \ep\|w\|_1)
  \\
  &\qquad
  + \pi_0
    \min\Big\{1,
      \Pr_{\tlx_w | y= 0}(\tlx_w \leq c_- + \ep\|w\|_1)
      + \Pr_{\tlx_w | y= 0}(\tlx_w \geq c_+ - \ep\|w\|_1)
    \Big\}
  \\
  &=
    \pi_- \brPhi(c_- - \ep\|w\|_1 + w^\top \mu)
  + \pi_+   \Phi(c_+ + \ep\|w\|_1 - w^\top \mu)
  \\
  &\qquad
  + \pi_0
    \min\big\{1, \Phi(c_- + \ep\|w\|_1) + \brPhi(c_+ - \ep\|w\|_1) \big\}
  \\
  &= \min\{ \tlR_1(w,c_+,c_-), \tlR_2(w,c_+,c_-) \}
  ,
\end{align*}
where $\Phi$ is the normal CDF, its complement is $\brPhi \coloneqq 1-\Phi$,
and
\begin{align*}
  \tlR_1(w,c_+,c_-)
  &\coloneqq
    \pi_- \brPhi(c_- - \ep\|w\|_1 + w^\top \mu)
  + \pi_+   \Phi(c_+ + \ep\|w\|_1 - w^\top \mu)
  + \pi_0
  , \\
  \tlR_2(w,c_+,c_-)
  &\coloneqq
    \pi_- \brPhi(c_- - \ep\|w\|_1 + w^\top \mu)
  + \pi_+   \Phi(c_+ + \ep\|w\|_1 - w^\top \mu)
  \\
  &\qquad
  + \pi_0
    \big\{ \Phi(c_- + \ep\|w\|_1) + \brPhi(c_+ - \ep\|w\|_1) \big\}
  .
\end{align*}
Now, $\tlR_1$ amounts to the two-class setting in \cref{linflin}
and is likewise minimized by
\begin{align*}
  \tlw_1^* &= \frac{\eta_\ep(\mu)}{\|\eta_\ep(\mu)\|_2}
  , &
  c_+ &= c_-
  = \tlc^*
  = \frac{\ln(\pi_-/\pi_+)}{2\|\eta_\ep(\mu)\|_2}
  ,
\end{align*}
since $\tlR_1$
is a decreasing function (for $c_+ \geq c_-$ fixed)
in $w^\top \mu - \ep \|w\|_1$,
which is itself maximized by $\eta_\ep(\mu)$.
Assuming $\ep < \|\mu\|_\infty/2$
prevents the degenerate case where $\eta_\ep(\mu) = 0$,
and with $w=\tlw_1^*$ fixed,
minimization with respect to $c_+ \geq c_-$
is as in the proof of \cref{thm:opt:threeclass:threshopt};
note that $\eta_\ep(\mu)^\top \mu - \ep \|\eta_\ep(\mu)\|_1 = \|\eta_\ep(\mu)\|_2^2$.
Thus,
\begin{equation*}
  \inf_{\substack{\|w\|_2 = 1\\\;c_+ \geq c_-}} \tlR_1(w,c_+,c_-)
  = \tlR_1(\tlw_1^*,\tlc^*,\tlc^*)
  .
\end{equation*}
Next, note that $\tlR_2(w,c_+,c_-) \geq \tlR_1(w,c_+,c_-)$
when $c_- + \ep\|w\|_1 \geq c_+ - \ep\|w\|_1$
so we need only minimize $\tlR_2(w,c_+,c_-)$
over $c_- + \ep\|w\|_1 \leq c_+ - \ep\|w\|_1$,
which is equivalently expressed
via change of variables as
\begin{equation*}
  \inf_{\substack{\|w\|_2 = 1\\\;c_+ \geq c_- + 2\ep\|w\|_1}} \tlR_2(w,c_+,c_-)
  = \inf_{\substack{\|w\|_2 = 1\\\;\tau_+ \geq \tau_-}}
    \tlR_2(w,\tau_+ + \ep\|w\|_1,\tau_- - \ep\|w\|_1)
  .
\end{equation*}
For any $\tau_+ \geq \tau_-$,
\begin{align*}
  &\tlR_2(w,\tau_+ + \ep\|w\|_1,\tau_- - \ep\|w\|_1)
  \\
  &\qquad=
    \pi_- \brPhi(\tau_- - 2\ep\|w\|_1 + w^\top \mu)
  + \pi_+   \Phi(\tau_+ + 2\ep\|w\|_1 - w^\top \mu)
  + \pi_0
    \big\{ \Phi(\tau_-) + \brPhi(\tau_+) \big\}
\end{align*}
is a decreasing function of $w^\top \mu - 2\ep\|w\|_1$,
which is maximized by
$\tlw_2^* \coloneqq \eta_{2\ep}(\mu)/\|\eta_{2\ep}(\mu)\|_2$;
again the case $\eta_{2\ep}(\mu) = 0$ is prevented
by $\ep < \|\mu\|_\infty/2$.
Fixing $w=\tlw_2^*$,
minimization with respect to $c_+ \geq c_- + 2\ep\|w\|_1$
is as in the proof of \cref{thm:opt:threeclass:threshopt}.
Namely,
\begin{align*}
  \tlc_+^* &\coloneqq
  + \frac{(\tlw_2^*)^\top \mu}{2}
  + \frac{\ln(\pi_0/\pi_+)}{\|\eta_{2\ep}(\mu)\|_2}
  , &
  \tlc_-^* &\coloneqq
  - \frac{(\tlw_2^*)^\top \mu}{2}
  - \frac{\ln(\pi_0/\pi_-)}{\|\eta_{2\ep}(\mu)\|_2}
  ,
\end{align*}
are optimal if $\tlc_+^* \geq \tlc_-^* + 2\ep\|\tlw_2^*\|_1$,
and setting $c_+ = c_- + 2\ep\|\tlw_2^*\|_1$ is optimal otherwise.
Thus
\begin{equation*}
  \inf_{\substack{\|w\|_2 = 1\\\;c_+ \geq c_- + 2\ep\|w\|_1}} \tlR_2(w,c_+,c_-)
  =
  \begin{cases}
    \tlR_2(\tlw_2^*,\tlc_+^*,\tlc_-^*)
    , & \text{if } \tlc_+^* \geq \tlc_-^* + 2\ep\|\tlw_2^*\|_1
    , \\
    \inf_{c \in \bbR} \tlR_2(\tlw_2^*,c + 2\ep\|\tlw_2^*\|_1,c)
    , & \text{otherwise}
    .
  \end{cases}
  .
\end{equation*}
Putting it all together, we conclude that
\begin{align*}
  &
  \inf_{\substack{\|w\|_2 = 1\\\;c_+ \geq c_-}} \tlR(w,c_+,c_-)
  = \min\Bigg\{
    \inf_{\substack{\|w\|_2 = 1\\\;c_+ \geq c_-}} \tlR_1(w,c_+,c_-),
    \inf_{\substack{\|w\|_2 = 1\\\;c_+ \geq c_- + 2\ep\|w\|_1}} \tlR_2(w,c_+,c_-)
  \Bigg\}
  \\
  &\qquad
  =
  \begin{cases}
    \min\{\tlR_1(\tlw_1^*,\tlc^*,\tlc^*), \tlR_2(\tlw_2^*,\tlc_+^*,\tlc_-^*)\}
    , & \text{if } \tlc_+^* \geq \tlc_-^* + 2\ep\|\tlw_2^*\|_1
    , \\
    \min\{
      \tlR_1(\tlw_1^*,\tlc^*,\tlc^*),
      \inf_{c \in \bbR} \tlR_2(\tlw_2^*,c + 2\ep\|\tlw_2^*\|_1,c)
    \}
    , & \text{otherwise}
    ,
  \end{cases}
  \\
  &\qquad
  =
  \begin{cases}
    \min\{\tlR_1(\tlw_1^*,\tlc^*,\tlc^*), \tlR_2(\tlw_2^*,\tlc_+^*,\tlc_-^*)\}
    , & \text{if } \tlc_+^* \geq \tlc_-^* + 2\ep\|\tlw_2^*\|_1
    , \\
    \tlR_1(\tlw_1^*,\tlc^*,\tlc^*)
    , & \text{otherwise}
    ,
  \end{cases}
\end{align*}
where the final equality follows from the observation that
\begin{equation*}
  \tlR_2(\tlw_2^*,c + 2\ep\|\tlw_2^*\|_1,c)
  = \tlR_1(\tlw_2^*,c + 2\ep\|\tlw_2^*\|_1,c)
  \geq \tlR_1(\tlw_1^*,\tlc^*,\tlc^*)
  .
\end{equation*}
Hence, we have optimal \revone{linear} interval classifiers given by two cases:
i) $\tlw_1^*$ and $\tlc^*$,
or ii) $\tlw_2^*$ and $\tlc_\pm^*$.
Note that (ii) remains a valid/feasible choice
so long as $\tlc_+^* \geq \tlc_-^*$
even if $\tlc_+^* < \tlc_-^* + 2\ep\|\tlw_2^*\|_1$;
it may just be sub-optimal in that case.
Finally, noting that weights and thresholds can be scaled,
i.e.,
\begin{align*}
  \intclass(x;\tlw_1^*,\tlc^*,\tlc^*)
  &= \intclass(x;
    \tlw_1^*\|\eta_\ep(\mu)\|_2,
    \tlc^*\|\eta_\ep(\mu)\|_2,
    \tlc^*\|\eta_\ep(\mu)\|_2
  )
  , \\
  \intclass(x;\tlw_2^*,\tlc_+^*,\tlc_-^*)
  &= \intclass(x;
    \tlw_2^*\|\eta_{2\ep}(\mu)\|_2,
    \tlc_+^*\|\eta_{2\ep}(\mu)\|_2,
    \tlc_-^*\|\eta_{2\ep}(\mu)\|_2
  )
  ,
\end{align*}
and simplifying completes the proof.
\qed

% !TEX root = ../adv.tex

\section{Proofs for finite sample analysis}

\subsection{Proof of \cref{prop:empirical:risk:opt:linloss}}
The first portion of this proof follows that of \revone{\cite[Lemma 10]{chen2020more}, in that we first consider the following problem:}
\begin{align*}
    w_n^* &\in \argmin_{\norm{w}_{\revone{2}} \leq 1} \sum_{i=1}^n \max_{\norm{\delta_i}_{\infty} \leq \ep} -y_i\left(\langle x_i + \delta_i, w\rangle\right) = \argmax_{\norm{w}_{\revone{2}}\leq 1} \sum_{i=1}^n \min_{\norm{\delta_i}_\infty \leq \ep} y_i\left(\langle x_i + \delta_i, w\rangle\right).
\end{align*}
\revone{In} the inner minimization problem, \revone{it holds that $\min_{\norm{\delta_i}_\infty \leq \ep} y_i\langle x_i + \delta_i, w\rangle = y_i\langle x_i, w\rangle - \ep\norm{w}_1$} by the definition of the dual norm.  Therefore the original problem takes the form
\begin{align*}
    w_n^* &\in \argmax_{\norm{w}_{\revone{2}}\leq 1} \sum_{i=1}^n y_i\langle x_i, w\rangle - \ep\norm{w}_1
    = \argmax_{\norm{w}_{\revone{2}}\leq 1} n\left(\langle u, w\rangle - \ep\norm{w}_1 \right)
\end{align*}
where we have defined $u \coloneqq \frac{1}{n}\sum_{i=1}^n y_ix_i$.  Now if we let $w(j)$ and $u(j)$ denote the $j^{\text{th}}$ components of the vectors $w$ and $u$ respectively, we have
\begin{align*}
    w_n^*\in \argmax_{\norm{w}_{\revone{2}}\leq 1} \sum_{j=1}^d u(j)w(j) - \ep|w(j)|
\end{align*}
Notice that if $u(j) \neq 0$, then $\sign(u(j)) = \sign(w_n^*(j))$ as flipping the signs will only make the $j^{\text{th}}$ term smaller.  On the other hand, if $u(j) = 0$, then the maximum is achieved when $w_n^*(j) = 0$.  Thus $\sign(u) = \sign(w_n^*)$.  Now in a similar way to what was done in the proof of \cref{linflin}, let us assume WLOG that $u \succeq 0$, which implies that $w_n^* \succeq 0$ as well.  Then we wish to solve
\revone{%
\begin{align*}
    \argmax_w \quad\hspace{1pt} \langle u - \ep \mathbbm{1}, w\rangle \quad \text{subject to} \:\: \norm{w}_2 \leq 1 , w\succeq 0.
\end{align*}
It} follows that $\revone{w_n^\star} = \eta(u, \ep) / \norm{\eta(u, \ep)}$ where $\eta$ is the soft-thresholding operator.  
\qed

\subsection{Proof of \cref{1dfs2}}

The formula of the robust risk for a classifier $\hat y$ is
$$R(\hat y,\ep)
=
P(y=1)P_{x|y=1}(S_{-1}+B_\ep)
+P(y=-1)P_{x|y=-1}(S_{1}+B_\ep).
$$
This expression holds for any classification problem, and the set $S_{1}$ (resp. $S_{-1}$) denotes the set of all $x \in \mathbb{R}^p$ which are classified to $+1$ (resp. $-1$) by the classifier $\hat{y}$. When $\hat{y}$ is a linear classifier, both sets $S_{1}$ and  $S_{-1}$ are half-spaces, e.g.\revone{,}
$S_{1} = \{x \in \mathbb{R}^p: w^T x  - c \geq 0 \} $. Furthermore, it is easy to see that the sets $S_{+1} + B_\ep$ and $S_{-1} + B_\ep$ are also half-spaces. E.g.\revone{,} we have 
$S_{1} + B_\ep = \{x \in \mathbb{R}^p: w^T x - c + \ep \| w \|_{*} \geq 0\}$ where $\|\cdot\|_{*}$ is the dual norm. In other words, we can interpret $S_{1} + B_ep$ as the set of all the points that are classified as $+1$ by a slightly shifted \emph{linear} classifier $(w, c - \ep\| w\|_{*})$. Hence, the term $P_{x|y=-1}(S_{1}+B_\ep)$ is the probability that the the new linear classifier $(w, c - \ep\| w\|_{*})$ labels a point $x$ as $+1$ while $x$ is generated conditioned on $y = -1$.  

Let now $(x_i,y_i)$ for $i=1,\ldots,n$ be sampled iid from a joint distribution $P_{x,y}$ for $i=1,\ldots,n$. Let the fraction of 1-s be $\pi_n\in[0,1]$. Let $P_{n\pm}$ be the empirical distributions of $x_i$ given $y_i=1$ and $-1$, respectively. We can write the finite sample robust risk as 
\begin{equation} \label{R_emp_formula}
R_n(\hat y,\ep)
=
\pi_n \cdot P_{n+}(S_{-1}+B_\ep)
+(1-\pi_n) \cdot P_{n-}(S_{1}+B_\ep).
\end{equation}
As explained above, for any linear classifier $(w,c)$ the sets $S_{1} + B_\ep$ and $S_{-1} + B_\ep$ are equivalent to half-spaces created by slightly shifted linear classifiers. Hence, considering the hypothesis class of all linear classifiers, the complexity of the sets $S_{1}$ (resp. $S_{-1}$)  is the same as the complexity of the sets $S_{+1} + B_\ep$ (resp. $S_{-1} + B_\ep$).  Now, by using standard arguments from uniform-convergence theory and PAC learning, and noting that the class of halfspaces has VC-dimension $p+1$, we conclude that 
For any $\delta > 0$,
\begin{equation*}
  \Pr\Big\{
    \forall_{(w,c) \in \bbR^p \times \bbR} \quad
    \Big| P_{n+}(S_{1} + B_\ep) - P_{x|y=1}(S_{-1}+B_\ep) \Big|
    \leq \delta
  \Big\} \geq 1- \exp(C( p - n\delta^2 )),
\end{equation*}
where $C$ is a constant independent of $n,p$. A similar result can be obtained for uniform concentration on the sets $S_{-1} + B_\ep$. We also note (using\revone{,} e.g.\revone{,} Hoeffding's inequality) that  $ \text{Pr} ( | \pi_n - P(y=1) | < \delta ) \geq 1-  2\exp(- n \delta^2)$. The result of the theorem now follows by incorporating the bounds obtained above into \eqref{R_emp_formula} and choosing $C$ sufficiently large but independent of $n,p$.    
\qed

\end{document}